\documentclass[preprint,3p]{elsarticle}
\usepackage{graphicx,epsfig}
\usepackage{epstopdf}
\usepackage{pdfpages}
\usepackage{algorithm}
\usepackage{algorithmicx}
\usepackage{algpseudocode}
 \algdef{SE}[DOWHILE]{Do}{doWhile}{\algorithmicdo}[1]{\algorithmicwhile\ #1}%

\usepackage{amssymb}
\usepackage{empheq}
\usepackage{cases}
\usepackage{amsthm,amsmath}
\usepackage{caption,lipsum}
\usepackage{stmaryrd}
\usepackage{tabularx}
\usepackage{color}
\usepackage{hyperref}
\usepackage{empheq}
\usepackage[normalem]{ulem}
\usepackage{enumerate}
\DeclareGraphicsExtensions{-eps-converted-to.pdf}

\usepackage{float}
\usepackage{caption}
\usepackage{subfigure}
\usepackage{graphicx}
\usepackage{comment}
\usepackage{multirow}
\usepackage{mathrsfs}
\usepackage{amsmath,amssymb} 
\usepackage{color}
\usepackage{epsfig}
\usepackage{diagbox}
\usepackage{afterpage}
\usepackage{makecell}
\usepackage{empheq}
\usepackage{lineno}
\usepackage{amsthm}   
\DeclareMathOperator*{\argmax}{arg\,max}
\DeclareMathOperator*{\argmin}{arg\,min}
\usepackage{bm}
\usepackage{booktabs}       
\usepackage{hyperref}

\numberwithin{equation}{section}

\newcommand{\bR}{{\mathbb R}}

\newcommand{\ro}[1]{\uppercase\expandafter{\romannumeral#1}}

\newcommand{\bx}{\mathbf{x}}

\newtheorem{theorem}{Theorem}[section]
\newtheorem{lemma}[theorem]{Lemma}
\newtheorem{assumption}[theorem]{Assumption}

\newtheorem{example}[theorem]{Example}

\theoremstyle{remark}
\newtheorem{remark}[theorem]{Remark}

\journal{Elsevier}

\begin{document}
	
	\begin{frontmatter}
		
		
		\title{Multi-stage neural operator learning with application for convolutions}

		\author[Eit]{Zhiping Mao}\ead{zmao@eitech.edu.cn}
		\address[Eit]{School of Mathematical Sciences, Eastern Institute of Technology, Ningbo, 315200 Ningbo, China}
		
		\author[Whu]{Zhenye Wen}\ead{wenzhy@whu.edu.cn}
		\address[Whu]{School of Mathematics and Statistics, Wuhan University, 430072 Wuhan, China}

		\author[Tju,Lab]{Yong Zhang}\ead{Zhang_Yong@tju.edu.cn}
		\address[Tju]{Center for Applied Mathematics and KL-AAGDM, Tianjin University, 300072 Tianjin, China}
    	\address[Lab]{State Key Laboratory of Synthetic Biology, Tianjin University, Tianjin, 300072, China}

		\author[Whu,Hubei]{Xiaofei Zhao}\ead{matzhxf@whu.edu.cn}
		\address[Hubei]{Computational Sciences Hubei Key Laboratory, Wuhan University, 430072 Wuhan, China}
		
		\begin{abstract}
		Convolution integrals widely exist in applications, and 
        to enable fast and accurate
computations, this paper introduces two general multi-stage neural operator learning frameworks. The first, Deep Collocation Neural Operator (DCNO), is a supervised approach that iteratively refines the operator approximation by learning residuals from input-output data pairs. The second, Deep Galerkin Neural Operator (DGNO), is an unsupervised framework applicable when the target operator can be represented by a PDE, leveraging the weak form of the PDE residual for training. Both methods progressively construct basis operators through multiple training stages to enrich the approximation space, leading to significantly improved accuracy over standard one-shot operator learning. We provide theoretical analysis for their approximation capabilities and implement them for learning convolutions. 
        Extensive numerical experiments demonstrate that both DCNO and DGNO achieve high accuracy, approaching machine precision under single float for convolution problems, and offer substantial efficiency gains for numerous queries or parametric variations compared to traditional solvers. We also extend these frameworks to handle multi-input operator learning scenarios involving variations in both the density and kernel of a convolution.
			
		\end{abstract}
		

		\begin{keyword}
			Operator learning, Convolution, Residual-driven approximation, Multiple stages, Deep operator network, High accuracy
			
		\end{keyword}
		
	\end{frontmatter}
	
	\section{Introduction}
Convolution computations arise naturally in a wide range of scientific and engineering applications, such as potential evaluation in electrostatics,  gravitational particle simulations and signal processing \cite{greengard1988rapid, ying2004kernel, oppenheim1999discrete, gonzalez2009digital, ji2025machine}, playing a crucial role in efficiently capturing long-range interactions.
This paper focuses on the accurate and efficient evaluation of convolution. Let $\mathcal{G}$ denote the target convolution operator mapping the density $\rho(\bx)$ to the potential $\phi(\bx)$, i.e.,
 \begin{equation}\label{eq:convolution}
 	\phi(\bx) = \mathcal{G}(\rho)(\bx) = \int_{\mathcal{R}} K(\bx,\tilde\bx) \rho(\tilde\bx) \, d\tilde\bx, \quad \bx \in\mathcal{R}\subset \bR^d
 \end{equation}
 where $d\in\mathbb{N}_+$ is the spatial dimension, $\mathcal{R}$ denotes a region, and $K(\bx,\tilde\bx)$ denotes a kernel.

 Despite the linearity of the mapping, 
the nonlocal nature of convolution \eqref{eq:convolution} renders its evaluation computationally demanding under direct discretizations. Fast solvers such as fast multipole methods, hierarchical matrix techniques, Gaussian-Summation method, and kernel truncation method have therefore been extensively developed to enable efficient  computation \cite{greengard1988rapid, ying2004kernel, hackbusch1999sparse, GauSum,  KTM} of convolution for an input $\rho$.  They are able to produce predicted $\phi$ on $N\in\mathbb{N}_+$ prescribed grids at cost $\mathcal{O}(N\log N)$.   
Although effective, 
repeated calls of the traditional solver for new density inputs in scenarios like  parametric variations or many-query settings may still incur substantial costs. 
This forms a natural circumstance for 
operator learning to apply, where neural networks are set to learn the mapping from the input function space (e.g., density and/or kernel) to the output solution space (e.g., potential). The trained operator network will be able to offer a fast prediction of the solution function for a wide range of input functions at arbitrary points of the region \cite{DeepONet,brunton2024promising, wang2023scientific,FNO}. In contrast, traditional solvers produce the predicted solution only on the prescribed grid points, and values elsewhere will need to be obtained posteriorly via proper interpolations to maintain accuracy, e.g., \cite{NUFFT}. 

In the literature, a wide range of neural operator learning methods have been developed, with representative approaches including Fourier neural operators \cite{FNO,GeoFNO}, deep operator networks (DeepONet)~\cite{DeepONet,MIONet,PODDeepONet,piDeepONet}, graph kernel networks~\cite{GKN}, Orthogonal Polynomial Neural Operator~\cite{OPNO}, and random feature-based models~\cite{RFMO}. 
Successful applications have already been made in a broad range of scientific computing problems \cite{DeepMMnet,NF,DiLeoni2023,Li2024water,GE2021,Mao2021}, and for the particular studies on the convolution problem \eqref{eq:convolution}, we refer to \cite{DeepONet,IDE}.
 Despite the remarkable progress, the existing operator learning works are based on direct one-shot training, whose accuracy is often quite limited. The resulting operator network for the convolution problem \eqref{eq:convolution} would therefore be less competitive compared to traditional fast solvers.  
 In fact, such a one-shot training accuracy limitation was noted earlier in deep learning of fixed-configuration PDE problems~\cite{GNN,DCM,MSNN}, and 
to address this issue, several residual-based approaches reaching high accuracy have been developed.
They progressively seek new bases to improve solution accuracy by minimizing the current residual of the PDE. 
The Galerkin Neural Networks (GNN)~\cite{GNN} employs a Galerkin framework based on the weak formulation of PDEs to iteratively enrich the solution approximation, particularly for the problems involving singularities. Its extension, Extended Galerkin Neural Networks~\cite{xGNN}, further generalizes this framework to a broader class of boundary value problems. Under the strong formulation, the Multi-level Neural Networks~\cite{MLNN} and the Multi-stage Neural Networks~\cite{MSNN}  progressively improve the approximation by sequentially introducing new neural network components to approximate the residual, whereas the Multigrade Neural Networks~\cite{MGNN} achieves iterative refinement through progressive network deepening. More recently, the Deep Collocation Method~\cite{DCM} adopts a collocation-based least-squares framework to improve training efficiency while maintaining high accuracy.
These high-accuracy PDE solvers offer valuable insights for developing high-accuracy operator learning frameworks, which will be one of our goals. 

To this end, in this work we shall first propose two high-accuracy operator learning frameworks that enrich the approximation space by progressively constructing neural operator bases as DeepONets through multiple residual-based training stages guided by the residuals.  
Specifically, a deep collocation neural operator (DCNO) method is designed for settings where input-output data pairs are available, falling within a supervised learning paradigm. On the other hand, when the task operator can be represented as 
a PDE,  
a deep Galerkin neural operator (DGNO) method is proposed, whose training uses the residual of the PDE in the weak form, enabling fully unsupervised operator learning. Theoretical analysis is done for each of them to support their approximation ability and convergence of the iterative frameworks. 
Afterwards, we shall apply the two proposed methods to the convolution operator learning task \eqref{eq:convolution}, and demonstrate through a series of numerical experiments and comparisons the accuracy and efficiency of the proposed algorithms. Extensions to the case with multiple inputs are implemented in the end using MIONet \cite{MIONet}. 
The main contributions of this work are summarized as follows.
\begin{itemize}
	\item[1)] Two general multi-stage operator learning frameworks are proposed, including a supervised one based on data pairs and an unsupervised one based on PDE. 
   They are implemented using DeepONet and analyzed for  accuracy and convergence. 
	\item[2)] The proposed algorithms are applied to convolution operator learning with extensions made using MIONet. Experimental results demonstrate efficient predictions with accuracy that approaches machine precision under single-precision floating-point arithmetic, making the proposed methods a competitive addition to the state-of-the-art solvers for convolution. 
\end{itemize}

The rest of this paper is organized as follows. Section \ref{sec:2} provides the necessary preliminaries on residual-based neural network methods and operator networks. Sections \ref{sec:3} and \ref{sec:4} present and analyze the DCNO and DGNO methods, respectively. Applications to convolution operator learning problems are given in Section \ref{sec:5}. Some conclusions are drawn in Section \ref{sec:con}.
 
 \section{Preliminary} \label{sec:2}
In this section, we briefly review some preliminary knowledge that is fundamental for proposing our method. 
The first part presents a class of residual-based neural network methods for high-accuracy PDE solving, and the second part presents the neural network for operator learning.
 
 
\subsection{Residual-based neural network methods for high-accuracy PDE solving}
Although deep neural networks are capable of representing intricate solution structures of PDEs \cite{PINN}, direct one-shot training often faces limitations in accuracy. 
Residual-based neural network methods in the literature have addressed this challenge by iteratively constructing adaptive approximation spaces using multiple neural networks, thereby progressively improving approximations for PDEs.
For illustration here, we briefly present a representative approach, namely the Galerkin neural network (GNN) \cite{GNN}. 

Consider the following PDE problem in a bounded domain $\Omega \subset \mathbb{R}^d$ with boundary $\partial \Omega$:
\begin{equation}\label{eq:strong form}
	\begin{cases}
		\mathcal{L} u(x) = f(x), & x \in \Omega, \\
		 \mbox{boundary conditions}, & x \in \partial \Omega,
	\end{cases}
\end{equation}
where $\mathcal{L}$ is a linear differential operator  and $f(x)$ is a source term. The bounded setup here is natural for computations. 
Assume that \eqref{eq:strong form} can be reformulated into a variational (weak) form: find
\begin{equation}\label{eq:var_prob_residual}
	u \in X \quad \text{such that} \quad a(u,v)=F(f,v), \qquad \forall v\in X .
\end{equation}
Here, 
$X$ denotes the solution space, 
 $F(f,\cdot): X \to \mathbb{R}$ is a bounded linear functional on $X$ induced by the source term $f$, and
 $a(\cdot,\cdot):X\times X\to\mathbb{R}$ is a bounded, symmetric, bilinear form defined on $X$ satisfying 
\[
a(v,v) \geq 0, \  \text{and} \  a(v,v)=0 \ \Longleftrightarrow\ v=0 .
\]
Under this assumption, $a(\cdot,\cdot)$ induces an inner product on $X$, 
with the corresponding norm defined by $\|v\|_X
= \sqrt{a(v,v)}$.
For example, in \eqref{eq:strong form}, when $\mathcal{L} = -\Delta$ with homogeneous Dirichlet boundary conditions, one takes
$X = H^1_0(\Omega)$, and
the operator $a$ and $F$ associated with \eqref{eq:strong form} read $a(u,v) = \int_\Omega \nabla u \cdot \nabla v \, d\mathbf{x}$, $F(f,v) = \int_\Omega f\, v \, d\mathbf{x}$.
The Galerkin subspace within GNN is constructed as follows. 
Starting from an initial basis function $\psi_0=0$,
at stage $s\geq1$, let
$
\Psi_{s-1}=\mathrm{span}\{\psi_0,\dots,\psi_{s-1}\}\subset X
$
denote the current approximation space spanned by the basis functions $\{\psi_j\}_{j=0}^{s-1}$, and let $u_{s-1}\in\Psi_{s-1}$ be the corresponding Galerkin approximation.
In principle, the most effective enrichment direction is given by the error
$
e_{s-1}:=u-u_{s-1},
$
since adding this direction to the approximation subspace would recover the exact solution in a single step.
Thus, GNN seeks a neural-network basis function $\psi_s$ that approximates the normalized error
$
\varphi_{s-1}
=
{(u-u_{s-1})}/
{\|u-u_{s-1}\|_X}.
$
Since the exact solution $u$ is not accessible, the residual functional
$
R_{s-1}(f,v)
:= F(f,v) - a(u_{s-1}, v)
$ can be used as a substitute due to the fact that the normalized error $\varphi_{s-1}$ mathematically maximizes $R_{s-1}(f,v)$ over the unit ball in $X$ \cite{GNN}: 
\begin{equation}\label{eq:gnn_optimal}
		\varphi_{s-1}\in
		\argmax_{\|v\|_{X}=1}
		R_{s-1}(f,v),\quad R_{s-1}(f,\varphi_{s-1})
		=
		\|u-u_{s-1}\|_X.
	\end{equation}
This characterization suggests the construction of the new basis function:
\begin{equation}\label{eq:GNN_max}
    \psi_s
    =
    \argmax_{v\neq 0}
    \frac{R_{s-1}(f,v)}
    {\|v\|_X},
\end{equation}
and with the enriched finite-dimensional subspace $\Psi_s= \mathrm{span}\{\psi_0,\dots,\psi_{s}\}\subset X$, the Galerkin approximation $u_s\in\Psi_s$ is determined by solving a $(s+1)\times (s+1)$ linear system
\begin{equation}\label{eq:Galerkin_project}
	u_{s} \in \Psi_s: \quad a(u_{s}, \psi_{j}) = F(v_{j}) \quad 0\leq j\leq s,
\end{equation}
for the linear expansion coefficients of $u_s$ on the basis of $\Psi_s$.

In addition to GNN, residual-based neural network methods based on the strong form of PDE \eqref{eq:strong form} have also been proposed, such as Multi-level Neural Networks \cite{MLNN} and the Deep Collocation Method \cite{DCM}. We shall develop some high-accuracy neural operator learning methods by leveraging the residual-based framework.

 \subsection{Deep operator network}
 
 In this work, we adopt DeepONet as the neural architecture for operator approximation, designed to learn mappings between infinite-dimensional function spaces \cite{DeepONet}.
 Consider an operator $\mathcal{G}: \mathcal{W} \to X$, mapping an input function $\rho \in \mathcal{W}$ to an output function $\phi \in X$, where $\mathcal{W}$ and $X$ both are (for simplicity) function spaces defined on $\Omega\subset \mathbb{R}^{d}$ with $d\in \mathbb{N}_+$. 
 For any $\mathbf{x} \in \Omega$, DeepONet approximates $\mathcal{G}(\rho)$ via a trunk-branch decomposition:
 \begin{equation}
 	\label{DeepONet}
 	\mathcal{G}(\rho)(\mathbf{x}) \approx \mathcal{G}_{\theta}(\rho)( \mathbf{x}) 
 	= \sum_{k=1}^{p} \mathfrak{b}_k\!\left(\rho(\mathbf{s}_1), \ldots, \rho(\mathbf{s}_m); \theta_{\mathfrak{b}}\right)\,
 	\mathfrak{t}_k(\mathbf{x}; \theta_{\mathfrak{t}}),
 \end{equation}
 where $\theta=\{\theta_{\mathfrak{b}},\,\theta_{\mathfrak{t}}\}$ denotes the trainable parameters, with $\theta_{\mathfrak{b}}$ and $\theta_{\mathfrak{t}}$ representing the parameters of the branch and trunk networks, respectively. Here, $\{\mathbf{s}_l\}_{l=1}^m \subset \Omega$ are sensor locations, and latent dimension $p \in \mathbb{N}_+$ is the output width of branch and trunk networks.


The branch network is responsible for extracting a $p$-dimensional latent representation 
$\mathfrak{b}=(\mathfrak{b}_1,\mathfrak{b}_2,\dots,\mathfrak{b}_p)\\\in\mathbb{R}^{p}$
of the input function $\rho$ and is typically implemented as a multilayer perceptron (MLP) \cite{rosenblat1958perceptron}.
In the standard setting, the branch input consists of the evaluations of $\rho$ at fixed sensors $\{\mathbf{s}_l\}_{l=1}^{m}$.
More generally, the branch input can also be taken as the expansion coefficients of $\rho$ with respect to a prescribed basis of $\mathcal{W}$ \cite{DOOL,DeepONet,PODDeepONet}.
The trunk network captures the dependence of the output function on the coordinate $\mathbf{x}\in\Omega$. 
It is implemented as another MLP that maps $\mathbf{x}$ to a latent vector $\mathfrak{t} = (\mathfrak{t}_1, \mathfrak{t}_2, \dots, \mathfrak{t}_p) \in \mathbb{R}^p$, which can be interpreted as a set of learned basis functions over the output domain. 
The operator approximation is then constructed through the inner product of $\mathfrak{b}$ and $\mathfrak{t}$, resulting in a continuous representation of $\mathcal{G}$.

Since its introduction, DeepONet has been successfully applied to a wide range of problems \cite{DeepMMnet,NF,Li2024water,GE2021}.
 However, its standard supervised training paradigm, which relies on labeled data and one-shot optimization, often suffers from limited approximation accuracy. 
The aim of this work, on the one hand, is to develop a general multi-stage framework for high-accuracy operator learning by progressively constructing basis operators guided by residuals and improving the accuracy of operator approximation. On the other hand, 
the generation of high-fidelity training data for labels typically requires repeated calls to traditional numerical solvers, leading to substantial computational cost. Thus, in addition to the multi-stage framework, we also work for an unsupervised operator learning manner that eliminates the need for labeled data.

\section{Deep collocation neural operator (DCNO) method}\label{sec:3}


In this section, we introduce the deep collocation neural operator (DCNO) method, a supervised framework for high-accuracy operator learning. 
Instead of training a single, monolithic neural operator to approximate a target operator $\mathcal{G}(\rho)=\phi$ directly, DCNO constructs a sequence of basis operators and then approximates the target operator as a linear combination of trained basis operator networks. 
The following first presents the method and then provides some analysis to support its approximation validity. 

\subsection{Neural network basis operators and collocation scheme}
Let us first outline the basic framework.  We construct 
the basis operators through a greedy iterative strategy similarly as DCM and GNN \cite{GNN, DCM} (see also \cite{MLNN, MSNN, MGNN}), where each basis operator is implemented as DeepONet \eqref{DeepONet} to approximate the solution residual generated at the current iteration.  
Beginning with an initial operator $\mathcal{G}_0:=\mathcal{G}_{\theta_0}$, which is chosen as the zero operator in this work, the initial residual is defined as $e_0=\mathcal{G}-\mathcal{G}_{\theta_0}=\mathcal{G}.$ The first basis operator $\mathcal{G}_{\theta_1}$ is then trained to approximate the initial residual mapping, i.e., $\mathcal{G}_{\theta_1}\approx e_0$. The resulting approximation for $\mathcal{G}$ is then updated to $\mathcal{G}_1=\mathcal{G}_{\theta_0}+\mathcal{G}_{\theta_1}.$
Assume that we are now at some stage $s\geq 2$ with the approximation from the previous stage
$\mathcal{G}_{s-1} = \sum_{j=0}^{s-1}  \mathcal{G}_{\theta_j}$, then the corresponding residual operator is defined as $e_{s-1} = \mathcal{G} - \mathcal{G}_{s-1}$.
By training a new basis operator $\mathcal{G}_{\theta_s}$ to approximate the residual mapping $\rho \mapsto e_{s-1}$, the update
$\mathcal{G}_{s} = \sum_{j=0}^{s}  \mathcal{G}_{\theta_j}$
thereby incrementally improves the overall approximation for $\mathcal{G}$. 

Now, we detail the implementation. We 
consider here for DCNO in the particular case that solution data  $\{ (\rho^{(b)},\phi^{(b)}) \}_{b=1}^{N_b}$ are available, where $\rho^{(b)}$ are sampled inputs with $N_b \in \mathbb{N}_+$ and $\phi^{(b)}=\mathcal{G}(\rho^{(b)})$ are possibly obtained from numerical solvers. 
Then, each stage can be trained in a supervised manner, as detailed below. 
Let $\Omega$ denote the computational (bounded) domain in practice. 
The trunk network in~\eqref{DeepONet} takes the spatial coordinate $\mathbf{x}\in\Omega$ as input. 
A set of collocation points is then generated through discretization:
$
\{\mathbf{x}_k\}_{k=1}^{N_{\bx}} \subset \Omega,
$
where $N_{\bx}\in\mathbb{N}_+$ denotes the number of collocation points.
Accordingly, for each input sample $\rho^{(b)}$, $b=1,\dots,N_b$, with corresponding output $\phi^{(b)}$, 
the training data are constructed as
$\{(\rho^{(b)}(\mathbf{x}_k), \phi^{(b)}(\mathbf{x}_k))\,:\,k=1,\ldots,N_{\bx} \}_{b=1}^{N_b}$. 
Let $\mathcal{G}_{0}=0.$ 
For $s\geq1$, the residual corresponding to the $b$-th training sample at stage $s-1$ is given by
\begin{equation}\label{eq:residual}
	e_{s-1}(\rho^{(b)}) (\mathbf{x}):= \phi^{(b)}(\mathbf{x}) - \mathcal{G}_{s-1}(\rho^{(b)})(\mathbf{x}), 
	\qquad b=1,\dots,N_b.
\end{equation}
We measure it under a discrete norm:
\begin{equation}\label{cs def}
\|e_{s-1}\|
=\sqrt{
\frac{1}{N_b\times N_{\bx} }
\sum_{b=1}^{N_b}\sum_{k=1}^{N_{\bx}}|e_{s-1}(\rho^{(b)})(\mathbf{x}_k)|^2}=:\beta_{s-1},\quad s\geq1.
\end{equation}
As long as $\beta_{s-1}>0$, DCNO could in principle proceed to the next stage and the DeepONet $\mathcal{G}_{\theta_s}$ is determined by minimizing the empirical loss
\begin{equation}\label{eq:loss_deeponet}
	\text{Loss}_s(\theta) := \frac{1}{N_b\times N_{\bx} } \sum_{b=1}^{N_b} \sum_{k=1}^{N_{\bx}}  \left| \frac{1}{\beta_{s-1}}e_{s-1}(\rho^{(b)})(\mathbf{x}_k) -\mathcal{G}_{\theta}(\rho^{(b)})(\mathbf{x}_k) \right|^2,\quad \mbox{and} \quad  \theta_s\in\argmin_{\theta}\left\{ \text{Loss}_s(\theta)  \right\}.
\end{equation}
Here, the residual is scaled to order one for practical training efficiency \cite{MLNN}, and the approximate operator for $\mathcal{G}$ at stage $s$ reads:
\begin{equation}\label{Gs def}
\mathcal{G}_{s} (\rho):= \sum_{j=1}^{s}  \beta_{j-1}\mathcal{G}_{\theta_j}(\rho),\quad s\geq1.
\end{equation}
In our numerical experiments later, 
the unconstrained optimization problem in  \eqref{eq:loss_deeponet} for every $s\geq1$ will be solved by Adam optimizer \cite{diederik2014adam} with the Xavier method \cite{glorot2010understanding} for parameter initialization. 
It is common for the optimizer to produce a practical minimizer for \eqref{eq:loss_deeponet} with $\text{Loss}_s(\theta_s)>0$, which is in fact fine here since DCNO will pass the error as some residual to the next stage for refinement. 

This process can be terminated once $
\beta_s=\|e_{s}\|
\leq \epsilon$ for some $s$, where $\epsilon>0$ is a prescribed tolerance. This means that the target operator has been approximated to the desired accuracy on the training set. If the training data are sampled sufficiently large and effective, the final DCNO approximation $
\mathcal{G}(\rho)\approx\mathcal{G}_s(\rho) $ can therefore become valid for any input function $\rho$ (inside or outside the training set) from a desired functional space. 
Note that by construction, $e_s=\mathcal{G}-\mathcal{G}_s
=e_{s-1}-\beta_{s-1}\mathcal{G}_{\theta_s}$, and so the stopping criterion $\|e_{s}\|
\leq \epsilon$ can be replaced equivalently as $$
\beta_{s-1}\,\sqrt{\text{Loss}_s(\theta_s)}
\leq \epsilon.$$ 

The whole process of the DCNO method is summarized in Algorithm \ref{alg:DCNO}. Note that DCNO itself does not require the target operator $\mathcal{G}$ to be linear, despite that we will apply it later to convolutions.

\begin{algorithm}[h!]
	\caption{DCNO method}
	\label{alg:DCNO}
	\begin{algorithmic}[0]
        \State Parameters: tolerance $\epsilon$; maximum step $N_e$ for optimizer
		
		\State \textbf{Input:} Training dataset 
		$\{(\rho^{(b)}, \phi^{(b)})\}_{b=1}^{N_b}$; 
		collocation points $\{\mathbf{x}_k\}_{k=1}^{N_\bx}$
		
		\State \textbf{Output:} Basis operator set 
		$\{\mathcal{G}_{\theta_j}\}_{j=1}^{s}$; Solution operator $\mathcal{G}_s(\rho) \approx \mathcal{G}(\rho)$
		
	\end{algorithmic}
	\begin{algorithmic}[1]
		
		\State \textbf{Initialize:} 
        $s= 0$; $\mathcal{G}_0=0$;

		\State \textbf{Part I: Construction of neural operator basis}

        \Do

        \State Set $s \gets s+1$
		\State Compute $e_{s-1}(\rho^{(b)})$ \eqref{eq:residual} for every $b$ and compute \eqref{cs def} for $\beta_{s-1}$
		
		\State Initialize DeepONet parameters $\theta_s$

		\For{$epoch = 1$ \textbf{to} $epoch = N_e$}
		\State compute the loss function by \eqref{eq:loss_deeponet} 
		\State update $\theta_s$ via the optimizer
		\EndFor
        
		\State Update approximation $\mathcal{G}_s =  \sum_{j=1}^{s}\beta_{j-1}\mathcal{G}_{\theta_j}$
        
		\doWhile{$\beta_{s-1}\,\sqrt{\text{Loss}_s(\theta_s)}>\epsilon$}

		\State \textbf{Part II: Inference for new input $\rho$}
		
		\State Given $\rho$, evaluate 
		$\{\mathcal{G}_{\theta_j}(\rho)\}_{j=1}^{s}$

		\State Form the prediction 
		$\mathcal{G}_s(\rho) = \sum_{j=1}^{s} \beta_{j-1} \mathcal{G}_{\theta_j}(\rho)$
		
	\end{algorithmic}
\end{algorithm}

\subsection{Numerical analysis}
Here we present some numerical analysis for the proposed DCNO method. 
Let $\tilde{\mathcal{G}}(\rho)$ and $\mathcal{G}(\rho)$ denote the exact operator and the reference operator, respectively, where the latter serves as the training label in our algorithm and may deviate from the exact one. They both map $\mathcal{W}$ to $X\subset L^2(\Omega)$.  
With the numerical solution $\mathcal{G}_s
=\sum_{j=1}^{s}  \beta_{j-1}\mathcal{G}_{\theta_j}$ at the $s$-th stage, 
the stage-wise error is defined by
\begin{equation}\label{eq:error}
	e_s(\rho) := \mathcal{G}(\rho) - \mathcal{G}_s(\rho), 
\end{equation}
and by taking the data error $\tilde{e}(\rho) := \tilde{\mathcal{G}}(\rho) - \mathcal{G}(\rho)$ into account, we further define
\begin{equation}\label{eq:true_error}
	\tilde{e}_s(\rho) := \tilde{\mathcal{G}}(\rho) - \mathcal{G}_s(\rho).
\end{equation}
We consider the DeepONets for stage-wise construction of DCNO to be from 
\begin{equation}\label{eq:DeepONet_class}
	\mathcal{D}:= 
	\left\{ 
	\mathcal{G}_{\theta}(\rho)(\mathbf{x}) = {\mathfrak{b}}(\rho)\cdot {\mathfrak{t}}(\bx) 
	\;\middle|\; 
	{\mathfrak{b}}  \in \mathcal{N}_{branch}, {\mathfrak{t}} \in \mathcal{N}_{trunk} 
	\right\},
\end{equation}
where $\mathcal{N}_{branch}$ and $\mathcal{N}_{trunk}$ are classes of continuous neural network functions (e.g., the MLP class), and each picked $\mathfrak{b},\mathfrak{t}$ are two vectors of the same size so that their inner product makes sense. 
To quantify the approximation performance over the input space $\mathcal{W}$, we consider a finite measure $\mu$ supported on $\mathcal{W}$ and introduce the Bochner space $L^2(\mathcal{W};L^2(\Omega))$ equipped with the norm \cite{adams2003sobolev}
\begin{equation}\label{eq:Bochner_norm}
	\|v\|_{L^2(\mathcal{W};L^2(\Omega))}^2 
	:= \int_{\mathcal{W}} \|v(\rho)\|_{L^2(\Omega)}^2 \, d\mu(\rho), \quad v: \mathcal{W} \to L^2(\Omega).
\end{equation}
Note that \eqref{cs def} can be interpreted as a realization/discretization of \eqref{eq:Bochner_norm}, so we will still denote $\beta_s=\|e_{s}\|_{L^2(\mathcal{W};L^2(\Omega))}$ for simplicity in the following. 
For the error estimate of the DCNO construction process, we make the following basic assumption.

\begin{assumption}\label{ass:DCNO}  
   Suppose that $\mathcal{G}:\mathcal{W}\to X$ is a continuous operator,
where $\mathcal{W}$ is a compact subset of $C(\overline{\Omega})$ and
$X=C(\overline{\Omega})$ with $0<|\Omega|<\infty$ and  $0<C:=\int_{\mathcal{W}}d\mu(\rho)<\infty$. 
    Before the construction process of DCNO in Algorithm~\ref{alg:DCNO} stops at some $S\in\mathbb{N}_+$, the residual operator is non-zero, i.e.,
\[
\beta_s=\|e_{s}\|_{L^2(\mathcal{W};L^2(\Omega))} > 0,
\quad s=0,1,\ldots S-1. 
\]
We assume that the training and sampling within DCNO are sufficient at each stage so that the resulting DeepONet basis $\mathcal{G}_{\theta_s}$ for $s=1,\ldots S$ do not pollute the universal approximation ability of $\mathcal{D}$ up to a prescribed threshold $\tau/\sqrt{C|\Omega|}$. 
\end{assumption}

Under Assumption~\ref{ass:DCNO}, we can directly have the following error estimate result.

	\begin{theorem}\label{Prop:error_bound_all_Bochner}
    Fix one $0<\tau<1$ and let
        Assumption~\ref{ass:DCNO} hold. Then, the error of DCNO is strictly decreasing during construction of each stage, i.e.,
       $
		\|e_{s}\|_{L^2(\mathcal{W};L^2(\Omega))}< \|e_{s-1}\|_{L^2(\mathcal{W};L^2(\Omega))},$         
        and the error at the $s$-stage satisfies the bound
		\begin{equation}\label{eq:error_bound}
		\| e_{s} \|_{L^2(\mathcal{W};L^2(\Omega))} 
		\leq 
        \tau^s \,
		\| e_{0} \|_{L^2(\mathcal{W};L^2(\Omega))}, \quad s=1,\ldots,S.
		\end{equation}
 Taking into account the data error $\tilde{e}$, we further have
		\begin{equation}\label{eq:error_bound_noise}
		\| \tilde{e}_{s} \|_{L^2(\mathcal{W};L^2(\Omega))} 
		\leq 
		\|\tilde{e}\|_{L^2(\mathcal{W};L^2(\Omega))}
		+  \tau^s \, \| e_{0} \|_{L^2(\mathcal{W};L^2(\Omega))},\quad s=1,\ldots,S.
		\end{equation}
	\end{theorem}
	\begin{proof}
    Since each $\mathcal{G}_{\theta_s}$ is taken from $\mathcal{D}$, we have $\mathcal{G}_s$ as defined in \eqref{Gs def} is a continuous operator and so does $e_s(\rho)$ given by \eqref{eq:error} for each $s$. 
By the universal approximation theorem for DeepONet \cite{DeepONet}, the basis operator network generated within DCNO at stage $s$, i.e., $\mathcal{G}_{\theta_s}\in \mathcal{D}$ for $1\leq s\leq S$, can approximate the last scaled residue $e_{s-1}/\beta_{s-1}$ to a range: 
    \begin{equation}\label{dcno uni-appr thm}
    \left|\frac{1}{\beta_{s-1}}e_{s-1}(\rho)(\mathbf{x})-\mathcal{G}_{\theta_s}(\rho)(\mathbf{x})\right|\leq\frac{\tau}{\sqrt{C|\Omega|}},\quad \forall\rho\in\mathcal{W},\ \mathbf{x}\in \Omega.\end{equation}

By taking the square on both sides of \eqref{dcno uni-appr thm} and then integrating 
with respect to $\bx \in \Omega$, we have
	\[
	\left\|\frac{1}{\beta_{s-1}}e_{s-1}(\rho)- \mathcal{G}_{\theta_{s}}(\rho)\right\|_{L^2(\Omega)}^2 
	= \int_{\Omega} \left|\frac{1}{\beta_{s-1}}e_{s-1}(\rho)(\mathbf{x})-\mathcal{G}_{\theta_s}(\rho)(\mathbf{x})\right|^2 \, d\bx
	\leq\frac{\tau^2}{C}, \quad \forall \rho \in \mathcal{W}.
	\]
	Further integrating the above with respect to $\rho \in \mathcal{W}$ under the measure $
    \mu$ and then taking the square root, we are led to 
$\|\frac{1}{\beta_{s-1}}e_{s-1} - \mathcal{G}_{\theta_{s}}\|_{L^2(\mathcal{W};L^2(\Omega))}\leq\tau .$
 By the construction of DCNO and the definition \eqref{eq:error}, we have $e_{s}=e_{s-1}-\beta_{s-1}\mathcal G_{\theta_{s}}$, and so \begin{equation}\label{eq:Universal-L2}
		\|e_{s}\|_{L^2(\mathcal{W};L^2(\Omega))}\leq\tau \|e_{s-1}\|_{L^2(\mathcal{W};L^2(\Omega))},\quad s=1,\ldots,S.
	\end{equation}
    As $\tau<1$, the error strictly decreases. 
Applying the estimate \eqref{eq:Universal-L2} recursively, we obtain \eqref{eq:error_bound}. 

Then, by subtracting \eqref{eq:true_error} from \eqref{eq:error}, we see that $$\tilde{e}_s(\rho) = \tilde{e}(\rho) + e_s(\rho).$$
Applying the triangle inequality to it yields
        $
		\|\tilde{e}_{s}\|_{L^2(\mathcal{W};L^2(\Omega))}
		\le \|e_{s}\|_{L^2(\mathcal{W};L^2(\Omega))} + \|\tilde{e}\|_{L^2(\mathcal{W};L^2(\Omega))}.
		$
 Using \eqref{eq:error_bound} further gives \eqref{eq:error_bound_noise}.
	\end{proof}

\section{Deep Galerkin neural operator (DGNO) method}
\label{sec:4}

Alternative to DCNO, 
in this section, we develop an unsupervised high-accuracy operator learning framework, termed the deep Galerkin neural operator (DGNO) method, for approximating the solution operator of a PDE. For illustration, the target PDE here is set as~\eqref{eq:strong form} with $\mathcal{G}: f \mapsto u$, the mapping of the input source term to the corresponding solution. Similarly as DCNO, DGNO also builds a finite sum of  basis operator networks for approximation, but the training is fully based on equation residue without requiring labeled data.


\subsection{Constructing neural network basis operator in an unsupervised manner}

As we are working with a PDE, to reduce the dependence of regularity, we choose to build DGNO for \eqref{eq:strong form} upon its variational formulation \eqref{eq:var_prob_residual}, via a framework analogous to GNN. The progressive construction procedure now is to find DeepONets as basis neural operators based on residuals to enrich the approximation
space. 


The iterative procedure starts with a zero initial operator $\mathcal{G}_0=\mathcal{G}_{\theta_0}=0$, then the initial normalized error under an input $f$ is $\varphi_0(f)=(\mathcal{G}(f)-\mathcal{G}_{0}(f))/\|\mathcal{G}(f)-\mathcal{G}_{0}(f)\|_{X}=\mathcal{G}(f)/\|\mathcal{G}(f)\|_{X}$. 
The first basis operator $\mathcal{G}_{\theta_1}$ is trained to approximate the initial error mapping, i.e., $\mathcal{G}_{\theta_1}\approx \varphi_0$, 
and the resulting approximation for $\mathcal{G}$ is then updated to $\mathcal{G}_1:=\beta_1\mathcal{G}_{\theta_1},$ where the coefficient $\beta_1$ will be determined via a Galerkin projection detailed later. 
Assume that we are now at some stage $s\geq 2$ with the approximation from the previous stage
$\mathcal{G}_{s-1} = \sum_{j=1}^{s-1}\beta_j  \mathcal{G}_{\theta_j}$, the new basis operator $\mathcal{G}_{\theta_s}$ is trained to approximate the current optimal update direction, namely the normalized error operator read under an input $f$ as
\begin{equation}\label{eq:residual operator}
	\varphi_{s-1}(f):=
	\frac{\mathcal{G}(f)-\mathcal{G}_{s-1}(f)}
	{\|\mathcal{G}(f)-\mathcal{G}_{s-1}(f)\|_{X}}.
\end{equation}
For fixed $f$, by \eqref{eq:gnn_optimal}, 
$\varphi_{s-1}(f)$
can be achieved by maximizing the normalized PDE-residual functional:
\begin{equation*}
		\varphi_{s-1}(f)\in
		\argmax_{\|v\|_{X}=1}
		R_{s-1}(f,v),\quad \mbox{where}\quad 
        R_{s-1}(f,v)
	:= F(f,v) - a(\mathcal{G}_{s-1}(f), v).
	\end{equation*}
   To extend this pointwise characterization to the operator level, we consider the expectation of the normalized residual functional with respect to a measure $\mu$ for input $f$, and further motivated from the GNN strategy \eqref{eq:GNN_max}, we are led to the following operator-level optimization problem:
\begin{equation}\label{eq:Optimal basis operator}
\mathcal{G}_{\theta_s} \in \argmax_{\mathcal{G}_{\theta} \in \mathcal{D} } \mathbb{E}_{f \sim \mu} \left[
 \frac{R_{s-1}(f,\mathcal{G}_{\theta}(f))}{\|\mathcal{G}_{\theta}(f)\|_X} 
\right],\quad s\geq1.
\end{equation}
Here, $\mathcal D$ denotes the class of DeepONets defined in \eqref{eq:DeepONet_class}. The target function of the optimization problem \eqref{eq:Optimal basis operator} is purely based on the residue of PDE, without calling for labeled data, and thus provides an unsupervised way for training the basis operators.




Now we give some details on implementation.
At stage $s \geq 1$, 
given the set of constructed operators $\{\mathcal{G}_{\theta_j}\}_{j=0}^{s-1}$, the training set of input $\{ f^{(b)} \}_{b=1}^{N_b}$, and the approximate solution $\mathcal{G}_{{s-1}}(f^{(b)}) = \sum_{j=0}^{s-1} \beta_j^{(b)} \mathcal{G}_{\theta_j}(f^{(b)})$. We look for the basis operator $\mathcal{G}_{\theta_s}\approx \varphi_{s-1}$ based on \eqref{eq:Optimal basis operator}. Here, the expectation is discretized as the empirical average over the training samples in practice, and then we are motivated to train $\theta_s$ of $\mathcal{G}_{\theta_s}$ by minimizing the loss
\begin{equation}\label{eq:loss_deeponet_un}
	\text{Loss}_s(\theta)
	=
	-\frac{1}{N_b}
	\sum_{b=1}^{N_b}
	\frac{
		R_{s-1}(f^{(b)},\mathcal{G}_{\theta}(f^{(b)}))
	}{
		\|\mathcal{G}_{\theta}(f^{(b)})\|_X
	},
    \quad  \theta_s\in\argmin_{\theta}\left\{ \text{Loss}_s(\theta)  \right\}.
\end{equation}
To ensure high accuracy, both the residual functional $R_{s-1}(\cdot,\cdot)$ and the energy norm $\|\cdot\|_X$ are computed using a Gauss--Legendre quadrature rule on the quadrature points $\{\bx_k\}_{k=1}^{N_\bx}$. 
With the obtained basis operator set $\{\mathcal{G}_{\theta_j}\}_{j=1}^s$ and an input $f$, we seek a Galerkin approximation $\mathcal{G}_s(f)$ for the solution of \eqref{eq:var_prob_residual} in 
$
\mathrm{span}\{\mathcal{G}_{\theta_j}(f)\}_{j=1}^s, 
$
i.e.,
$
\mathcal{G}_s(f)(\mathbf{x}) = \sum_{j=1}^{s} \beta_j(f)\, \mathcal{G}_{\theta_j}(f)(\mathbf{x}),
$
such that 
$$a(\mathcal{G}_s(f),\mathcal{G}_{\theta_j}(f))
=F(f,\mathcal{G}_{\theta_j}(f)),\quad 1\leq j\leq s.$$
The coefficient vector $\boldsymbol{\beta}(f) = [\beta_0(f), \dots, \beta_s(f)]^T$ is thus obtained by solving the linear system
\begin{equation}\label{eq:galerkin_system}
\mathbf{K}(f)\, \boldsymbol{\beta}(f) = \mathbf{F}(f),
\end{equation}
where $\mathbf{K}=(\mathbf{K}_{k,j})_{s\times s}$ and $\mathbf{F}=(\mathbf{F}_j)_{s\times 1}$
with entries
\[
\mathbf{K}_{k,j}(f) = a\big( \mathcal{G}_{\theta_k}(f),\, \mathcal{G}_{\theta_j}(f) \big), \qquad
\mathbf{F}_{j}(f) = F\big(f, \mathcal{G}_{\theta_j}(f) \big).
\]
The Galerkin projection scheme \eqref{eq:galerkin_system}  ensures that for each input $f$, it provides the optimal approximation to the solution of the variational problem within the current approximation subspace given by the basis operators. The construction process of DGNO can be stopped when a prescribed maximum stage number $S\geq1$ is reached. 
The overall framework of DGNO is summarized in Algorithm \ref{alg:DGNO}.

\begin{remark}
    The current DGNO is designed based on the bilinear form \eqref{eq:var_prob_residual}. Nevertheless, by incorporating linearization techniques such as fixed-point iteration, the proposed framework could be extended to general nonlinear case, which will be investigated in our future work.
\end{remark}

\begin{algorithm}[h!]
	\caption{DGNO method}
	\label{alg:DGNO}
	\begin{algorithmic}[0]
		\State Parameters: maximum stage number $S$; maximum step $N_e$ for optimizer
		
		\State \textbf{Input:} Training dataset 
		$\{f^{(b)}\}_{b=1}^{N_b}$; 
		Gaussian quadrature points $\{\mathbf{x}_k\}_{k=1}^{N_x}$
		
		\State \textbf{Output:} Basis operator set 
		$\{\mathcal{G}_{\theta_j}\}_{j=1}^{S}$; Solution $\mathcal{G}_S  \approx \mathcal{G}(f)$
		
	\end{algorithmic}
	\begin{algorithmic}[1]
		
		\State \textbf{Initialize:} 
		$\mathcal{G}_0 \gets 0$; 
		$\Psi_s \gets \emptyset$

		\State \textbf{Part I: Greedy construction of neural operator basis}
		
		\For{$s = 1$ \textbf{to} $S$}
		
		
		\State Initialize DeepONet parameters $\theta_s$
		
		\For{$epoch = 1$ \textbf{to} $epoch = N_e$}
		\State compute the loss function by \eqref{eq:loss_deeponet_un} 
		\State update $\theta_s$ via the optimizer
		\EndFor
		
		\State Update basis set:
		$\Psi_{s+1} \gets \Psi_s \cup \{\mathcal{G}_{\theta_s}\}$
		
		\State Compute coefficients $\boldsymbol{\beta}(f^{(b)})$ for every $b$ via \eqref{eq:galerkin_system}
		
		\State Update approximation $\mathcal{G}_s(f^{(b)}) = \sum_{j=1}^{s} \beta_j(f^{(b)}) \, \mathcal{G}_{\theta_j}(f^{(b)})$ for $b=1,\cdots,N_b$
		
		\EndFor

		\State \textbf{Part II: Inference for new input $f$}
		
		\State Given $f$, evaluate 
		$\{\mathcal{G}_{\theta_j}(f)\}_{j=1}^{S}$
		
		\State Compute coefficient $\boldsymbol{\beta}(f)$ via solving linear system \eqref{eq:galerkin_system}
		
		\State   
		 Form the prediction $\mathcal{G}_S(f)(\bx) =\sum_{j=1}^{S} \beta_j(f) \mathcal{G}_{\theta_j}(f)(\bx)$

	\end{algorithmic}
\end{algorithm}

 

\subsection{Numerical analysis}
Here we perform some numerical analysis for the proposed DGNO method. 
 Our analysis considers a continuous projection neglecting the quadrature errors.
To quantify the overall approximation performance over the input space $\mathcal{W}$, 
we consider a probability measure $\mu$ supported on $\mathcal{W}$ and introduce the Bochner space $L^1(\mathcal{W};X)$ 
equipped with the norm \cite{adams2003sobolev}
\begin{equation}\label{eq:Bochner_norm1}
	\|v\|_{L^1(\mathcal{W};X)} 
	:= \int_{\mathcal{W}} \|v(f)\|_{X} \, d\mu(f),\quad v\in L^1(\mathcal{W};X).
\end{equation}
In the context of solving PDE, it is natural to consider the solution space $X$ as some Sobolev space and we make the following basic assumption. 
\begin{assumption}\label{ass:DGNO}  For some $0\leq q_1,q_2<\infty$, 
   let the target operator $\mathcal{G}:\mathcal{W}\to X$ be a continuous operator, where $\mathcal{W}$ is a compact subset of $H^{q_1}(\Omega)$ and
$X$ is subspace of $H^{q_2}(\Omega)$ with the energy norm $||\cdot||_X$ continuously controlled by $\|\cdot\|_{H^{q_2}(\Omega)}$ and $\Omega$ a bounded open domain. 
\end{assumption}
We first establish a universal approximation result for DeepONet in the Sobolev space. To do this, we begin with the following lemma.

\begin{lemma}[\cite{hornik1991approximation}]\label{Lemma:Universal approximation theorem_function}
	Suppose $0 \le q < \infty$ and $\Omega \subset \mathbb{R}^d$ is a bounded open domain. If $\sigma \in C^q(\mathbb{R})$ is non-constant and bounded, then the space
	$$
	V^\sigma  = \bigcup_{N=1}^{\infty}
	\left\{ f_{N} : \mathbb{R}^d \to \mathbb{R} \,\middle|\, f_{N}(\bx) = \sum_{j=1}^N \alpha_j \sigma(\boldsymbol{w}_j \cdot \bx + z_j),\ \forall \alpha_j,z_j\in\mathbb{R},\, \boldsymbol{w}_j\in\mathbb{R}^d \right\}.
	$$
	is dense in the Sobolev space
	$
	H^{q}(\Omega).
	$
\end{lemma}
Based on the neural-network approximation for functions in Sobolev spaces in Lemma \ref{Lemma:Universal approximation theorem_function}, we establish the following universal approximation in Sobolev space for operator networks.

\begin{lemma}[Universal approximation for stacked DeepONet in Sobolev norm]\label{Lemma:Universal approximation theorem_Sobolev}
Let Assumption \ref{ass:DGNO} hold. 
	If $\sigma \in C^{q_2}(\mathbb{R})$ is non-constant and bounded, then for any $\epsilon > 0$, there exists a stacked DeepONet 
    $\mathcal{G}_\theta (f)(\bx):=
\sum_{l=1}^p
\mathfrak{b}_{l,\theta_\mathfrak{b}}(\mathbf{c}_K(f))\,
\mathfrak{t}_{l,\theta_\mathfrak{t}}(\bx),$
    such that:
	\[
	 \sup_{f \in \mathcal{W}} \lVert \mathcal{G}(f) - \mathcal{G}_\theta(f) \rVert_{H^{q_2}(\Omega)} < \epsilon.
	\]
 Here, $\bx \in \Omega$, 
 $	\mathfrak{b}_{l,\theta_\mathfrak{b}}(\mathbf{c}_K(f))
		:=
		\sum_{k=1}^{N_l}
		\alpha_k^l
		\sigma(\boldsymbol{w}_k^l \cdot \mathbf{c}_K(f) + z_k^l)$,
$\mathfrak{t}_{l,\theta_\mathfrak{t}}(\bx)
		:=
		\sum_{j=1}^{M_l}
		\alpha_j^l
		\sigma(\boldsymbol{w}_j^l \cdot \bx + z_j^l)$.
        The coefficient vector $
	\mathbf{c}_K(f) = (\langle f, \zeta_1 \rangle_{H^{q_1}(\Omega)}, \dots, \langle f, \zeta_K \rangle_{H^{q_1}(\Omega)}) \in \mathbb{R}^K$, where ${\{\zeta_k\}}_{k=1}^{\infty}$  is an orthonormal basis of $H^{q_1}(\Omega)$.
\end{lemma}
\begin{proof}	
	Define the projection operator $P_K : H^{q_1}(\Omega) \to H^{q_1}(\Omega)$ by
	$
	P_K f := \sum_{k=1}^K \langle f, \zeta_k \rangle_{H^{q_1}} \, \zeta_k.
	$
    Since $\mathcal{W} \subset H^{q_1}(\Omega)$ is compact and
$\{P_K\}_{K\in\mathbb N}$ is a sequence of uniformly bounded projections, we have 
$\sup_{f \in \mathcal{W}} \|P_K f - f\|_{H^{q_1}(\Omega)} \to 0 \; \text{as} \; K \to \infty.$
Moreover, since $\mathcal G$ is continuous on the compact set $\mathcal W$,
it is uniformly continuous. Therefore, for any $\epsilon>0$, there exists
$\delta>0$ such that
$
\|f-P_K f\|_{H^{q_1}(\Omega)}<\delta
\Longrightarrow
\|\mathcal G(f)-\mathcal G(P_K f)\|_{H^{q_2}(\Omega)}
<
{\epsilon}/{4}.
$
Choosing $K$ sufficiently large so that
$
\sup_{f\in\mathcal W}
\|P_Kf-f\|_{H^{q_1}(\Omega)}
<
\delta,
$
we obtain
\[
\sup_{f\in\mathcal W}
\|\mathcal G(f)-\mathcal G(P_Kf)\|_{H^{q_2}(\Omega)}
<
\frac{\epsilon}{4}.
\]

	Let $\{\eta_l\}_{l=1}^\infty$ be an unit orthonormal basis of $H^{q_2}(\Omega)$.
	Since $P_K(\mathcal{W})$ is compact in $H^{q_1}(\Omega)$ and $\mathcal{G}$ is continuous, the set $\mathcal{G}(P_K(\mathcal{W}))$ is compact in $H^{q_2}(\Omega)$. Similarly, there exists a positive integer $p$ such that
	\[
	\sup_{f \in \mathcal{W}} \Bigl\lVert \mathcal{G}(P_Kf) - \sum_{l=1}^p \langle \mathcal{G}(P_Kf), \eta_l \rangle_{H^{q_2}}\, \eta_l \Bigr\rVert_{H^{q_2}(\Omega)} < \frac{\epsilon}{4}.
	\]
    
	For each $l=1,\ldots,p$, define 
$\mathcal{B}_l(\mathbf c_K(f))
:=
\langle \mathcal G(P_Kf),\eta_l\rangle_{H^{q_2}},
\; f\in\mathcal W.
$
Since $P_Kf$
is uniquely determined by
$
\mathbf c_K(f)
=
\bigl(
\langle f,\zeta_1\rangle_{H^{q_1}},
\ldots,
\langle f,\zeta_K\rangle_{H^{q_1}}
\bigr),
$
the above definition is well posed. Moreover, $\mathcal{B}_l$ is continuous on
$\mathbf c_K(\mathcal W)$ as it is the composition of continuous mappings.
Since $\mathcal W$ is compact in $H^{q_1}(\Omega)$, $\mathbf c_K(\mathcal W)$ is a compact subset of $\mathbb R^K$. By the universal approximation theorem for continuous functions \cite{Chen1995}, for each $l=1,\ldots,p$, there exists a neural network
$
\mathfrak{b}_{l,\theta_\mathfrak{b}}(\mathbf c_K(f))
=
\sum_{k=1}^{N_l}
\alpha_k^l
\sigma(\boldsymbol{w}_k^l\cdot \mathbf c_K(f)+z_k^l),
$
such that
\[
\sup_{f\in\mathcal W}
\left|
\mathcal{B}_l(\mathbf c_K(f))
-
\mathfrak{b}_{l,\theta_\mathfrak{b}}(\mathbf c_K(f))
\right|
<
\frac{\epsilon}{4\sqrt p}.
\]
Since $p$ is finite, these approximations can be chosen simultaneously for all
$l=1,\ldots,p$.
     Then, 
	 \begin{align*}
	 \sup_{f \in \mathcal{W}} \Bigl\lVert \sum_{l=1}^p \bigl(\mathcal{B}_l(\mathbf{c}_K(f)) - \mathfrak{b}_{l,\theta_\mathfrak{b}}(\mathbf{c}_K(f))\bigr)\eta_l \Bigr\rVert_{H^{q_2}(\Omega)}
	 &= \sup_{f \in \mathcal{W}}
	 \left(
	 \sum_{l=1}^p |\mathcal{B}_l(\mathbf{c}_K(f)) - \mathfrak{b}_{l,\theta_\mathfrak{b}}(\mathbf{c}_K(f))|^2
	 \right)^{1/2}\\
	 &\le
	 	\left(\sum_{l=1}^p
	 	\sup_{f \in \mathcal{W}}
	 	|\mathcal{B}_l(\mathbf{c}_K(f)) - \mathfrak{b}_{l,\theta_\mathfrak{b}}(\mathbf{c}_K(f))|^2\right)^{1/2}
	 <
	 \frac{\epsilon}{4}.
	 \end{align*}

	Now, define the network
	$
	\mathfrak{t}_{l,\theta_\mathfrak{t}}(\bx) := \sum_{j=1}^{M_l} \alpha_j^l \sigma(\boldsymbol{w}_j^l \cdot \bx + z_j^l),
	\, 1 \le l \le p. 
	$
    Since $p$ is finite, for each $1 \le l \le p$, by the universal approximation from Lemma \ref{Lemma:Universal approximation theorem_function}, there exists $M_l$ such that 
	$
	\|\eta_l - \mathfrak{t}_{l,\theta_\mathfrak{t}}\|_{H^{q_2}(\Omega)} < \delta_0
	$.
	Moreover, since $\mathbf c_K(\mathcal W)$ is compact and $\mathfrak{b}_{l,\theta_\mathfrak{b}}$ is continuous, there exists a constant $C_0>0$ such that
	$
	|\mathfrak{b}_{l,\theta_\mathfrak{b}}(\mathbf{c}_K(f))| \le C_0, \; \forall f \in \mathcal{W}, \; 1 \le l \le p.
$	
	Defining $\mathcal{G}_\theta(f)(\bx):=\sum_{l=1}^p \mathfrak{b}_{l,\theta_\mathfrak{b}}(\mathbf{c}_K(f)) \mathfrak{t}_{l,\theta_\mathfrak{t}}(\bx)$ and by the triangle inequality, we have
	\begin{align*}
		\sup_{f \in \mathcal{W}} \Bigl\lVert \sum_{l=1}^p \mathfrak{b}_{l,\theta_\mathfrak{b}}(\mathbf{c}_K(f))\eta_l 
		- \mathcal{G}_\theta(f) \Bigr\rVert_{H^{q_2}(\Omega)} 
		&\le \sup_{f \in \mathcal{W}} \Big(\sum_{l=1}^p |\mathfrak{b}_{l,\theta_\mathfrak{b}}(\mathbf{c}_K(f))| \cdot \|\eta_l - \mathfrak{t}_{l,\theta_\mathfrak{t}}\|_{H^{q_2}(\Omega)} \Big) \\
		& \le C_0 \sum_{l=1}^p \|\eta_l - \mathfrak{t}_{l,\theta_\mathfrak{t}}\|_{H^{q_2}(\Omega)}  
		\le C_0 p \delta_0.
	\end{align*}
	By choosing $\delta_0 = \epsilon/(4C_0p)$, we obtain
	$
	\Bigl\lVert \sum_{l=1}^p \mathfrak{b}_{l,\theta_\mathfrak{b}}(\mathbf{c}_K(f))\eta_l 
	- \mathcal{G}_\theta(f) \Bigr\rVert_{H^{q_2}(\Omega)} < \epsilon/{4}.
	$ 
	 Then we have
	\begin{align*}
		\sup_{f \in \mathcal{W}} \lVert \mathcal{G}(f) - \mathcal{G}_\theta(f) \rVert_{H^{q_2}(\Omega)}
		&\le \sup_{f \in \mathcal{W}}\Bigl(\lVert \mathcal{G}(f) - \mathcal{G}(P_Kf) \rVert_{H^{q_2}(\Omega)} 
		 + \Bigl\lVert \mathcal{G}(P_Kf) - \sum_{l=1}^p \mathcal{B}_l(\mathbf{c}_K(f)) \eta_l \Bigr\rVert_{H^{q_2}(\Omega)} \\
		 +& \Bigl\lVert \sum_{l=1}^p [\mathcal{B}_l(\mathbf{c}_K(f)) - \mathfrak{b}_{l,\theta_\mathfrak{b}}(\mathbf{c}_K(f))] \eta_l \Bigr\rVert_{H^{q_2}(\Omega)} + 
		\Bigl\lVert \sum_{l=1}^p \mathfrak{b}_{l,\theta_\mathfrak{b}}(\mathbf{c}_K(f))\eta_l 
		- \mathcal{G}_\theta(f) \Bigr\rVert_{H^{q_2}(\Omega)} \Big) \\
		&< \epsilon.
	\end{align*}
\end{proof}

Lemma \ref{Lemma:Universal approximation theorem_Sobolev} only considers DeepONet constructed from stacked shallow networks. 
The following extends the result by allowing for more general branch and trunk networks.


\begin{lemma}[Generalized universal approximation for DeepONet in Sobolev norms]\label{Th:Universal approximation theorem_Sobolev}
	Under Assumption \ref{ass:DGNO}, for any $\epsilon > 0$,  there exist positive integers $K, p$, continuous functionals
$c_1,c_2,\ldots,c_K:H^{q_1}(\Omega)\to\mathbb{R}$, and continuous vector functions $\mathfrak{b}: \mathbb{R}^K \to \mathbb{R}^p$, $\mathfrak{t}: \mathbb{R}^d \to \mathbb{R}^p$,  such that
	\begin{equation}\label{Universal_sobolev}
		\sup_{f \in \mathcal{W}} \lVert \mathcal{G}(f) - \underbrace{\mathfrak{b}(c_1(f), c_2(f), \ldots, c_K(f)) }_{\text{branch}} \cdot \underbrace{\mathfrak{t} }_{\text{ trunk}}  \rVert_{H^{q_2}(\Omega)}
		 < \epsilon.
	\end{equation}
The functions $\mathfrak{b}$ and $\mathfrak{t}$ can be chosen as diverse classes 
of neural networks which satisfy the classical universal approximation
 theorem of functions, including (stacked/unstacked) fully connected
 neural networks and convolutional neural networks.
\end{lemma}
\begin{proof}
	Let $\mathfrak{b} = (\mathfrak{b}_1, \cdots, \mathfrak{b}_p)^T$ and $\mathfrak{t} = (\mathfrak{t}_1, \cdots, \mathfrak{t}_p)^T$. We choose
	$
	\mathfrak{b}_l(\mathbf{c}_K(f)) = \sum_{k=1}^{N_l} \alpha_k^l \sigma(\boldsymbol{w}_k^l \cdot \mathbf{c}_K(f) + z_k^l)
	$, and
	$
	\mathfrak{t}_l(\bx) = \sum_{j=1}^{M_l} \alpha_j^l \sigma(\boldsymbol{w}_j^l \cdot \bx + z_j^l)
	$
	for $\bx \in \mathbb{R}^d$, where all the notations involved are defined in Lemma~\ref{Lemma:Universal approximation theorem_Sobolev}. Then, this result holds immediately following Lemma~\ref{Lemma:Universal approximation theorem_Sobolev}.
	
	Note that the domains of $\mathfrak{b}$ and $\mathfrak{t}$ can be restricted on compact sets. In fact, $\mathfrak{t}$ is defined on $\Omega$, and $\mathfrak{b}$ can be defined on $[-Z, Z]^K$, where $Z > 0$ is a uniform bound of the functions in $\mathcal{W}$. Therefore, $\mathfrak{b}$ and $\mathfrak{t}$ can be approximated by neural networks according to the universal approximation of functions \cite{Chen1995}.
\end{proof}

Now, we are ready to present a convergence result for the proposed DGNO method.

\begin{theorem}\label{thm:dgno}
 Under Assumption~\ref{ass:DGNO},  fix one $0<\tau<1$ and assume that the training and sampling within DGNO are sufficient at each stage so that the resulting DeepONet basis $\mathcal{G}_{\theta_s}$ for $s=1,2,\ldots S$ can meet \eqref{Universal_sobolev} with $\epsilon=\tau$. Then,
     the following error estimates hold. 
    \begin{itemize}
        \item[1)] The error of DGNO is strictly decreasing during construction of each stage, i.e., $\|e_s\|_{L^1(\mathcal{W};X)} < \|e_{s-1}\|_{L^1(\mathcal{W};X)}$, and the error at the $s$-stage satisfies the bound 
	\begin{equation}\label{eq:error_bound_DGNO}
	    \|e_{s}\|_{L^1(\mathcal{W};X)} < \tau^s \, \|e_{0}\|_{L^1(\mathcal{W};X)} .
	\end{equation}
        \item[2)] Define $\mathcal{J}_{s-1}(\mathcal{G}_{\theta_s}):=\int_{\mathcal{W}}  R_{s-1}(f,\mathcal{G}_{\theta_s}(f)) \, d\mu(f),$
    then the error at stage $s-1$ of DGNO satisfies the bound
	\begin{equation}\label{eq:error_bound_posterior}
	    \frac{1}{1+\tau}\,\mathcal{J}_{s-1}(\mathcal{G}_{\theta_s}) < \|e_{s-1}\|_{L^1(\mathcal{W};X)} <  \frac{1}{1-\tau}\, \mathcal{J}_{s-1}(\mathcal{G}_{\theta_s}).
	\end{equation}
    \end{itemize}
    
\end{theorem}
\begin{proof}
By the definition \eqref{eq:residual operator}, $\varphi_{s-1}$ is a continuous operator from $\mathcal{W} \subset H^{q_1}(\Omega)$ to $X \subset H^{q_2}(\Omega)$. Under the assumption and also 
by Lemma \ref{Th:Universal approximation theorem_Sobolev}, the basis operator network generated within DGNO at stage $s$, i.e., $\mathcal{G}_{\theta_s}\in \mathcal{D}$ for $1\leq s\leq S$, can approximate the  $\varphi_{s-1}$ to a range:
\begin{equation}\label{eq:approximate_basis_DGNO}
    \|\mathcal{G}_{\theta_s}(f) -\varphi_{s-1}(f)\|_X < \tau, \quad \forall f \in \mathcal{W}.
\end{equation}
	For any fixed $f\in\mathcal{W}$, note that $\mathcal{G}_{s-1}+\|e_{s-1}(f)\|_X \mathcal{G}_{\theta_s}(f) \in 
\mathrm{span}\{\mathcal{G}_{\theta_j}(f)\}_{j=1}^s.   
    $ Hence, by the property of Galerkin projection (see also \cite{GNN}), we have
\begin{equation}\label{eq:error_decay_pointwise}
\begin{aligned}
    \|e_{s}(f)\|_X = \|\mathcal{G}(f) - \mathcal{G}_s(f)\|_X \leq
& \|\mathcal{G}(f) - \mathcal{G}_{s-1}(f) - \|e_{s-1}(f)\|_X \mathcal{G}_{\theta_s}(f)\|_X \\
=& \| \|e_{s-1}(f)\|_X \varphi_{s-1}(f) - \|e_{s-1}(f)\|_X \mathcal{G}_{\theta_s}(f)\|_X \\
=& \|e_{s-1}(f)\|_X \|\varphi_{s-1}(f) - \mathcal{G}_{\theta_s}(f)\|_X \\
<& \tau \|e_{s-1}(f)\|_X.
\end{aligned}
\end{equation}
where the final step follows \eqref{eq:approximate_basis_DGNO}. 
Integrating over $\mathcal{W}$ with respect $f$ under the measure $\mu$ immediately yields $\|e_s\|_{L^1(\mathcal{W};X)} <\tau\,\|e_{s-1}\|_{L^1(\mathcal{W};X)}$.
    As $\tau<1$, the error strictly decreases. Applying this estimate recursively, we obtain \eqref{eq:error_bound_DGNO}.

By the definition of $R_{s-1}$, for every $ f \in \mathcal{W}$, $R_{s-1}(f,\varphi_{s-1}(f)) = \|e_{s-1}(f)\|_X$. It is then direct to deduce that 
\begin{align*}
&\left| R_{s-1}\big(f,\mathcal{G}_{\theta_s}(f)\big)
      - \|e_{s-1}(f)\|_X \right|
=
\left| R_{s-1}\big(f,\mathcal{G}_{\theta_s}(f)\big)
      - R_{s-1}\big(f,\varphi_{s-1}(f)\big) \right| 
=
\left| R_{s-1}\big(f,\mathcal{G}_{\theta_s}(f) - \varphi_{s-1}(f)\big) \right| \\[4pt]
=& \left| F\big(f,\mathcal{G}_{\theta_s}(f) - \varphi_{s-1}(f)\big)-a\big(\mathcal{G}_{\theta_{s-1}}(f),\mathcal{G}_{\theta_s}(f) - \varphi_{s-1}(f)\big) \right|\\
=&
 \left| a\big(\mathcal{G}(f),\mathcal{G}_{\theta_s}(f) - \varphi_{s-1}(f)\big)-a\big(\mathcal{G}_{\theta_{s-1}}(f),\mathcal{G}_{\theta_s}(f) - \varphi_{s-1}(f)\big) \right|
=\left|a\big(e_{s-1}(f),\, \mathcal{G}_{\theta_s}(f) - \varphi_{s-1}(f)\big) \right|.
\end{align*}
By the Cauchy–Schwarz inequality and \eqref{eq:error_decay_pointwise}, we find
\begin{align*}
\left| R_{s-1}\big(f,\mathcal{G}_{\theta_s}(f)\big)
      - \|e_{s-1}(f)\|_X \right|
=&\left|a\big(e_{s-1}(f),\, \mathcal{G}_{\theta_s}(f) - \varphi_{s-1}(f)\big) \right|\\
\le&
\|e_{s-1}(f)\|_X \,
\|\mathcal{G}_{\theta_s}(f) - \varphi_{s-1}(f)\|_X < \tau \, \|e_{s-1}(f)\|_X ,
\end{align*}
Thus,
\[
(1-\tau)\, \|e_{s-1}(f)\|_X<R_{s-1}\big(f,\mathcal{G}_{\theta_s}(f)\big) < (1+\tau)\, \|e_{s-1}(f)\|_X,
\]
and integrating over $\mathcal{W}$ immediately yields \eqref{eq:error_bound_posterior}.
\end{proof}

\section{Application to convolution problems}\label{sec:5}

In this section, we apply the proposed DCNO and DGNO methods to solve the convolution problem \eqref{eq:convolution}.
We first detail the implementation of the methods for the convolution operator, followed by a series of numerical experiments to validate and compare accuracy and efficiency. Extensions to the multi-input case are made in the end.
All experiments\footnote{Code is available at {https://github.com/wzhy0777/DCNO-DGNO}} are conducted using a single CPU with 32 GB RAM for CPU-based methods and a single NVIDIA A800-PCIE-80GB GPU for GPU-based methods.

\subsection{Implementation of DCNO \& DGNO on convolution}\label{sec:2.3} 
Focusing on the convolution problem \eqref{eq:convolution}: $\phi(\bx) = \mathcal{G}(\rho)(\bx) = (K*\rho)(\bx)$, 
the target operator $\mathcal{G}$ maps the input density $\rho$ to the solution potential $\phi$. The generality of $\mathcal{G}$ can separate the use of DCNO from DGNO. When $\mathcal{G}$ does not have a precise (local) PDE formulation, then DGNO might not be applicable. In such a case, we would employ DCNO for supervised training using data pairs $\{(\rho^{(b)}, \phi^{(b)})\}_{b=1}^{N_b}$ that could be generated by traditional numerical convolution solvers. 
By sampling the input density function $\rho^{(b)}$, the solver produces accurate enough 
$\phi^{(b)}=\mathcal{G}(\rho^{(b)})$.  
After the training dataset is constructed, the learning of the convolution operator via DCNO can be performed as in Algorithm \ref{alg:DCNO}.

When $\mathcal{G}$ can be expressed through a proper PDE, i.e., \eqref{eq:strong form}, 
DGNO becomes available and its unsupervised training
can certainly be rewarded with more efficiency. 
The task of learning the operator $f \mapsto u$ in \eqref{eq:strong form} then follows the DGNO procedure in Algorithm~\ref{alg:DGNO}.
If the original convolution problem \eqref{eq:convolution} is imposed in the whole space: $\phi(\bx) = \int_{\mathbb{R}^d} K(\bx,\tilde\bx) \rho(\tilde\bx) \, d\tilde\bx$, 
the resulting equivalent whole PDE problem, for computational reasons, usually needs to be truncated onto a bounded domain $\Omega$ with appropriate/approximate boundary conditions. Careless truncations may cause loss of accuracy or require a very large domain size. 
To enlarge computational efficiency in such a case, we here combine with a moment-matching strategy \cite{liu2026fast} briefly presented below, which can accurately approximate the whole-space convolution as a local PDE with zero boundary conditions in a bounded domain. 

The moment-matching strategy begins with a decomposition \cite{bao2015computing,liu2026fast} on the density: $\rho=\rho_{\mathrm{mm}}+f$ with $\rho_{\mathrm{mm}}$ constructed for moment-matching and $f$ the residual. Here, $\rho_{\mathrm{mm}}$ is assumed to be 
	$\rho_{\mathrm{mm}}(\mathbf{x}) = \sum_{|\boldsymbol{\alpha}|=0}^{m} \gamma_{\boldsymbol{\alpha}} \partial^{\boldsymbol{\alpha}} G(\mathbf{x}),$ with  $G(\mathbf{x}) = (2\pi\gamma^2)^{-d/2} \exp(-\frac{|\mathbf{x}|^2}{2\gamma^2}),$
where $\boldsymbol{\alpha}$ the multi-index and the coefficients $\{\gamma_{\boldsymbol{\alpha}}\}$ are determined by imposing the moment-matching condition  up to order $m$:
\begin{equation}\label{eq:moment_matching}
	\int_{\mathbb{R}^d} \rho_{\mathrm{mm}}(\mathbf{x}) \mathbf{x}^{\boldsymbol{\alpha}} \, d\mathbf{x} = \int_{\mathbb{R}^d} \rho(\mathbf{x}) \mathbf{x}^{\boldsymbol{\alpha}} \, d\mathbf{x}, \quad \text{for } |\boldsymbol{\alpha}| = 0, 1, \dots, m.
\end{equation} 
Conditions in \eqref{eq:moment_matching} read equivalently as a linear system:
$
	\sum_{|\boldsymbol{\alpha}|=0}^{m}
\gamma_{\boldsymbol{\alpha}}\int_{\mathbb{R}^d} \partial^{\boldsymbol{\alpha}} G(\mathbf{x}) \mathbf{x}^{\boldsymbol{\beta}} \, d\mathbf{x} = \int_{\mathbb{R}^d} \rho(\mathbf{x}) \mathbf{x}^{\boldsymbol{\beta}} \, d\mathbf{x},$  for $|\boldsymbol{\beta}| = 0, 1,\dots, m,$ 
which admits an analytical solution \cite{liu2026fast}. 
The advantage of this construction is that  $\phi_{\mathrm{mm}} = K * \rho_{\mathrm{mm}}$ can be analytically evaluated and for the residual $u=\phi-\phi_{\mathrm{mm}} = K * f$, the condition \eqref{eq:moment_matching} ensures a rapid decay of $f$ at infinity. This in turn can lead to a rapid far-field decay also on $u$ for $K$ satisfying suitable far-field asymptotic behavior \cite{liu2026fast}, which is the case for Green's functions of common elliptic operators like the Laplacian. 
Then, the unbounded-domain problem can be truncated to a $\Omega$ not too large and $u$ can be accurately approximated as the solution of
\begin{equation}
	\label{eq:bvp_for_phi2}
	\begin{cases}
		\mathcal{L}u(\mathbf{x}) = f(\mathbf{x}), & \mathbf{x}\in\Omega,\\
		u(\mathbf{x}) = 0, & \mathbf{x}\in\partial\Omega.
	\end{cases}
\end{equation}
Coming back to the convolution operator learning task, with the sampled training set $\{\rho^{(b)}\}_{b=1}^{N_b}$, now under this strategy, we first compute the corresponding $\{f^{(b)}=\rho^{(b)}-\rho^{(b)}_{\mathrm{mm}}\}_{b=1}^{N_b}$ as inputs associated with \eqref{eq:bvp_for_phi2} and then 
DGNO in Algorithm~\ref{alg:DGNO} can be applied for $\mathcal{G}_S(f)\approx u$. Thus, for any given $\rho$, the final prediction for the convolution outcome $\phi$ under our method is $\phi_{\mathrm{mm}}+\mathcal{G}_S(f)$.

By the structure of DeepONet \eqref{DeepONet}, 
the trained $\mathcal{G}_S$ from DCNO or DGNO can give the predicted solution at $N\in\mathbb{N}_+$ arbitrary spatial points with computational complexity $\mathcal{O}(N)$. In comparison, most existing traditional solvers such as the hierarchical matrix technique \cite{Hierarchicalmatrices}, the Gaussian-Summation method \cite{GauSum} and the kernel truncation method \cite{Liu2024,KTM} cost $\mathcal{O}(N\log N)$ for the prescribed $N$ grid points. For a new spatial point outside the prescribed grid, traditional solvers will need to do extra delicate interpolations, e.g., the NUFFT \cite{NUFFT}, if one wants to maintain high accuracy. For the traditional method SOE \cite{zhang2021fast}, which is specially designed to achieve $\mathcal{O}(N)$ cost for one-dimensional convolution, we will provide a numerical comparison below.


\subsection{Numerical experiments}
We now carry out some numerical experiments to illustrate the performance of DCNO and DGNO. The involved DeepONet \eqref{DeepONet} 
uses identical depth $L$ and width $W$ for both branch and trunk networks, with the same number of neurons in each hidden layer for simplicity.

\begin{example}[Poisson potential]\label{Ex:Poisson}
	We begin with a 2D Poisson potential in \eqref{eq:convolution}, i.e., $d=2,\ \mathbf{x}=(x,y)\in\mathbb{R}^2$ and the convolution kernel
	$
	K(x,y) = -\frac{1}{2\pi} \ln \sqrt{x^2 + y^2}. 
	$
	We consider the density function of form
	$
	\rho(x, y) = \mathrm{e}^{-(x^2 + y^2)/\alpha^2}
	$ with $\alpha>0$, and 
	the corresponding analytical solution of this problem is given by
	$$
	\phi(x, y)
	= -\frac{\alpha^2}{4}\left[-\mathrm{Ei}\!\left(-\frac{x^2 + y^2}{\alpha^2}\right) 
	+ \ln(x^2 + y^2)\right],\quad \mbox{satifying}\ -\Delta \phi=\rho\ \mbox{or}\ \phi=K*\rho\ \mbox{on}\ \mathbb{R}^2.
	$$
	Here, $\mathrm{Ei}(r)= \int_{-\infty}^{r} t^{-1} \mathrm{e}^{t} \mathrm{d}t$.
	Note that the limiting value at the origin is finite:
	$
	\phi(0, 0) = \frac{\alpha^2}{4}\left(\gamma_E - 2\ln\alpha\right)
	$
	with $\gamma_E \approx 0.5772156649$ the Euler--Mascheroni constant.
	Our goal is to learn the operator $\mathcal{G}: \rho \mapsto \phi$ and the computational domain is set as $\Omega = (-I, I)^2$ with $I=12$.
\end{example}

\begin{figure}[h!]
	\centerline{
		\psfig{figure=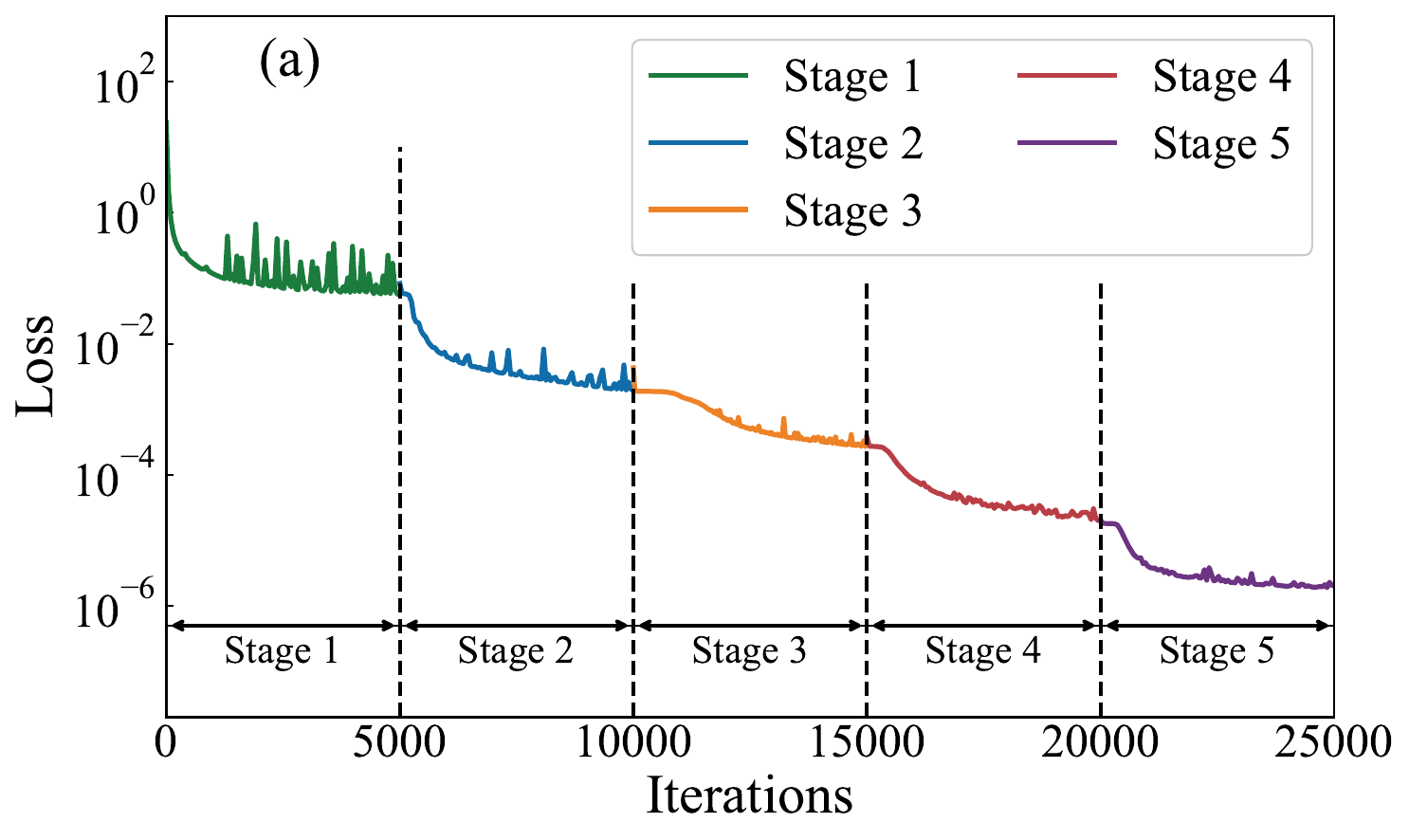,width=3in,height=1.8in,angle=0}
		\psfig{figure=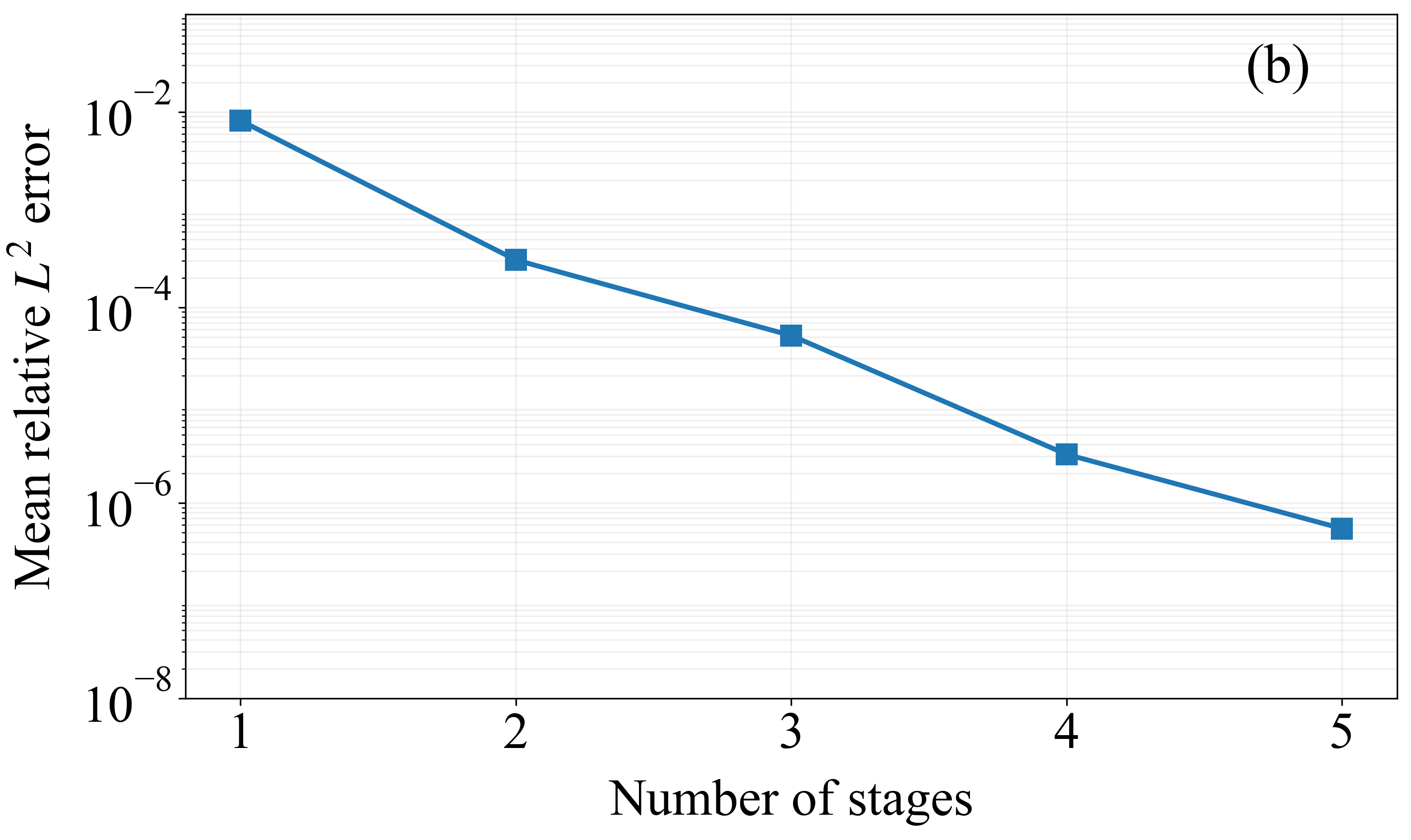,width=3in,height=1.8in,angle=0}
	}
	
	\caption{Results of DCNO in Algorithm \ref{alg:DCNO} for Example \ref{Ex:Poisson}: (a) Change of Loss$=\beta_{s-1}\sqrt{\text{Loss}_s}$ during the training; (b) Mean relative $L^2$ error with respect to the number of stages on test set.}
	\label{fig:poisson_supervised_error}
\end{figure}

\begin{table*}[htbp]
	\centering
	\caption{Relative $L^2$-errors obtained by DCNO and DGNO under different network architectures.}
	\label{tab:network_architecture}
	\begin{tabular}{c|cccc|cccc}
		\hline
		&
		\multicolumn{4}{c|}{DCNO}
		&
		\multicolumn{4}{c}{DGNO}
		\\
		\cline{2-9}
		&
		$W_0=10$ & $W_0=30$ & $W_0=50$ & $W_0=70$
		&
		$W_0=10$ & $W_0=30$ & $W_0=50$ & $W_0=70$
		\\
		\hline
		$L=1$
		& 1.75$\,e\!-\!3$
		& 1.81$\,e\!-\!3$
		& 1.20$\,e\!-\!3$
		& 5.37$\,e\!-\!4$
		& 2.99$\,e\!-\!3$
		& 2.15$\,e\!-\!3$
		& 1.27$\,e\!-\!3$
		& 6.70$\,e\!-\!4$
		\\
		
		$L=2$
		& 1.94$\,e\!-\!5$
		& 1.80$\,e\!-\!5$
		& 1.15$\,e\!-\!5$
		& 9.46$\,e\!-\!6$
		& 5.80$\,e\!-\!5$
		& 3.17$\,e\!-\!5$
		& 2.38$\,e\!-\!5$
		& 1.62$\,e\!-\!5$
		\\
		
		$L=3$
		& 4.01$\,e\!-\!6$
		& 3.17$\,e\!-\!6$
		& 1.48$\,e\!-\!6$
		& 8.29$\,e\!-\!7$
		& 1.68$\,e\!-\!5$
		& 6.03$\,e\!-\!6$
		& 3.47$\,e\!-\!6$
		& 9.20$\,e\!-\!7$
		\\
		
		$L=4$
		& 1.61$\,e\!-\!6$
		& 1.24$\,e\!-\!6$
		& 7.04$\,e\!-\!7$
		& 5.47$\,e\!-\!7$
		& 1.43$\,e\!-\!5$
		& 3.10$\,e\!-\!6$
		& 1.89$\,e\!-\!6$
		& 6.17$\,e\!-\!7$
		\\
		\hline
	\end{tabular}
\end{table*}

\begin{figure}[h!]
	\centerline{
		\psfig{figure=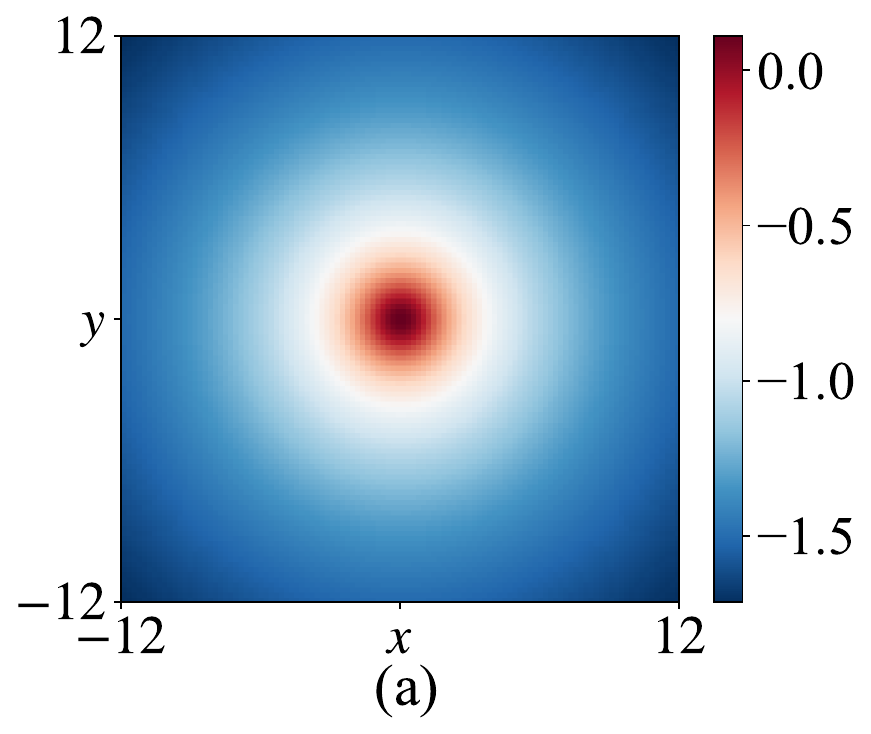,width=1.9in,height=1.5in,angle=0}
		\hspace{0in}
		\psfig{figure=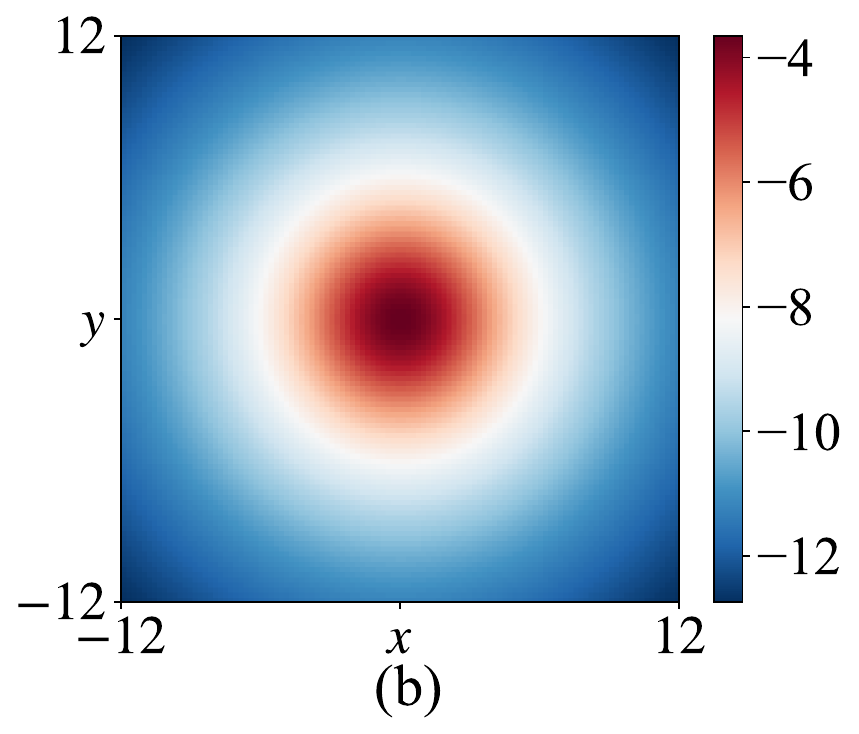,width=1.9in,height=1.5in,angle=0}
		\hspace{0in}
		\psfig{figure=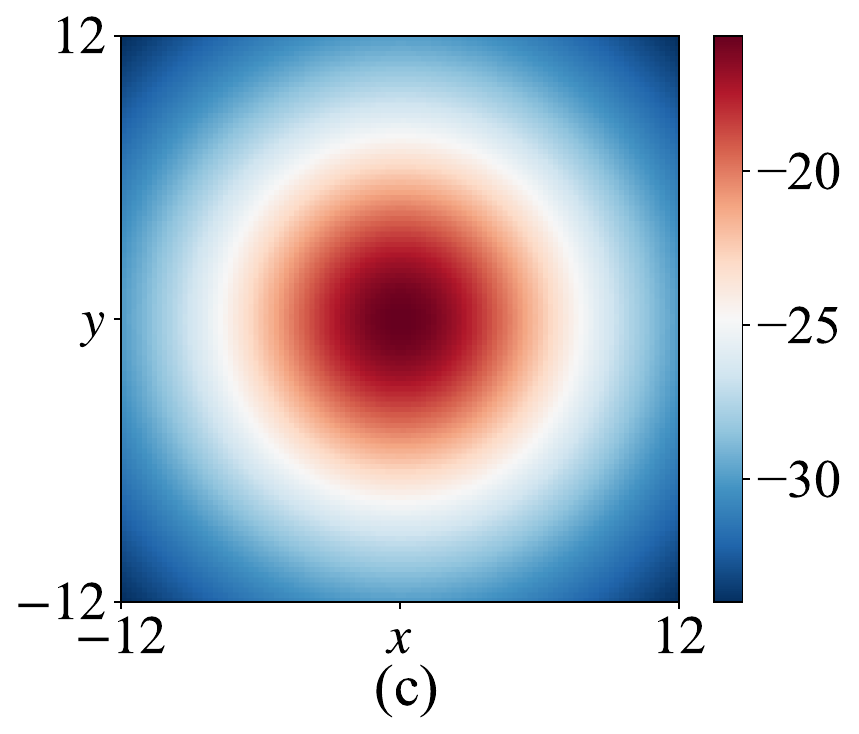,width=1.9in,height=1.55in,angle=0}
		
	}
	\centerline{
		\psfig{figure=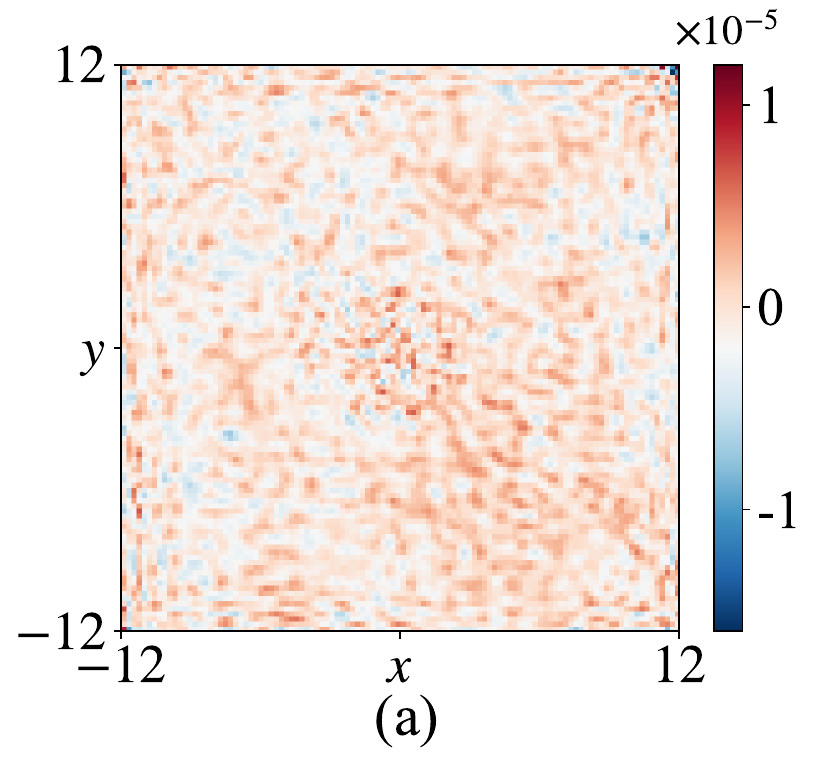,width=1.9in,height=1.5in,angle=0}
		\hspace{0in}
		\psfig{figure=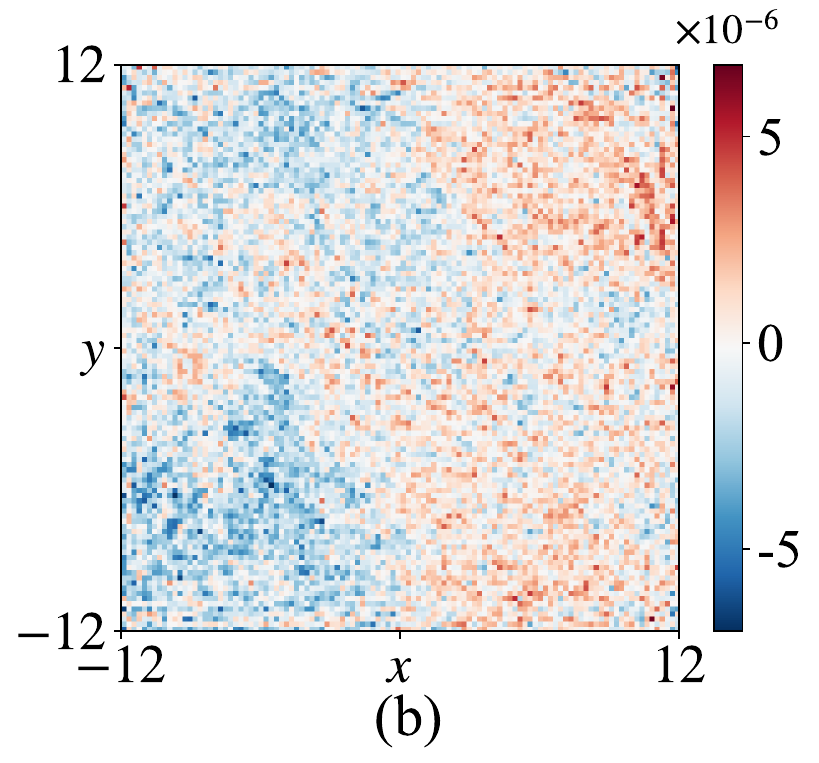,width=1.9in,height=1.5in,angle=0}
		\hspace{0in}
		\psfig{figure=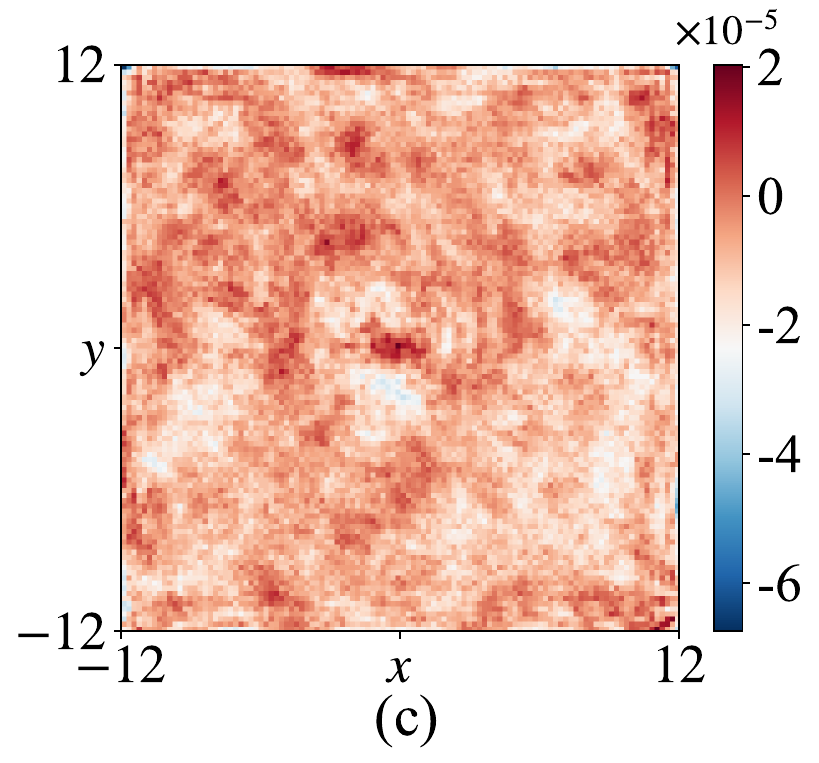,width=1.9in,height=1.55in,angle=0}
		
	}

	\caption{Numerical solution (1st row) and pointwise error (2nd row) of DCNO for Example \ref{Ex:Poisson} on three test samples (a) $\alpha=1$; (b) $\alpha=3$; (c) $\alpha=5$.}\label{fig:poisson_supervised_solution}
\end{figure}

\begin{figure}[t!]
	\centerline{
		\psfig{figure=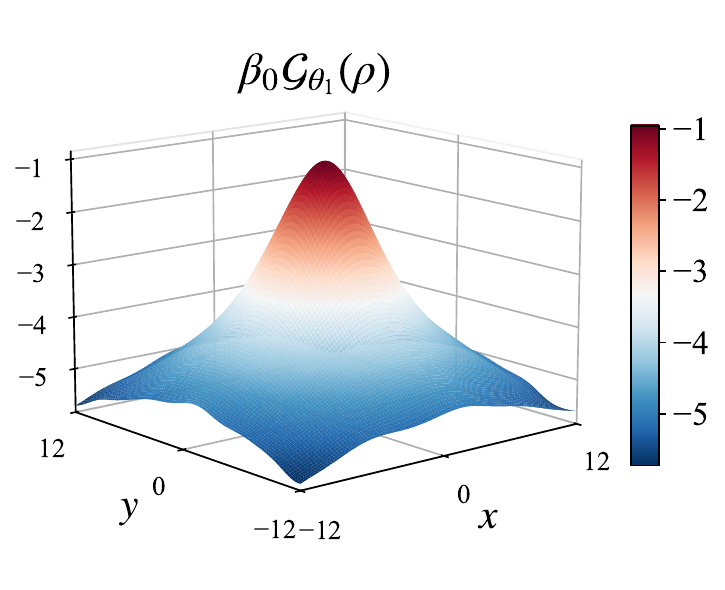,width=2in,height=1.7in,angle=0}
		\hspace{0in}
		\psfig{figure=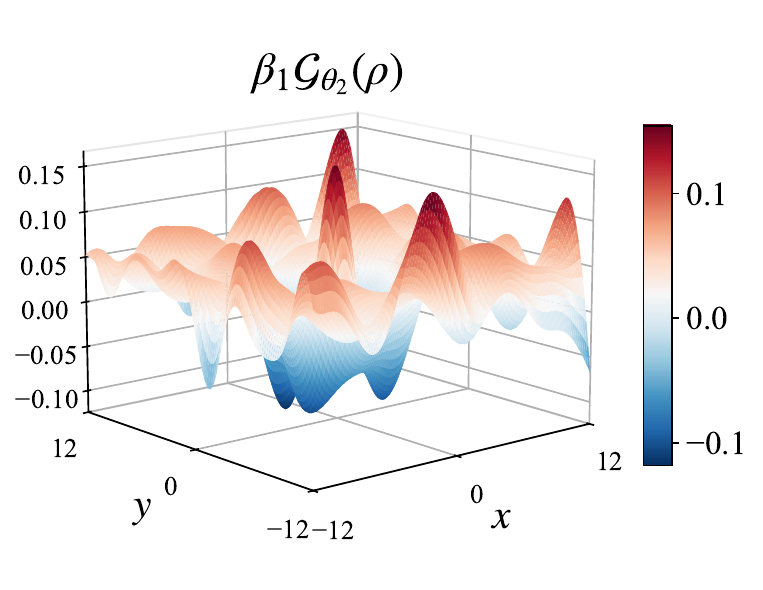,width=2in,height=1.7in,angle=0}
		\hspace{0in}
		\psfig{figure=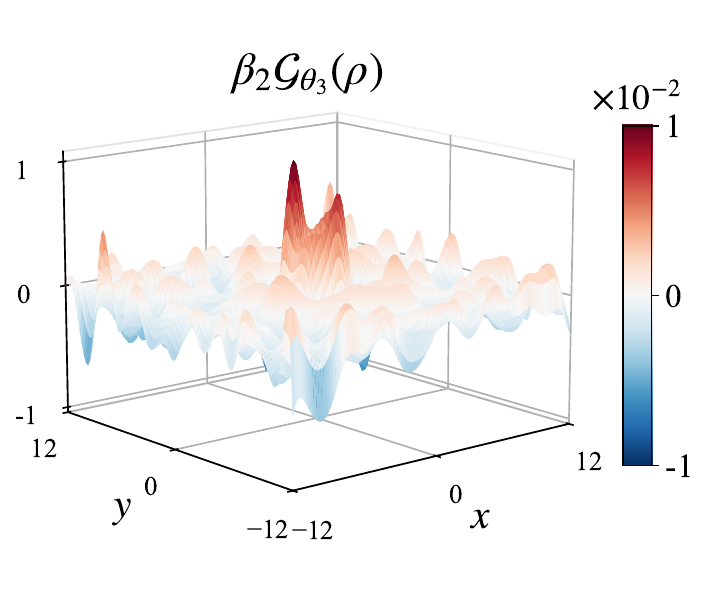,width=2in,height=1.7in,angle=0}
		
	}
	\centerline{
		\psfig{figure=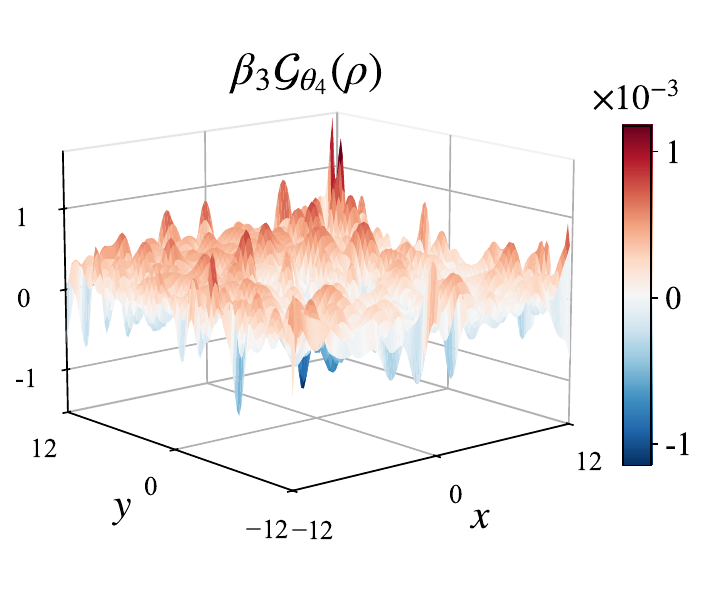,width=2in,height=1.7in,angle=0}
		\hspace{0.1in}
		\psfig{figure=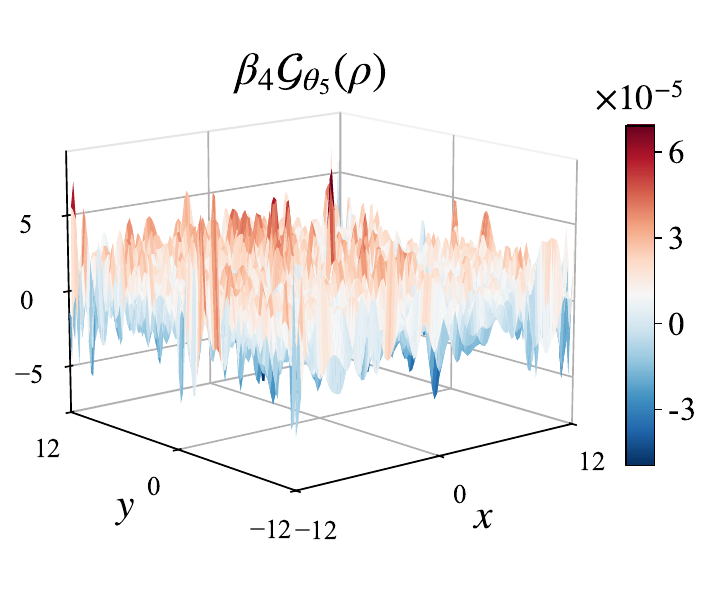,width=2in,height=1.7in,angle=0}
	}
	
	\caption{Basis functions generated by the basis operators of DCNO method.}\label{fig:basis function}
\end{figure}

We first test the performance of DCNO in the supervised scenario by applying Algorithm \ref{alg:DCNO} to Example \ref{Ex:Poisson}.
For training, a dataset $\{ (\rho^{(b)}, \phi^{(b)}) \}_{b=1}^{N_b}$ is generated with samples $N_b=100$, where each parameter $\alpha^{(b)}$ of $\rho^{(b)}$ is independently sampled from the uniform distribution $\mathcal{U}[1,5]$. 
Motivated by the strategy in \cite{GNN,MLNN,MSNN,DCM}, we take the width parameters $W$ and $p$ of the DeepONet \eqref{DeepONet} in a stage-dependent manner: $W = W_0 + 20\,(s-1)$ and $p = 120 + 20\,(s-1)$, where $s$ denotes the stage index.
The base width is selected as $W_0 \in \{10,30,50,70\}$, and the network depth is chosen as $L \in \{1,2,3,4\}$.
The input coordinates of the trunk network are sampled from a uniform spatial grid defined by
$x_k = -I+  2kI/N_x, k = 1, \dots, N_x$, $y_k = -I+  2kI/N_y, k = 1, \dots, N_y$, where $N_x=N_y=100$  is used during training and a finer grid with $N_x=N_y=150$ is adopted for testing error.
The activation function is chosen as $\sigma=\sin$, and the learning rate is set to $5 \times 10^{-4}$.
With $L=4$ and $W_0=70$ fixed, the stopping criterion in Algorithm \ref{alg:DCNO} is set to $\epsilon = 2 \times 10^{-6}$. 
Fig.~\ref{fig:poisson_supervised_error}(a) shows the evolution of $\beta_{s-1}\sqrt{\text{Loss}_s}$ during training at each stage, with a decay observed as the number of training iterations increases. 

For testing, 50 samples disjoint from the training set are randomly selected. 
The corresponding solutions are predicted by DCNO, and the performance is evaluated by the mean relative $L^2$ error over all samples in the test set.
The errors of the numerical solution from DCNO under different depth and width are presented in Table \ref{tab:network_architecture}. It can be observed that increasing the network width and depth consistently reduces the error, indicating improved approximation of the numerical solution.
With $L=4$ and $W_0=70$ fixed, Fig.~\ref{fig:poisson_supervised_error}(b) illustrates a clear decrease in the test error as the number of basis operators increases, reaching eventually $5.47\times 10^{-7}$ which is near the machine precision under single-precision floating-point arithmetic. 
Notably, without the progressive construction of basis operators employed in DCNO, the accuracy at stage 1 corresponds to that of the standard DeepONet, which is only $8.17\times10^{-3}$. This demonstrates the significant advantage of DCNO in improving the accuracy of operator learning. 
Fig.~\ref{fig:poisson_supervised_solution} shows the DCNO numerical solutions for several test samples, along with the corresponding pointwise error distributions.
As observed, DCNO method can accurately predict the potential $\phi$ across  the entire computational domain.
For a fixed $\rho(x, y) = \mathrm{e}^{-(x^2 + y^2)/{4}}$, Fig.~\ref{fig:basis function} visualizes the basis functions $\{\mathcal{G}_{\theta_j}(\rho)\}_{j=1}^{5}$ generated by the basis operators. It is observed that, as the stage index $s$ increases, the basis functions tend to capture higher-frequency components.

Also, we test the robustness of DCNO in the presence of noisy training data here to validate the theoretical result \eqref{eq:error_bound_noise}. 
Independent Gaussian noises with relative noise levels ranging from $10^{-6}$ to $10^{-1}$ are added to the exact solution data, which serve as training labels. 
With $L=4$ and $W_0=70$ fixed and the total number of stages set to $S=5$, all other training settings remain the same as before. 
The left plot of Fig.~\ref{fig:poission_error_noise} presents the test error of DCNO versus the noise level. 
The results indicate that the prediction error increases monotonically with the noise amplitude, while remaining consistently below the noise level across all tested regimes spanning several orders. This demonstrates that DCNO achieves robust operator approximation under noisy supervision. 
For illustration, a Gaussian noise with level $0.001$ is shown in the middle plot of Fig.~\ref{fig:poission_error_noise},  together with the corresponding pointwise error of the predicted solution for the test sample $\rho(x,y)=\mathrm{e}^{-(x^2+y^2)}$ in the right plot. 
The outcomes in total support Theorem \ref{Prop:error_bound_all_Bochner}.

\begin{figure}[t!]

        \centering

\includegraphics[width=0.38\textwidth,height=3.9cm]{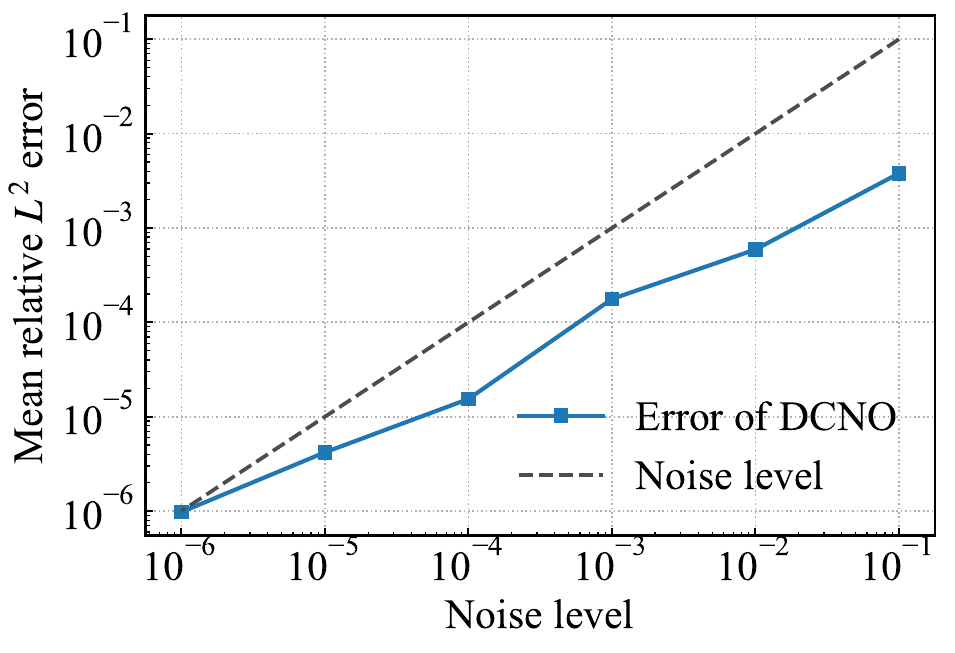}
\includegraphics[width=0.29\textwidth,height=4cm]{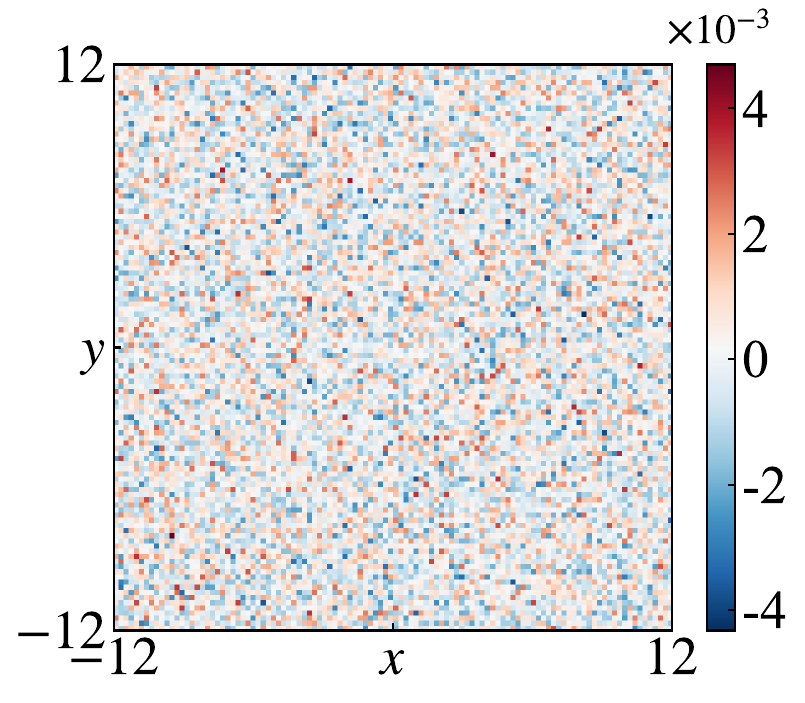}
\includegraphics[width=0.29\textwidth,height=4cm]{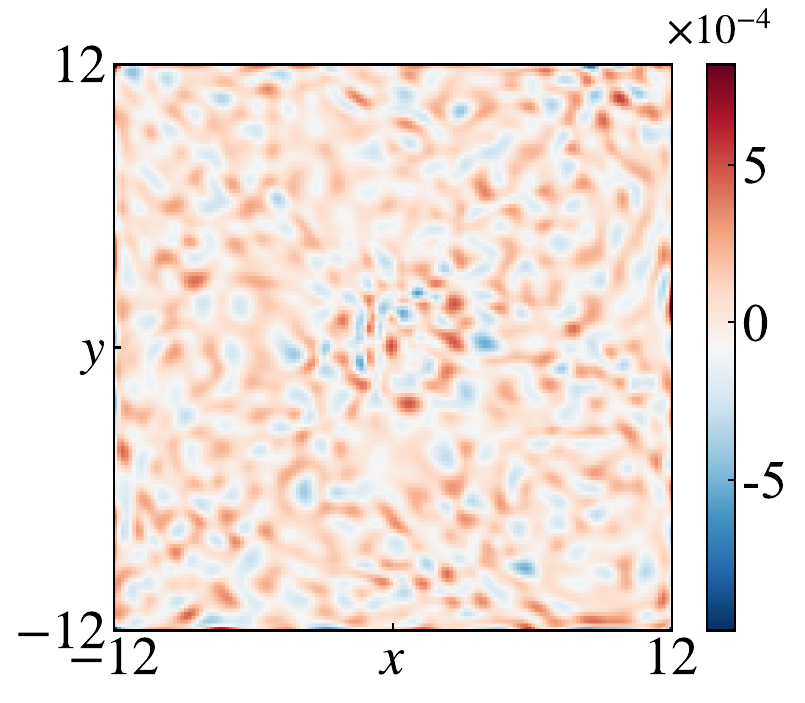}

	\caption{Relative $L^2$ error of DCNO under varying noise levels (left); A Gaussian noise sample with noise level $0.001$ (middle); Pointwise error of DCNO for Example \ref{Ex:Poisson} on a test sample with $\alpha=1$ (right).}
	\label{fig:poission_error_noise}
\end{figure}

\begin{figure}[t!]
	\centering
	\psfig{figure=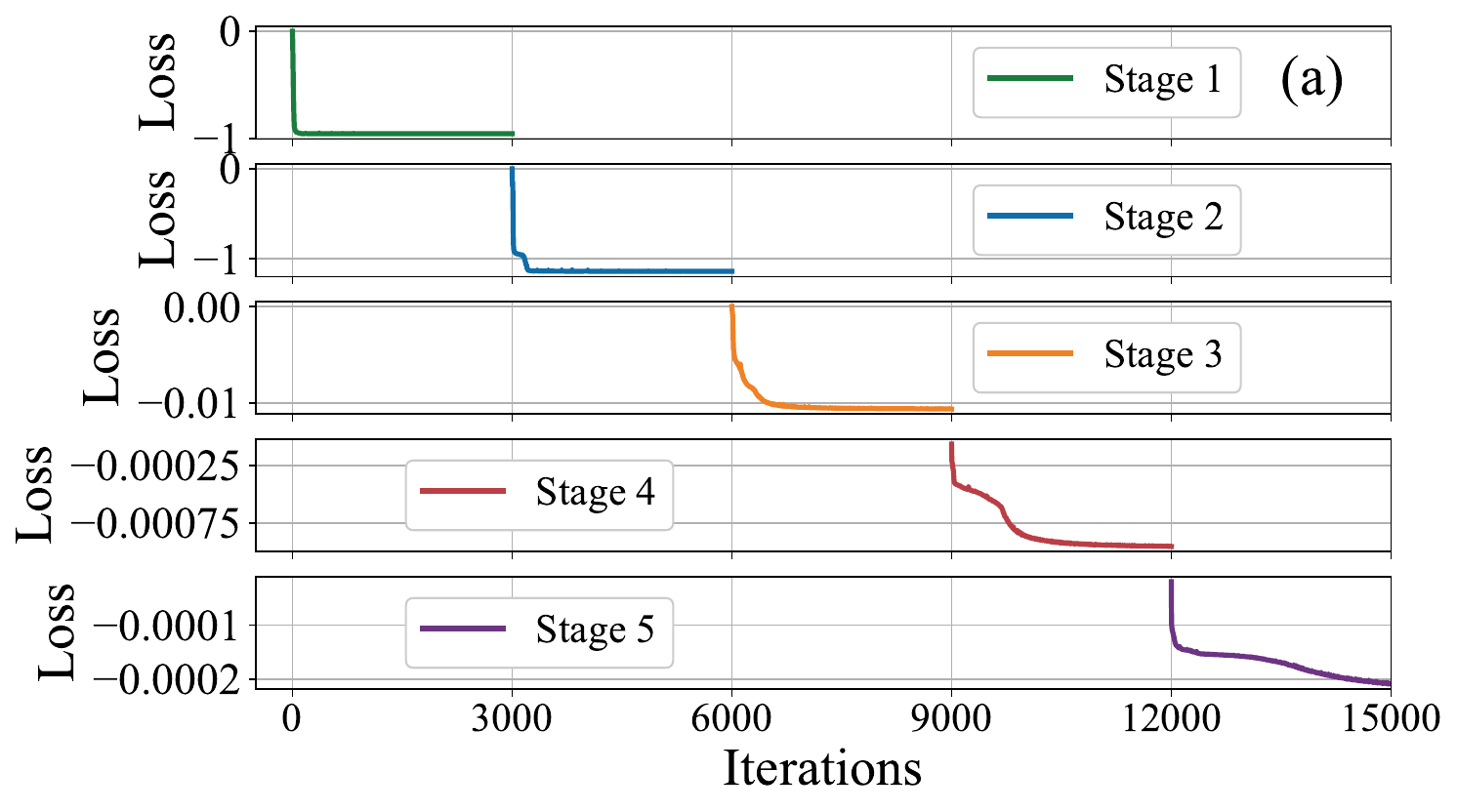,width=3in,height=1.8in,angle=0}
	\psfig{figure=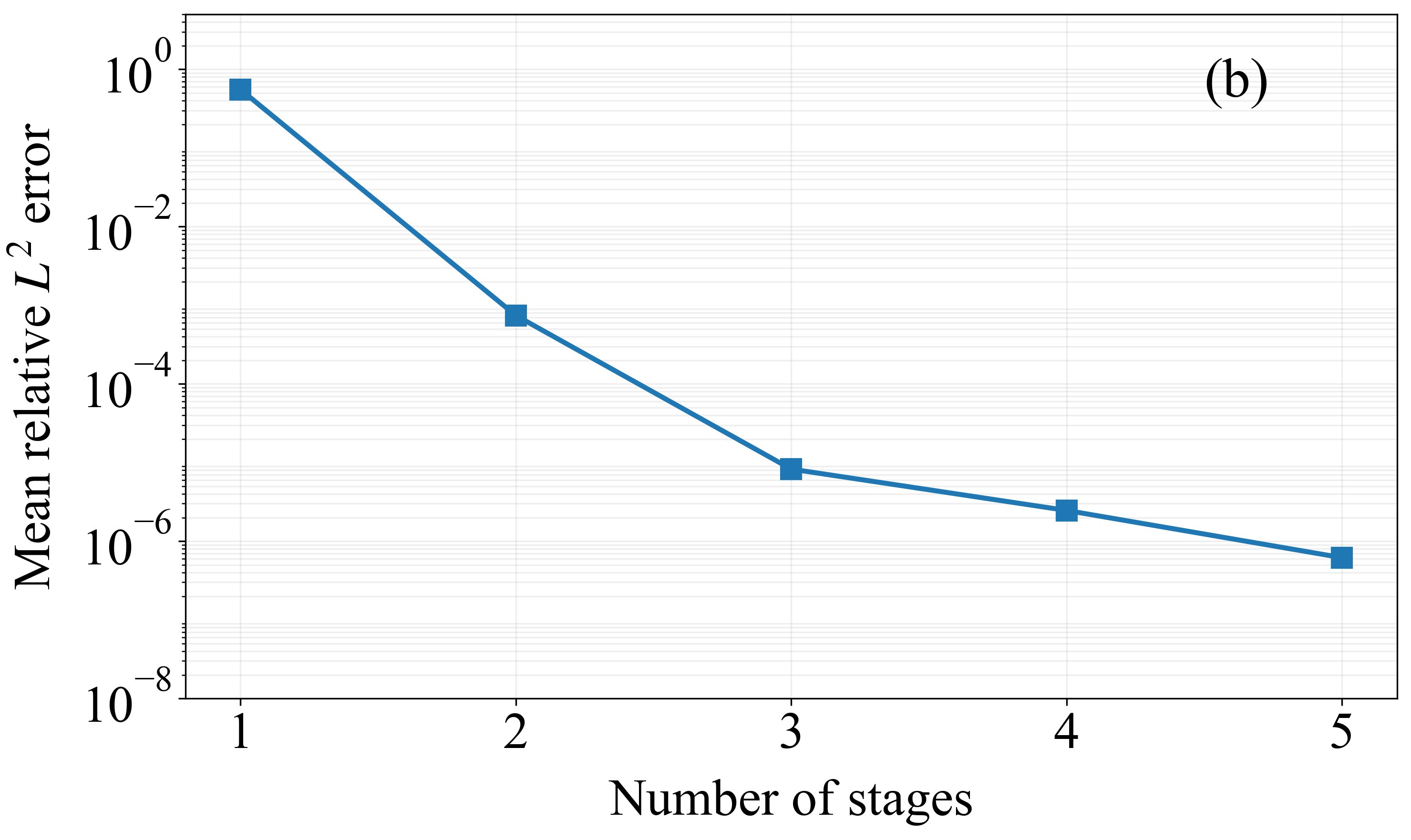,width=3in,height=1.8in,angle=0}
	\caption{ Results of DGNO in Algorithm \ref{alg:DGNO} for Example \ref{Ex:Poisson}: (a) Change of Loss during the iterations; (b) Mean relative $L^2$ error with respect to the number of stages.}
	\label{fig:poisson_unsupervised_error}
\end{figure}

 \begin{table}[t!] 	
	\centering 	
	\caption{Number of stages, epochs per stage, training time and error for DCNO and DGNO.} 	
	\label{tab:compare_su_un} 	
	\begin{tabular}{c|cccc} 		\hline 					
		& $\text{Stages}$  & \text{Epochs per stage} & $\text{Training time(s)}$  &  $\text{Mean relative } L^2 \, \text{error}$     \\ \hline 		
		supervised DCNO	   & 5 & 3000	& 251.23s  	& 1.31E-6  \\
		unsupervised DGNO  & 5 & 3000	& 491.60s  	& 6.17E-7  \\ \hline			
		
	\end{tabular}
\end{table}



We move on to test DGNO on Example \ref{Ex:Poisson}. 
The unsupervised DGNO needs only the input functions as the training set, which is adopted as the same $\{\rho^{(b)}\}_{b=1}^{N_b}$ with $N_b=100$ for training in DCNO. 
We set $\gamma=2$ and $m=4$ in the combined moment-matching strategy, and zero boundary conditions are encoded into DeepONet via $\mathcal{G}_{\theta_s}(f,x,y) =  (I^2 - x^2)(I^2 - y^2)  \sum_{k=1}^{p} \mathfrak{b}_k(f)\,\mathfrak{t}_k(x,y)$. 
The trunk input is discretized on Gauss–Legendre nodes with $N_x = N_y = 110$ in each spatial direction.
The total number of stages set to $S=5$, and all other training settings are kept the same as for DCNO above.
Fig.~\ref{fig:poisson_unsupervised_error}(a) shows the evolution of the loss function during training. In each stage, the loss decreases rapidly as the training iterations proceed.
In the testing phase, the same test set as that used for DCNO is adopted.
Table~\ref{tab:network_architecture} presents the errors obtained by DGNO for various network depths and widths. The results indicate a clear reduction in error with increasing network depth and width.
With $L=4$ and $W_0=70$ fixed, Fig.~\ref{fig:poisson_unsupervised_error}(b) demonstrates the reduction in the test error as the stage increases, with the final error $6.17\times10^{-7}$ also approaching the machine precision of single float. These results support Theorem \ref{thm:dgno}.



{Lastly, we compare the performance of DCNO and DGNO on Example \ref{Ex:Poisson}. The results are summarized in Table~\ref{tab:compare_su_un}, which presents the number of stages, epochs per stage, training time, and mean relative $L^2$ error for both approaches. 
It can be observed that, with the same number of stages and training epochs, DGNO achieves a lower generalization error, albeit at the cost of longer training time, which is due to the additional computational overhead of evaluating derivatives in its loss function. Overall, DGNO benefiting from its unsupervised feature is certainly more efficient.} Nevertheless, there are convolution problems that do not enjoy a PDE formulation, where DGNO is not applicable and DCNO becomes the solution. The following example is presented to illustrate this.


\begin{example}[Multiquadrics]\label{Ex:Multiquadrics}
Consider the multiquadric kernels $ K(x) = {1}/{\sqrt{x^{2} + a^{2}}}$ with $0 < a \ll 1 $, commonly used in electrostatic computation \cite{zhang2021fast,zhuang2013local,zhuang2017fast},  
 for a 1D convolution:
\[
\phi(x) = \int_{0}^{1} K(x-\tilde{x}) \rho(\tilde{x}) \, d\tilde{x}, \quad  x \in \Omega=(0, 1).
\]
It does not have a PDE formulation \cite{zhuang2013local,zhuang2017fast}, and
our input is the Gaussian-type density function $ \rho(x) = \mathrm{e}^{-\beta\left(x - \frac{1}{2}\right)^{2}}$ with $ \beta > 0$.
\end{example}

In this example, the parameter $a>0$ controls the regularization scale of the kernel, and
as $a \to 0$, the kernel approaches a weakly singular form.
For a fixed $\rho(x) = \mathrm{e}^{-2\left(x - \frac{1}{2}\right)^{2}}$, the corresponding solutions $\phi(x)$ for different $a$ are shown in the left plot of Fig.~\ref{fig:multiquadrics_a}. 
As $a$ decreases, the kernel becomes more singular at $0$, resulting in solutions with steeper variations near the boundary. 
We then set $a=0.0001 $ to examine DCNO in this near-singular regime.
The training data is provided by the traditional method SOE \cite{zhang2021fast}, generating a training set $\{ (\rho^{(b)}, \phi^{(b)}) \}_{b=1}^{N_b}$ containing $N_b=50$ samples, where $\rho^{(b)}$ is generated by uniformly sampling $ \beta \sim \mathcal{U}[1,5]$.
The trunk network is evaluated at collocation points defined on the uniform grid
$
x_k={k}/{N_x}, \, k=0,\ldots,N_x,
$
with $N_x=200$ for training and $N_x=300$ for testing.
The total number of stages is set to $S=4$, with $\sigma=\tanh$ and the depth of DeepONet fixed at $L=4$. 
The network width and $p$ are increased with the stage index $s$, i.e., $W = 70 + 10\,(s-1)$ and $p = 120 + 10\,(s-1)$.
For test, 100 samples that are disjoint from the training set are randomly selected to form the test set. The right plot of 
Fig.~\ref{fig:multiquadrics_a} shows the mean relative $L^2$ error on the test set, which decreases steadily as the number of basis operators increases with the stage index $s$, and eventually reaches $1.16 \times 10^{-6}$.
For the test sample $\rho(x)=\mathrm{e}^{-2\left(x-\frac{1}{2}\right)^2}$, Fig.~\ref{fig:multiquadrics_solution}(a) presents the exact and numerical solutions, while the corresponding pointwise error is shown in Fig.~\ref{fig:multiquadrics_solution}(b). The results demonstrate that DCNO can achieve high accuracy even for this near-singular setting. The mean relative $L^2$ errors of the approximate solution $\phi$ over the range of $a$ values are presented in Table~\ref{tab:linear_dt_error} to evaluate the robustness of DCNO. The results show that DCNO maintains consistently high accuracy across all tested values of $a$, indicating that DCNO is uniformly effective from the near-singular ($a = 10^{-4}$) regime to the regular ($a = 10^{-1}$) regime.

\begin{figure}[h!]
	\centering
	\psfig{figure=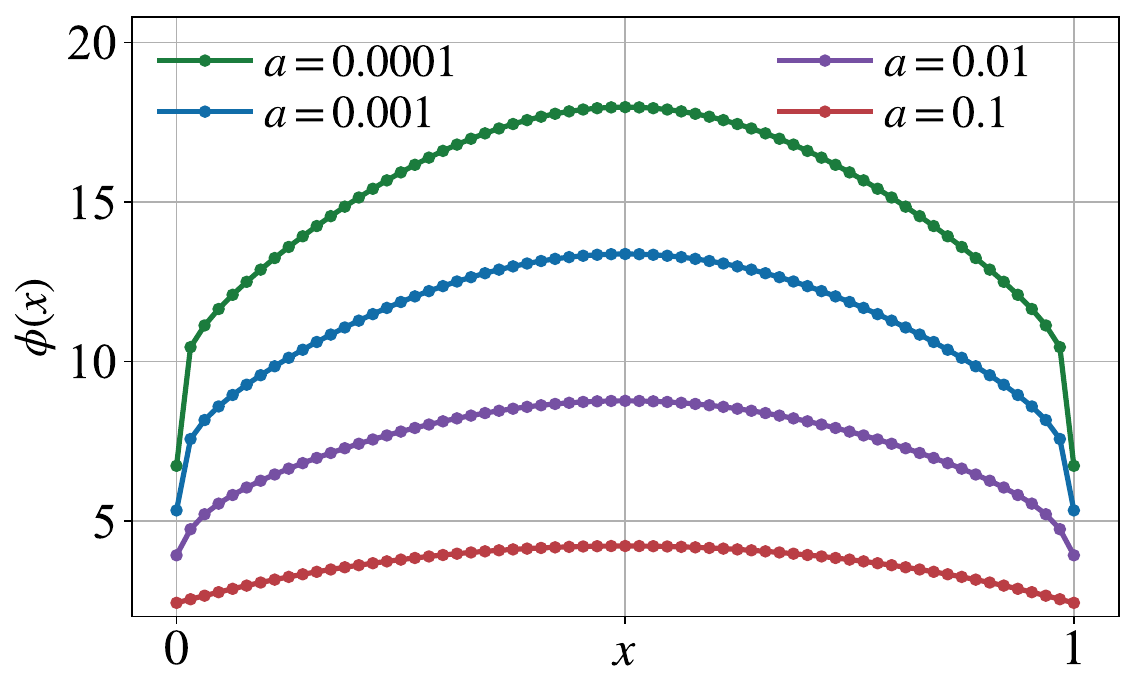,width=3in,height=1.8in,angle=0}
	\psfig{figure=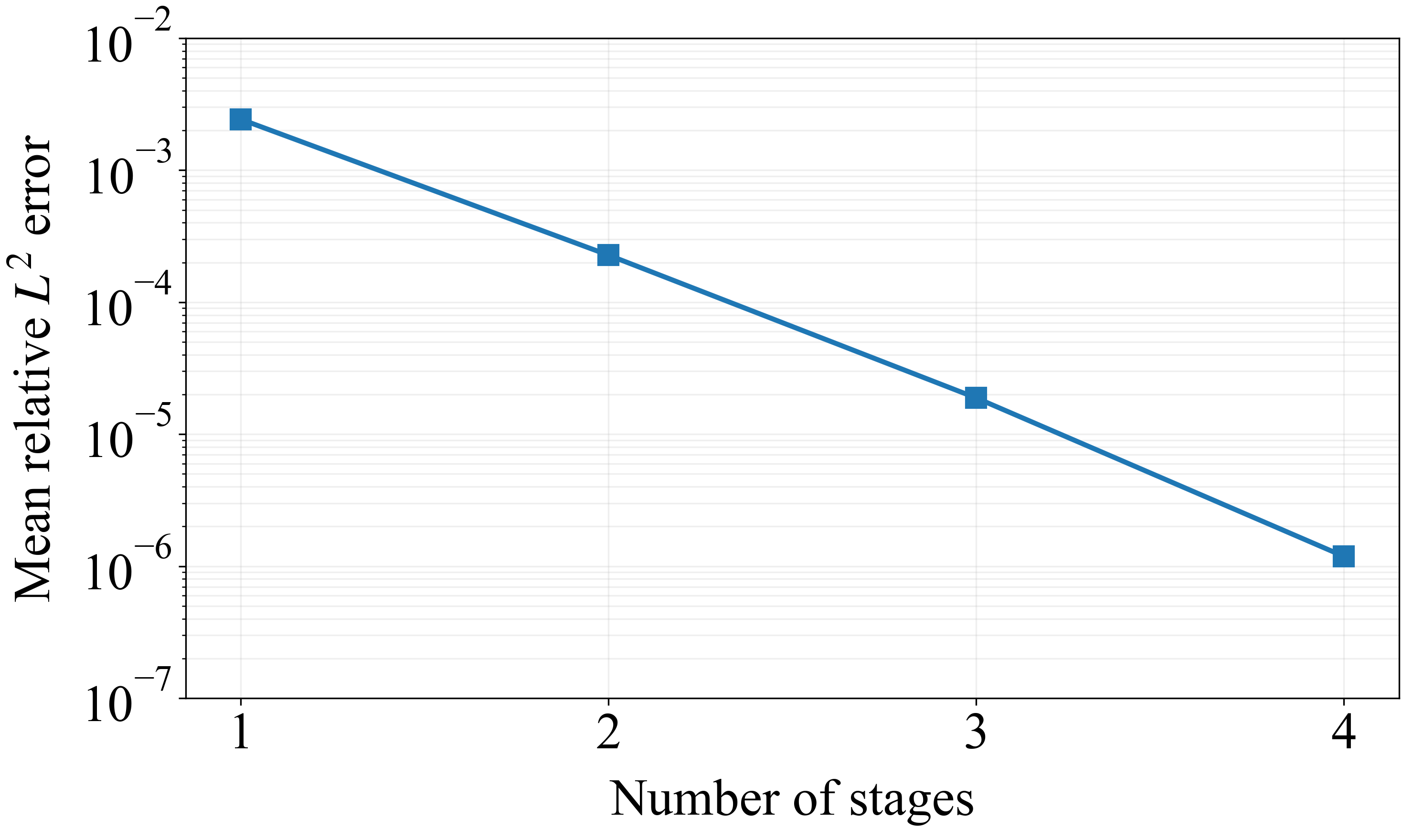,width=3in,height=1.8in,angle=0}
	\caption{Solutions for different values of the kernel parameter $a$ (left);  Mean relative $L^2$ error with respect to the number of stages for $a=0.0001 $ (right).}
	\label{fig:multiquadrics_a}
\end{figure}

\begin{figure}[h!]
	\centering
	\psfig{figure=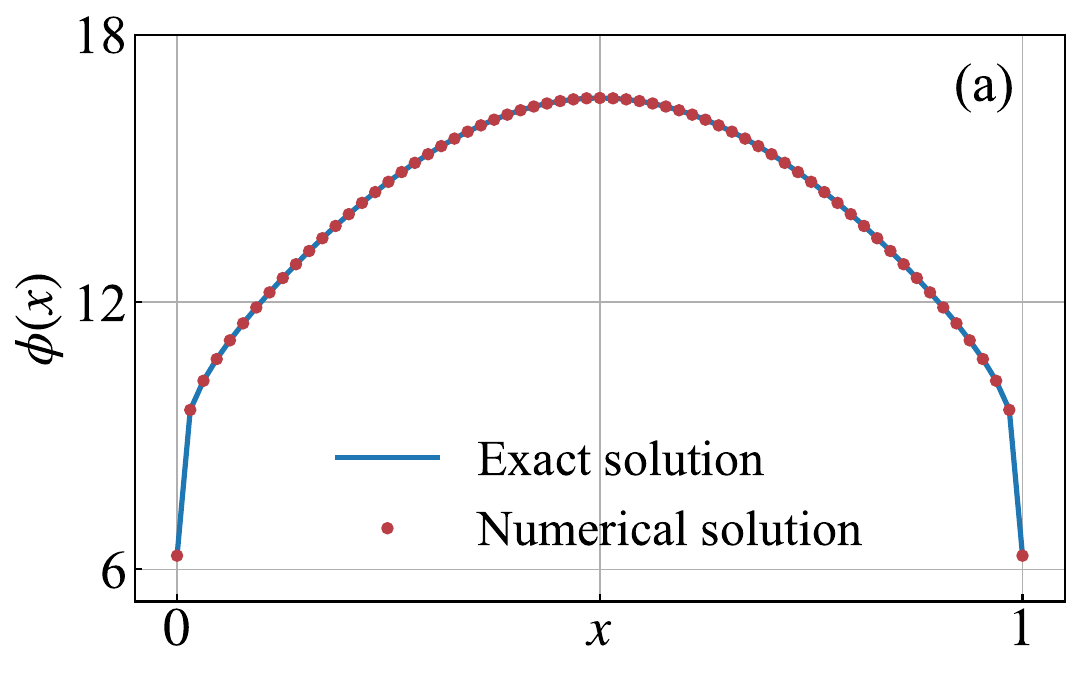,width=2.1in,height=1.4in,angle=0}
	\psfig{figure=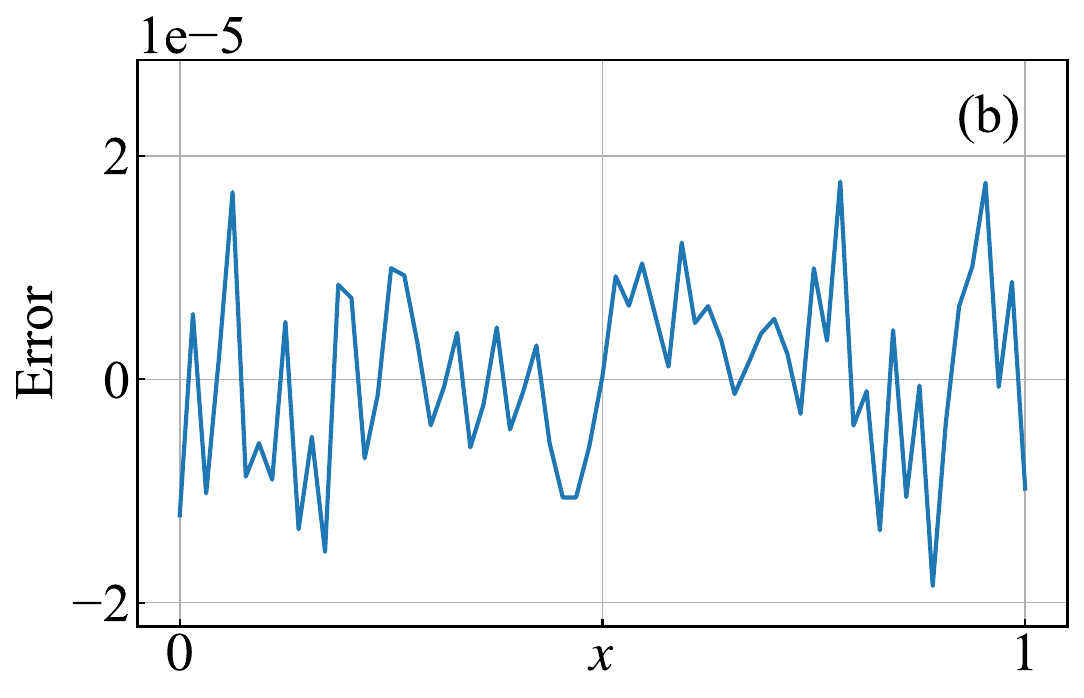,width=2.1in,height=1.4in,angle=0}
    \psfig{figure=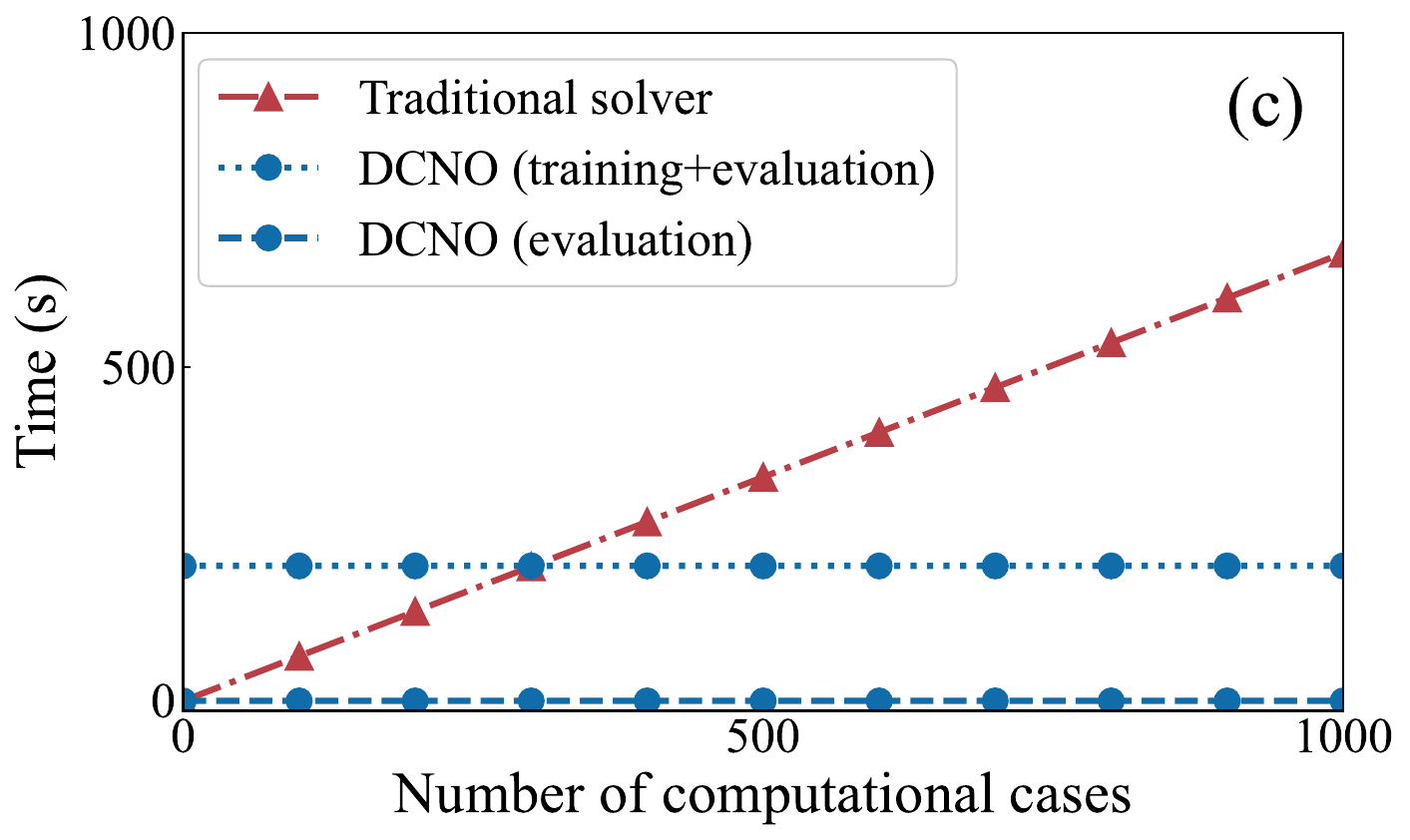,width=2.1in,height=1.4in,angle=0}

	\caption{Results of DCNO for Example \ref{Ex:Multiquadrics}: (a) Exact and numerical solutions; (b) Error of the numerical solution; (c) Computational time of the traditional solver and DCNO against the number of computational cases.}
	\label{fig:multiquadrics_solution}
\end{figure}

  \begin{table}[h!]
	\centering
	\caption{Mean relative $L^2$ error of numerical solutions from DCNO for varying $a$.}
	\label{tab:linear_dt_error}
	\begin{tabular}{c|cccc}
		\hline
		$a$ & $0.0001$ & $0.001$ & $0.001$ & $0.1$  \\ \hline
		Error & 1.20$\,e\!-\!6$ & 3.08$\,e\!-\!6$ & 2.71$\,e\!-\!6$ & 3.48$\,e\!-\!6$ \\ \hline
	\end{tabular}
\end{table}


By Example \ref{Ex:Multiquadrics}, we assess the computational efficiency of our operator learning by comparing the traditional SOE method with DCNO on a test set of 1000 samples. Both methods are evaluated on the same uniform grid with $64$ spatial points for $a=0.0001 $. 
The results in Table~\ref{tab:comp_traditional} indicate that while DCNO incurs additional training overhead, it drastically reduces the evaluation time compared to the traditional approach, achieving several orders of magnitude speedup during inference and demonstrating clear advantages for large‑scale evaluations. 
Fig.~\ref{fig:multiquadrics_solution}(c) further reports the total computational time of SOE and DCNO as the number of test cases increases. 
DCNO demonstrates a much slower growth rate, highlighting its superior computational efficiency at scale. So will DGNO if applicable.

	\begin{table}[t!]
		\centering
		\caption{Comparison between DCNO and the traditional method SOE on test set.}
		\begin{tabular}{c|cc}
			\hline
			Method  & Evaluation time (s) & Training time (s) \\
			\hline
			Traditional method   & 670.44   & ---\\
			DCNO method          & 0.021 & 202.14 \\
			\hline
		\end{tabular}
		\label{tab:comp_traditional}
	\end{table}


  
  \begin{example}[Dynamics of Schr\"{o}dinger-Poisson system]\label{EX:SPS}
  For quantum mechanics taking into account self-gravitating interaction, one considers 
  	the nonlinear Schr\"{o}dinger-Newton/Poisson equations that reads in dimensionless form \cite{penrose2014SPS,zhang2011computation}:
\begin{subequations}\label{SP system}
	\begin{align}
		&i\partial_{t}\psi(\bx,t)
		=-\frac{1}{2}\Delta\psi
		+V(\bx)\psi
		+\phi(\bx,t)\psi
		+|\psi|^{2}\psi,
		\quad
		\bx\in\Omega\subset\mathbb{R}^{2},\ t>0,
		\label{eq:S}\\
		&-\Delta\phi(\bx,t)=|\psi|^{2},
		\quad
		\bx\in\Omega,\ t>0,
		\label{eq:SP}\\
		&\psi(\bx,0)=\psi_{0}(\bx),
		\quad
		\bx\in\overline{\Omega}; \qquad 
		\psi(\bx,t)=0,\ 
		\phi(\bx,t)=0,
		\quad
		\bx\in\partial\Omega,\ t\geq0.
		\label{eq:SPbc}
	\end{align}
    \end{subequations}
    Here, $\psi$ is a complex-valued wave function, $\phi$ is the real-valued gravitational potential,  
$V(\bx)=\frac{1}{2}(x^2+y^2)$ denotes an external harmonic oscillator trapping potential.
We consider the circular computational domain
$\Omega=\{\bx\in\mathbb{R}^2:\ |\bx|<I\}$ with radius $I=8$ and solve its dynamics for $t\in[0,1]$.
  	 The initial condition is given by $\psi_0(x,y)=\frac{1}{\sqrt{2\pi}}\mathrm{e}^{-(x^2+y^2)/4}.$
  \end{example}

A popular numerical solver for the dynamics of \eqref{SP system} on a circular domain is 
the time-splitting finite element method, briefly described below as a reference.
The spatial domain is discretized by a conforming triangular mesh with $N_h$ interior nodes, where continuous piecewise linear ($P_1$) finite elements are employed.
Let $\{\vartheta_j\}_{j=1}^{N_h}$ denote the corresponding nodal basis functions.
The finite element approximation of $\psi(\cdot,t_n)$ is
given by $\psi_h(\cdot,t_n)
=
\sum_{j=1}^{N_h}\Psi_j^n\,\vartheta_j$, with $\Psi^n=(\Psi_1^n,\ldots,\Psi_{N_h}^n)^T\in\mathbb C^{N_h}$.
Given the time discretization $t_n=n\Delta t,\ n\geq0$,
the Strang splitting Crank--Nicolson finite element scheme \cite{SP_CN_FEM} reads
\begin{subequations}\label{SSFEM}
	\begin{align}
		&\left(M_h+\frac{i\Delta t}{8}K_h\right)\Psi^{*}
		=
		\left(M_h-\frac{i\Delta t}{8}K_h\right)\Psi^{n},
		\label{eq:femsplit1}
		\\[6pt]
		&\Psi^{**}
		=
		\exp\!\left[
		-i\Delta t
		\left(
		V+\Phi^{*}
		+|\Psi^{*}|^2
		\right)
		\right]
		\Psi^{*},
		\label{eq:femsplit2}
		\\[6pt]
		&\left(M_h+\frac{i\Delta t}{8}K_h\right)\Psi^{n+1}
		=
		\left(M_h-\frac{i\Delta t}{8}K_h\right)\Psi^{**}.
		\label{eq:femsplit3}
	\end{align}
\end{subequations}
Here, $M_h$ and $K_h$ denote the finite element mass and stiffness
matrices, respectively, with entries
$(M_h)_{jk}=  \langle \vartheta_j,\vartheta_k\rangle_{L^2(\Omega)},\,
(K_h)_{jk}=\langle \nabla\vartheta_j,\nabla\vartheta_k \rangle_{L^2(\Omega)}$. 
The numerical potential $\Phi^{*}\in\mathbb{R}^{N_h}$ is obtained from the finite element discretization of the Poisson equation \eqref{eq:SP}, leading to the linear system $K_h\Phi^{*}=F_h(|\Psi^{*}|^2),$
where $F_h(|\Psi^{*}|^2)$ denotes the assembled load vector.
Since the stiffness matrix $K_h$ is time-independent, its sparse LU factorization is performed only once before time stepping, while each Poisson solve requires only triangular substitutions.
Nevertheless, solving the Poisson equation repeatedly at every time step remains computationally demanding, especially for small $\Delta t$ or long-time simulations.

\begin{figure}[t!]
	\centering

    \psfig{figure=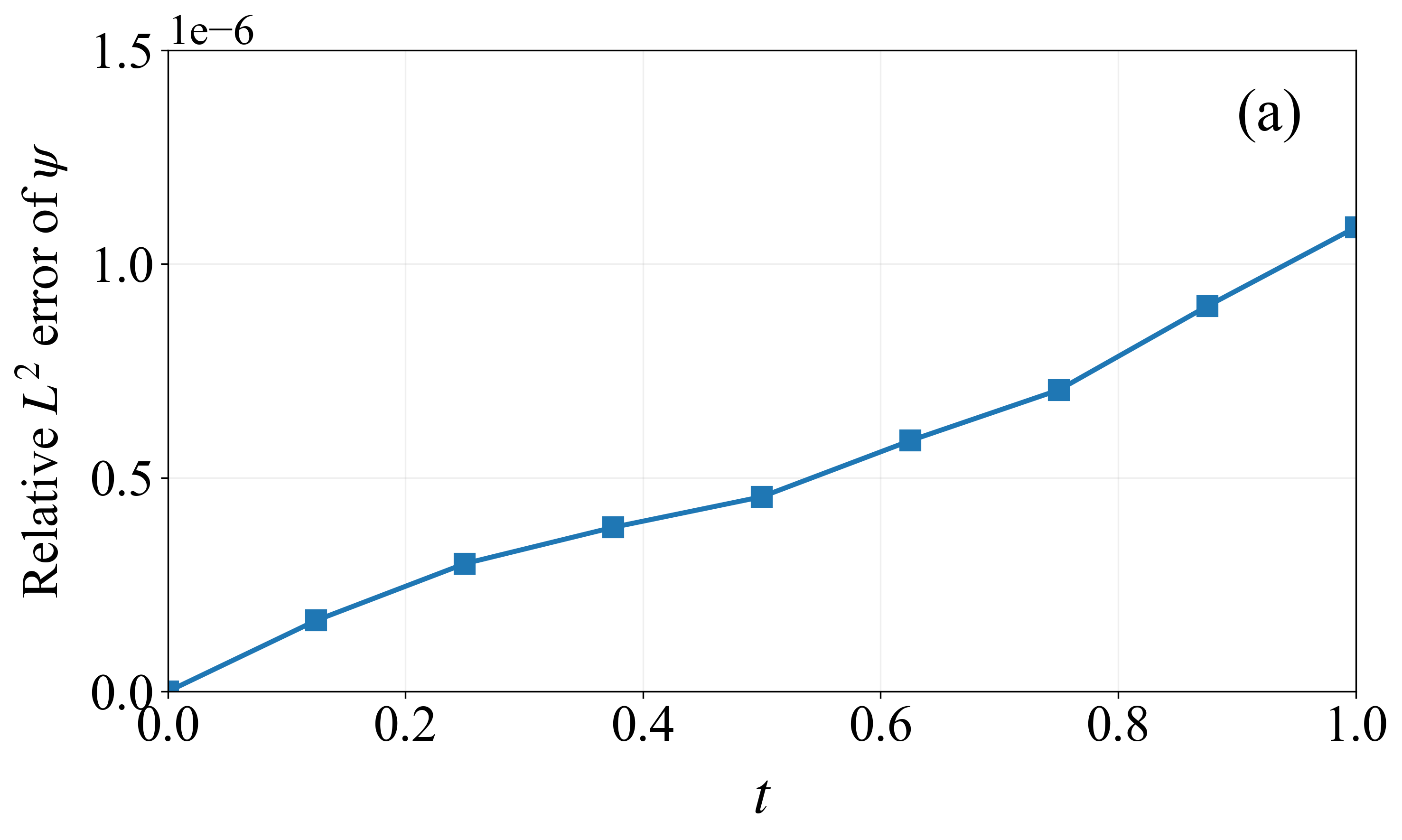,width=2.1in,height=1.4in,angle=0}
	\psfig{figure=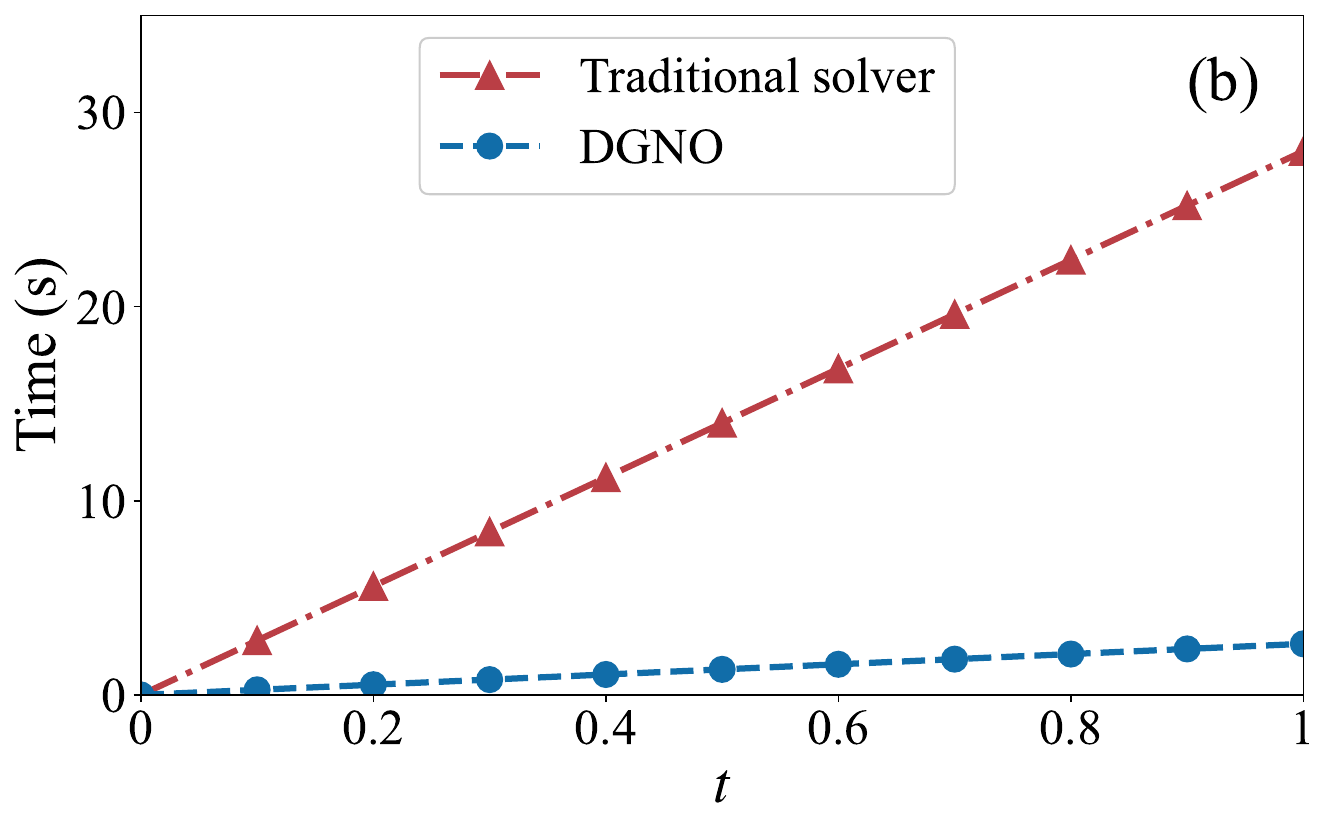,width=2.1in,height=1.4in,angle=0}
    \psfig{figure=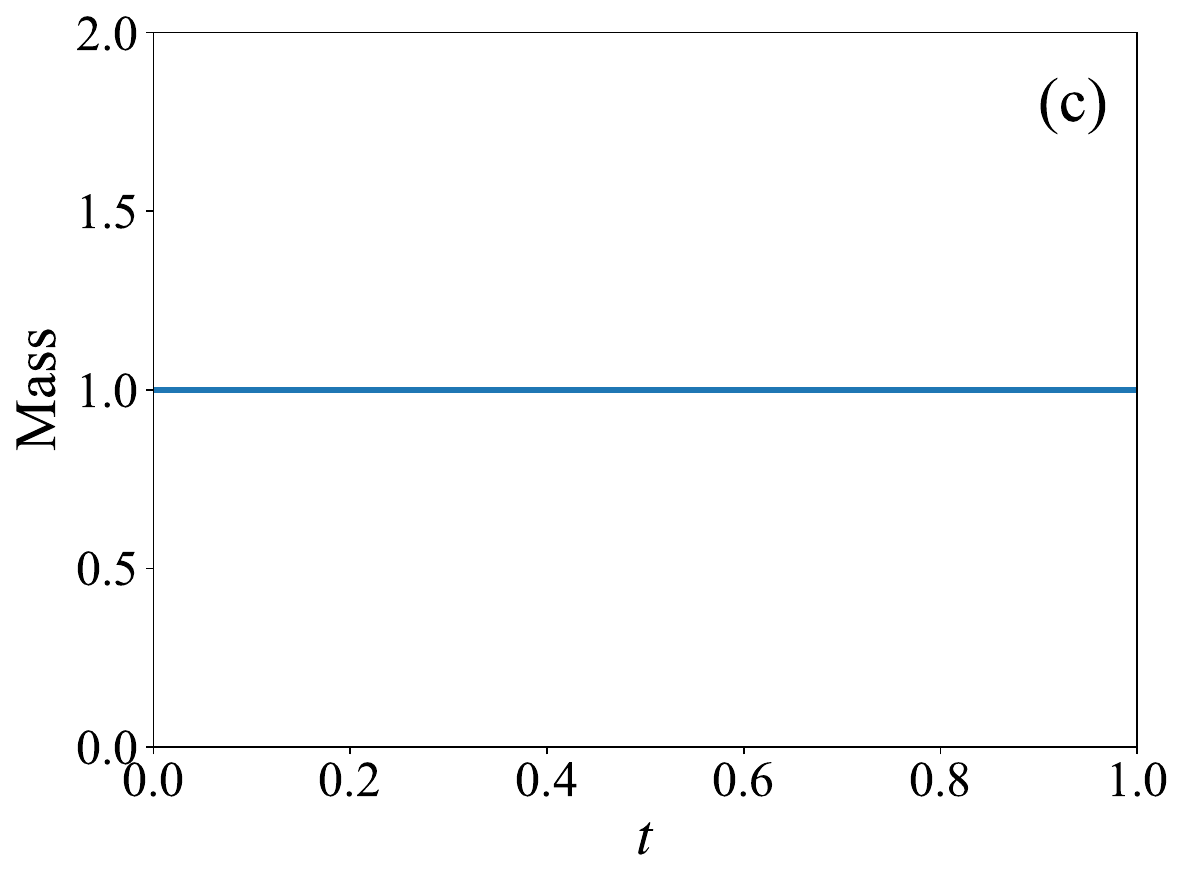,width=2.1in,height=1.4in,angle=0}
    
	\caption{
    Results of Example~\ref{EX:SPS}: (a) Relative $L^2$ error of $\psi$ over time; (b) Total evaluation time of the traditional solver and DGNO for solving the Poisson equation \eqref{eq:SP} during the simulation; (c) Evolution of the mass over time. 
    }
	\label{fig:SPS_sup_error}
\end{figure}

 \begin{figure}[t!]
 	\centering
\includegraphics[width=1.95in,height=1.8in]{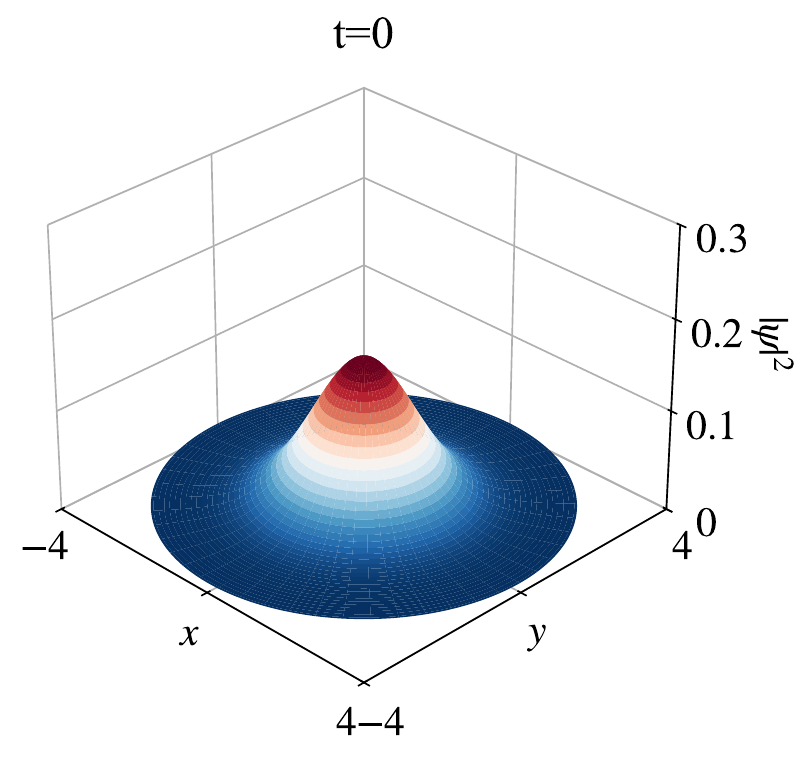}
	\includegraphics[width=1.95in,height=1.8in]{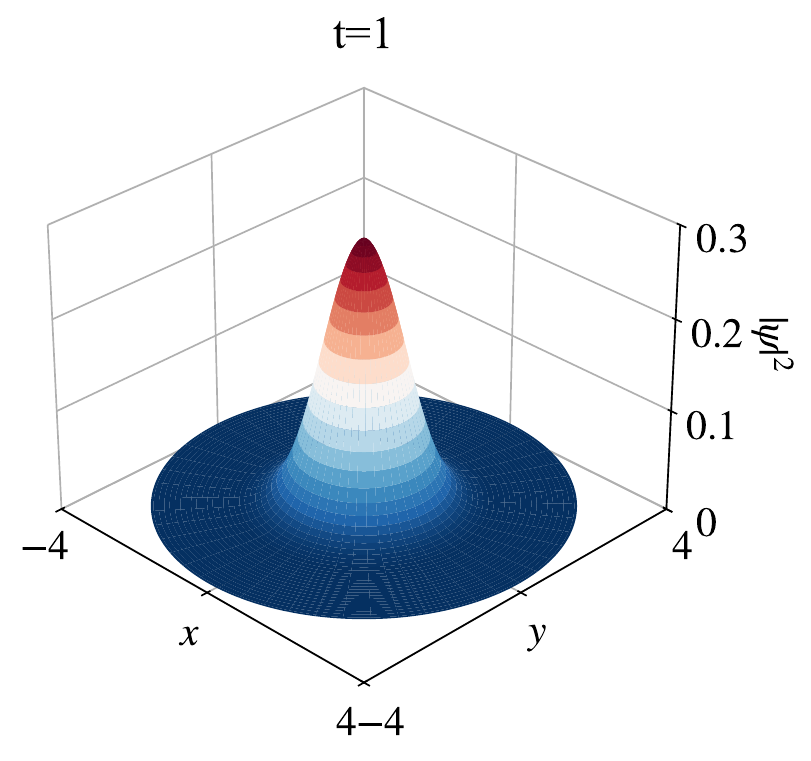}
	\includegraphics[width=1.8in,height=1.6in]{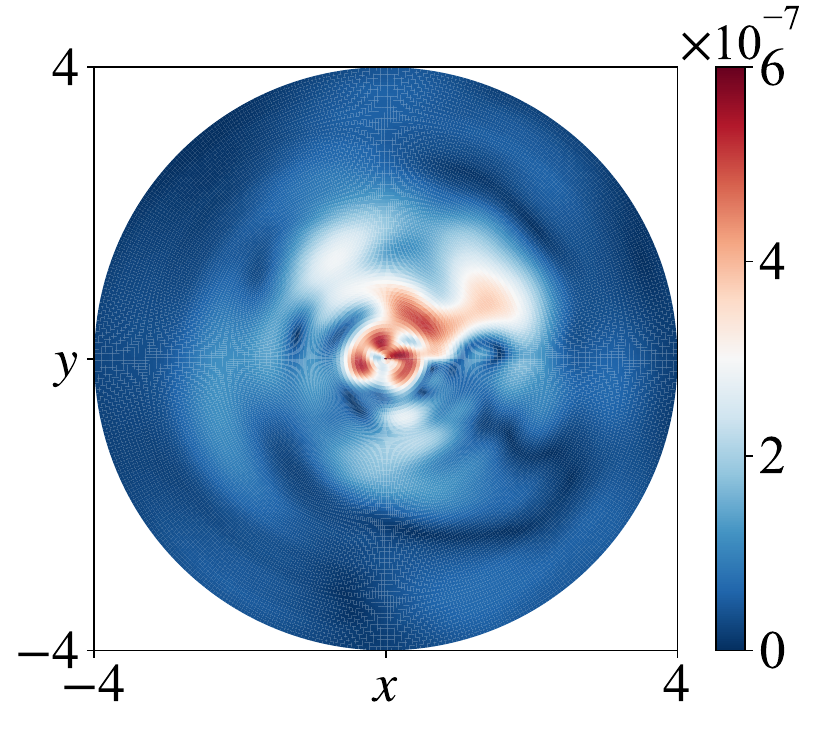}
 	\caption{$|\psi(\cdot,t=0)|^2$ (left); $|\psi(\cdot,t=1)|^2$ by DGNO (middle); Pointwise absolute error at $t=1$ (right) for Example~\ref{EX:SPS}.}
 	\label{fig:SPS_solution}
 \end{figure}

Here, we employ our DGNO method to replace the evaluation process for $\Phi^{*}$ in \eqref{SSFEM}. The training of DGNO is performed similarly to that in Example \ref{Ex:Poisson}. 
After training, only a single neural operator inference is required at each time step to efficiently obtain an approximation of the gravitational potential. We set $\Delta t=0.001, N_h=80000$ for the experiment. 
Fig.~\ref{fig:SPS_sup_error}(a) illustrates the relative $L^2$ error of the solution field $\psi$ over time. 
The linear drift in the solution error with respect to time comes from the natural accumulation of the approximation error. 
Fig.~\ref{fig:SPS_sup_error}(b) compares the accumulated evaluation time of the Poisson solver based on the conventional finite element method and DGNO throughout the simulation.
The evolution of the mass
$\int_{\Omega}|\psi(\bx,t)|^2\,\mathrm{d}\bx$ over time is presented in Fig.~\ref{fig:SPS_sup_error}(c).
Fig.~\ref{fig:SPS_solution} further presents  $|\psi|^2$ at $t=0$ and $t=1$, together with the corresponding pointwise absolute error at $t=1$. 
As can be seen, the proposed method remains stable and accurate throughout the evolution, provides an accelerated solver for \eqref{SP system}, and exhibits excellent mass conservation.


\subsection{Extension to multiple inputs}

We generalize our methods to further study the convolution problem \eqref{eq:convolution} in which both the density function and the convolution kernel vary, forming a multi-input operator learning task. 
In fact, this can be quite straightforward under our approach by adding additional branches to (\ref{DeepONet}), replacing DeepONet with the Multiple-Input Operator Network (MIONet) \cite{MIONet}.

\begin{example}[DCNO with MIONet]\label{Ex:MIO_sup}
	Consider the Gaussian kernel family
	$
	K_\alpha(\bx)
	=
	\frac{1}{2\pi\alpha^2}
	\mathrm{e}^{-\frac{|\bx|^2}{2\alpha^2}},
	\, \bx\in\Omega = (-8, 8)^2,\ \alpha>0.
	$  
	Define the convolution operator
	\[
	\mathcal{G}:(\rho,K_\alpha)\mapsto \phi, \qquad 
	\phi(\bx) = (K_\alpha * \rho)(\bx) = \int_{\Omega} K_\alpha(\bx-\tilde{\bx}) \rho(\tilde{\bx})\,d\tilde{\bx}.
	\]
    Such a convolution is repeatedly needed in constructing the scale-space of the scale-invariant feature transform algorithm \cite{Lowe2004SIFT}. 
	Here we take $\rho$ from a family of Gaussian densities
	$
	\rho(\bx)
	=
	\frac{1}{2\pi\tau^2}
	\mathrm{e}^{-{|\bx|^2}/{(2\tau^2)}}
	$ with parameter $\tau>0, 
	$
	and then the convolution admits the analytical solution
	$
	\phi(\bx)
	=
	\frac{1}{2\pi(\alpha^2+\tau^2)}
	\mathrm{e}^{-|\bx|^2/{(2\alpha^2+2\tau^2)}}
	$ as reference. 
\end{example}

To solve Example~\ref{Ex:MIO_sup}, we apply the DCNO algorithm with the basis operator network constructed using the MIONet architecture \cite{MIONet} to accommodate multiple input functions. 
For training, a dataset $\{ (\rho^{(b)},K_\alpha^{(l)}, \phi^{(b,l)}) \}_{b=1,l=1}^{N_b,N_l}$ with $N_b=N_l=20$ samples is generated, with each parameter $\alpha^{(b)}$ and $\tau^{(l)}$  is independently drawn from $\mathcal{U}[1,2]$. 
All other training settings are the same as those in Example~\ref{Ex:Poisson}, except that $N_x=N_y=64$ are used during training, while a finer grid with $N_x=N_y=100$ is adopted for error evaluation.
At each stage $s$, the practical loss function for DCNO in the multiple-input case is extended to
\begin{equation}\label{eq:loss_mionet_sup}
	\text{Loss}_s(\theta) = \frac{1}{N_b\times N_l\times N_{\bx} } \sum_{b=1}^{N_b}\sum_{l=1}^{N_l} \sum_{k=1}^{N_{\bx}}  \left| \frac{1}{\beta_{s-1}} e_{s-1}(\rho^{(b)},K_\alpha^{(l)})(\mathbf{x}_k) -\mathcal{G}_{\theta}(\rho^{(b)},K_\alpha^{(l)},\mathbf{x}_k) \right|^2.
\end{equation}

In the testing phase, 200 samples disjoint from the training set are randomly selected. 
 Fig.~\ref{fig:mio_unsupervised_error}(a) illustrates a clear decrease in the test error as the number of basis operators increases, and ultimately reaches an accuracy of $1.40 \times 10^{-5}$.
 Notably, without the progressive construction of basis operators employed in DCNO, the accuracy at stage 1 corresponds to that of the standard MIONet, which is only $4.84\times10^{-2}$. This demonstrates the significant advantage of DCNO in improving the accuracy of multi-input operator learning.
Fig.~\ref{fig:mionet_supervised_solution} presents the numerical solutions on the test set for different combinations of density functions and convolution kernels, together with the corresponding pointwise error distributions. 
These results demonstrate that DCNO maintains strong performance in the multiple-input setting for generalizing simultaneously over the density and the convolution kernel.

\begin{figure}[h!]
	\centering
	\psfig{figure=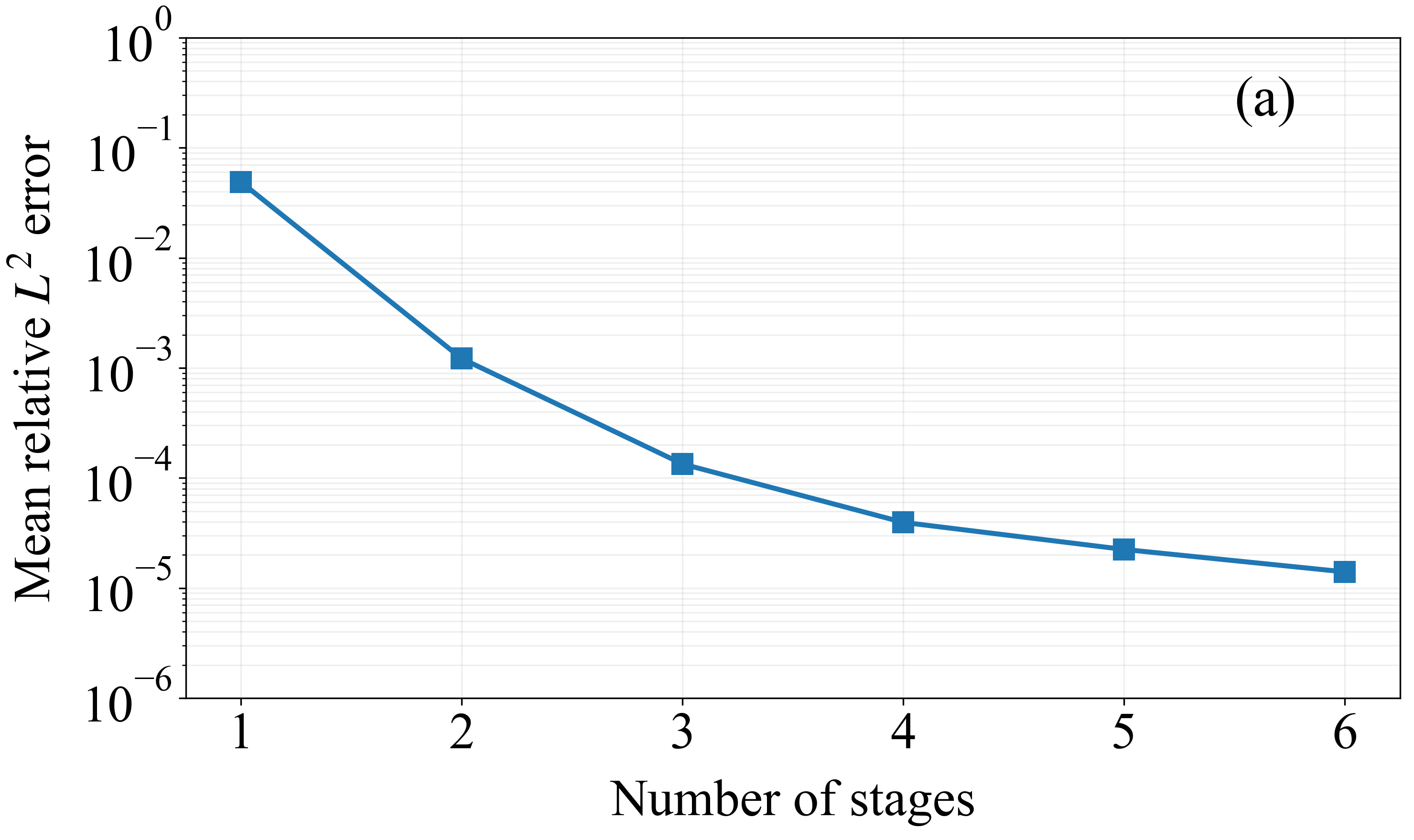,width=3in,height=1.8in,angle=0}
	\psfig{figure=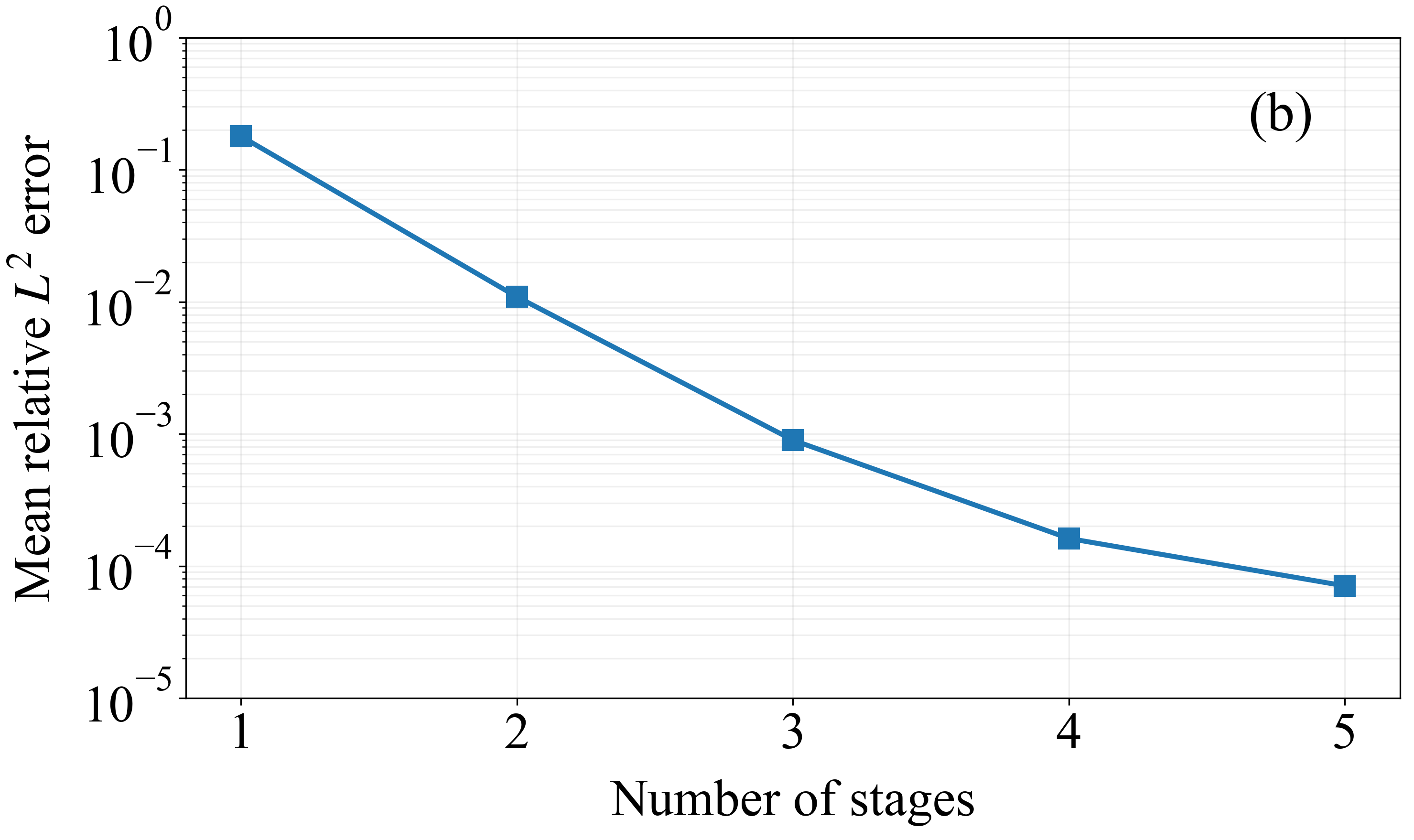,width=3in,height=1.8in,angle=0}
	\caption{ Mean relative $L^2$ error with respect to the number of stages: (a) Example \ref{Ex:MIO_sup}; (b) Example \ref{Ex:MIO_unsup} .}
	\label{fig:mio_unsupervised_error}
\end{figure}

\begin{figure}[t!]

	\centerline{
		\psfig{figure=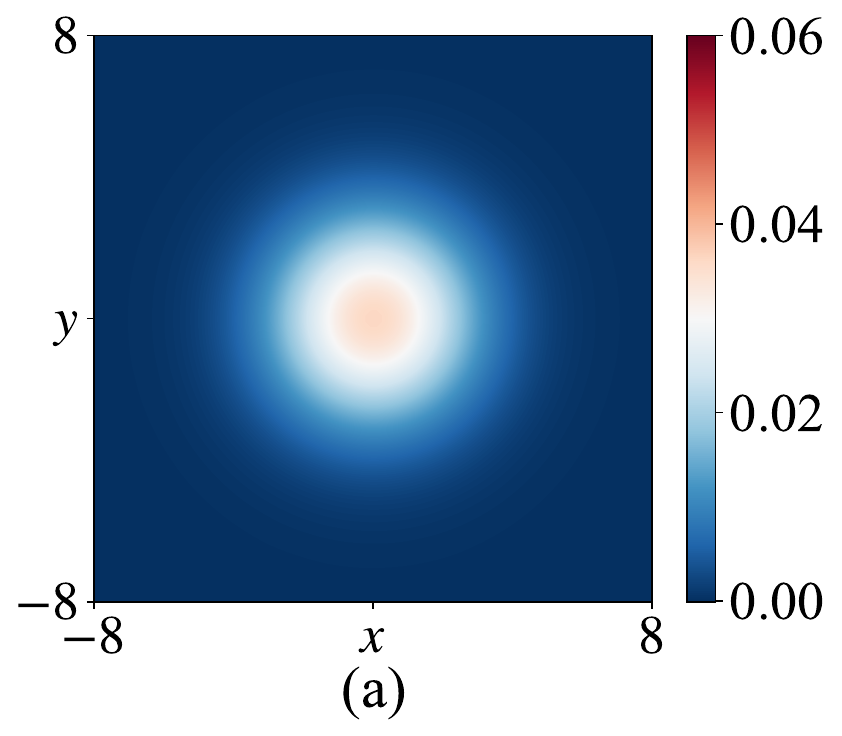,width=1.9in,height=1.5in,angle=0}
		\hspace{0in}
		\psfig{figure=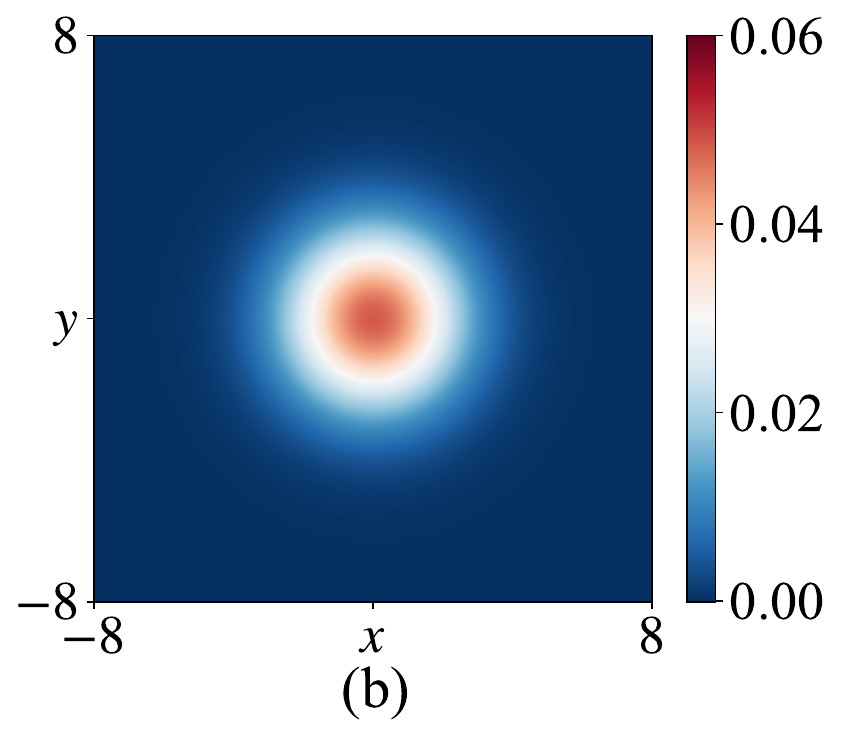,width=1.9in,height=1.5in,angle=0}
		\hspace{0in}
		\psfig{figure=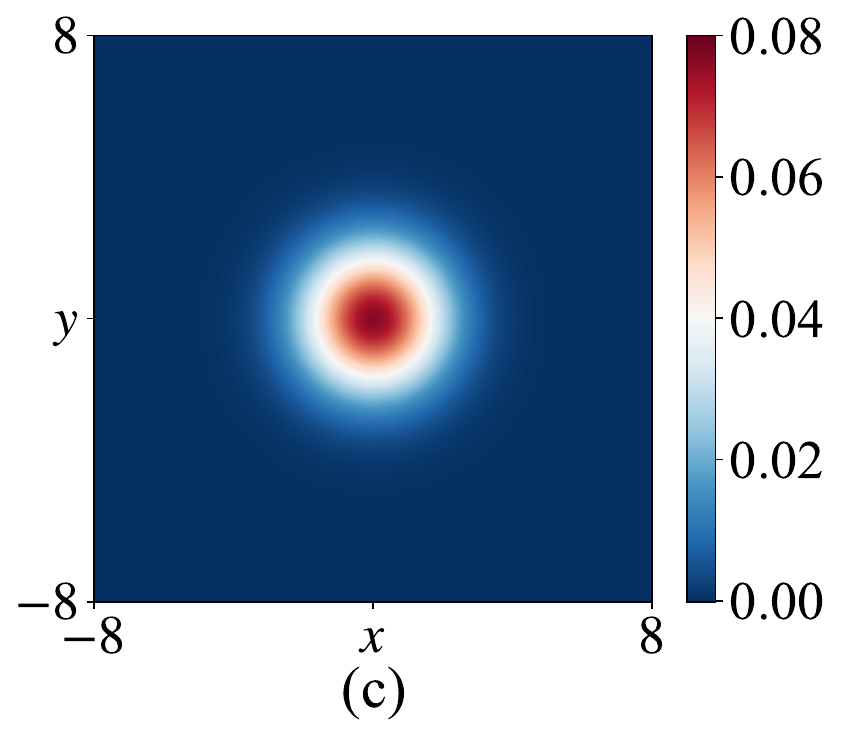,width=1.9in,height=1.5in,angle=0}
		
	}
	\centerline{
		\psfig{figure=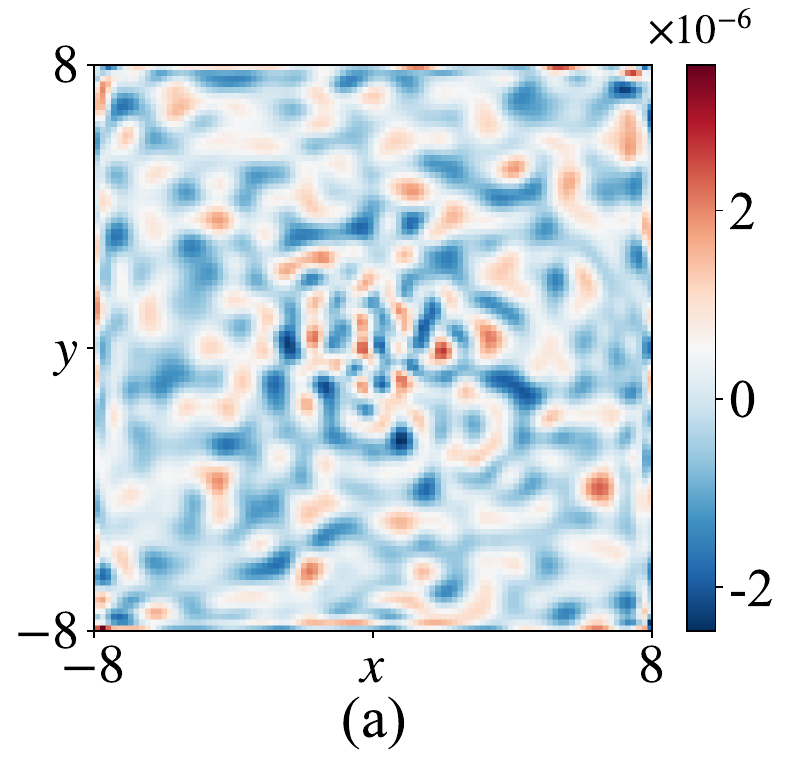,width=1.9in,height=1.5in,angle=0}
		\hspace{0in}
		\psfig{figure=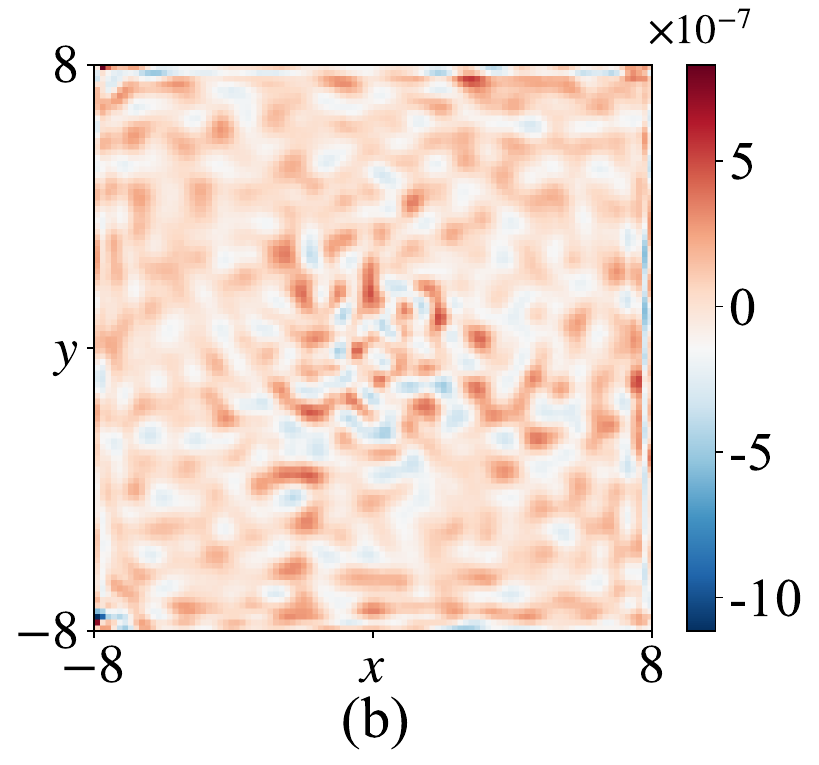,width=1.9in,height=1.5in,angle=0}
		\hspace{0in}
		\psfig{figure=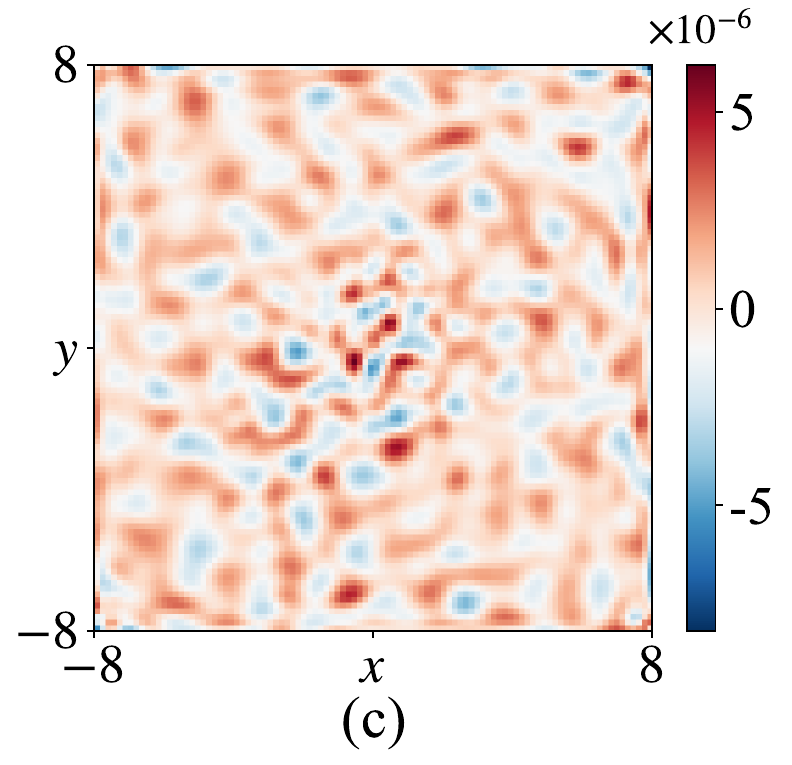,width=1.9in,height=1.5in,angle=0}
		
	}

	\caption{ Numerical solution (1st row) and pointwise error (2nd row) of DCNO for Example~\ref{Ex:MIO_sup} on three test samples:  \\ (a) $\alpha=\tau=2$; (b) $\alpha=\tau=1.5$; (c) $\alpha=\tau=1$.}
	\label{fig:mionet_supervised_solution}
\end{figure}

\begin{example}[DGNO with MIONet]\label{Ex:MIO_unsup}
	We turn to the unsupervised scenario by considering the following general elliptic boundary value problem
	\begin{equation}\label{eq:elliptic}
	\begin{cases}
		-\nabla \cdot \bigl(\kappa(\mathbf{x})\nabla \phi(\mathbf{x})\bigr) = \rho(\mathbf{x}), & \mathbf{x} \in \Omega, \\[4pt]
		\phi(\mathbf{x}) = 0, & \mathbf{x} \in \partial\Omega,
	\end{cases}
	\end{equation}
	where the computational domain is set as $\Omega = (0,1)^2$.
\end{example}
\begin{figure}[t!]
	\centerline{
		\psfig{figure=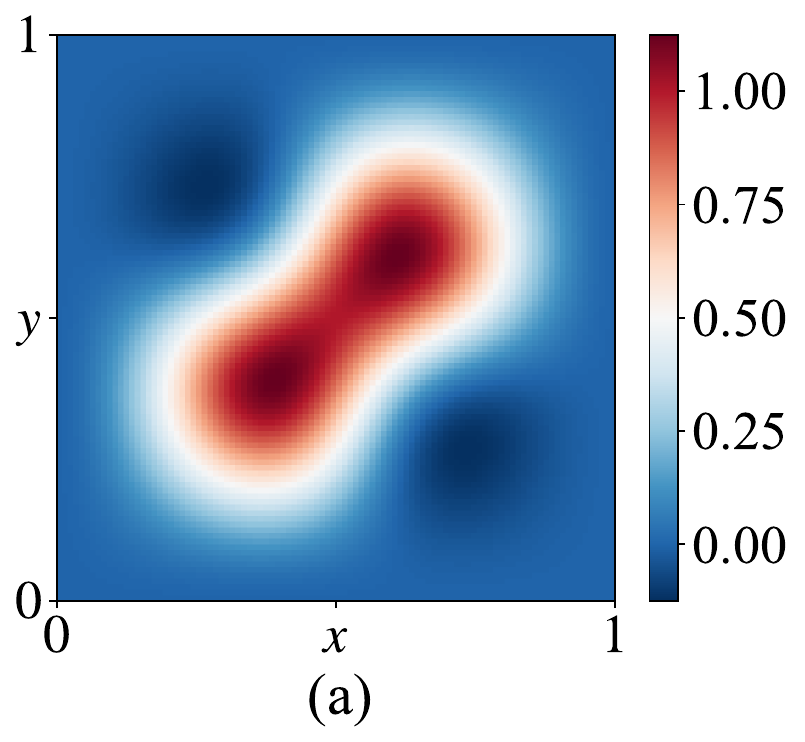,width=1.9in,height=1.5in,angle=0}
		\hspace{0in}
		\psfig{figure=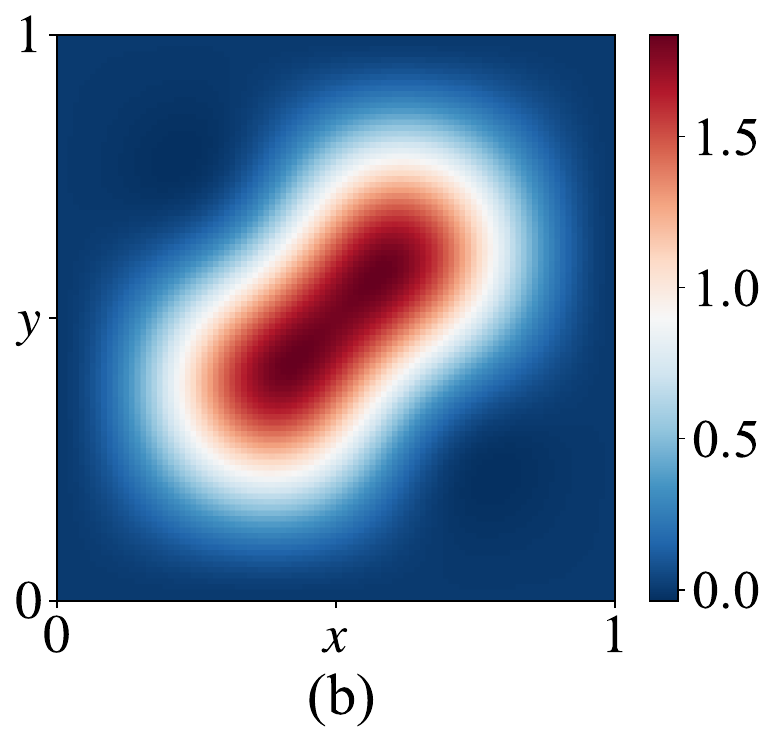,width=1.9in,height=1.5in,angle=0}
		\hspace{0in}
		\psfig{figure=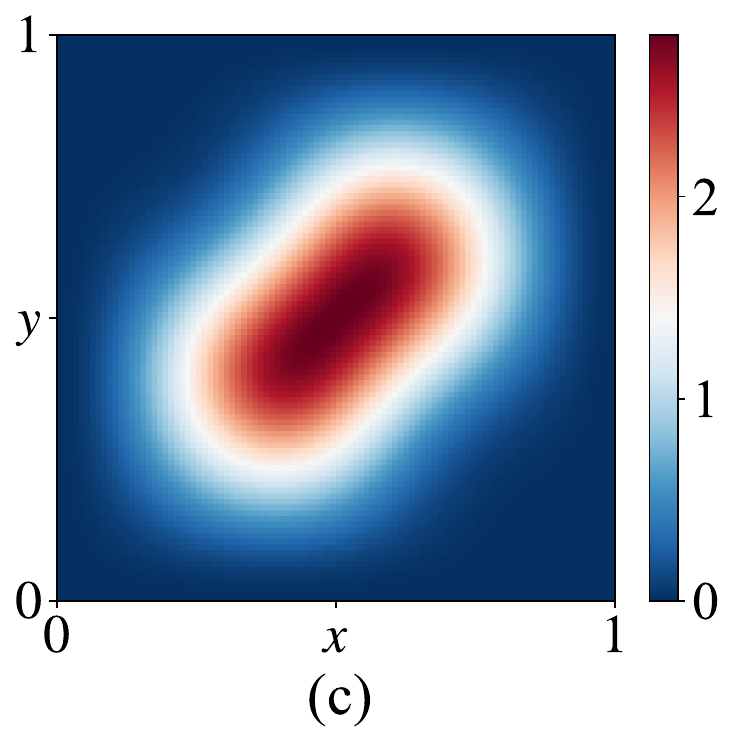,width=1.9in,height=1.5in,angle=0}
		
	}
    
    \centerline{
		\psfig{figure=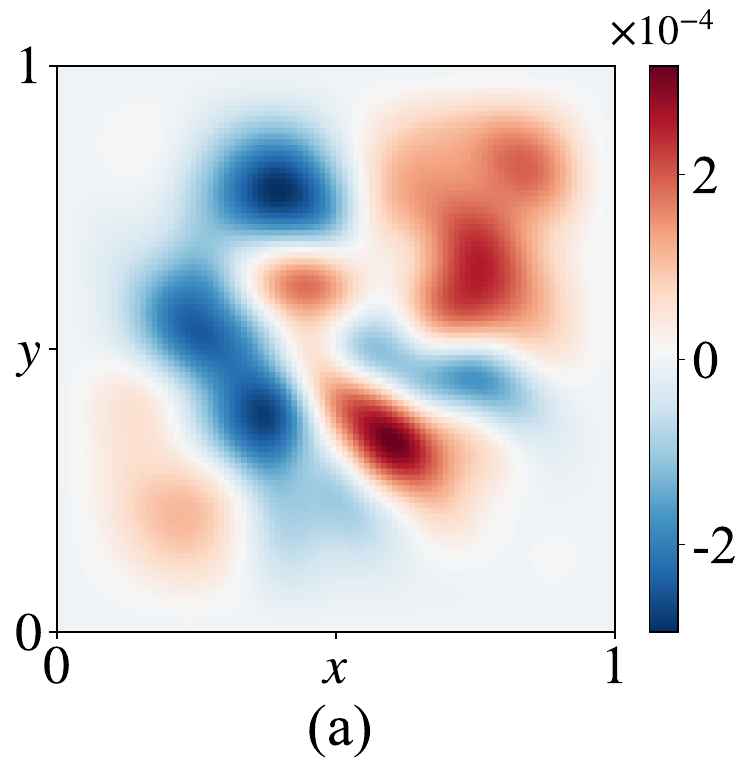,width=1.9in,height=1.5in,angle=0}
		\hspace{0in}
		\psfig{figure=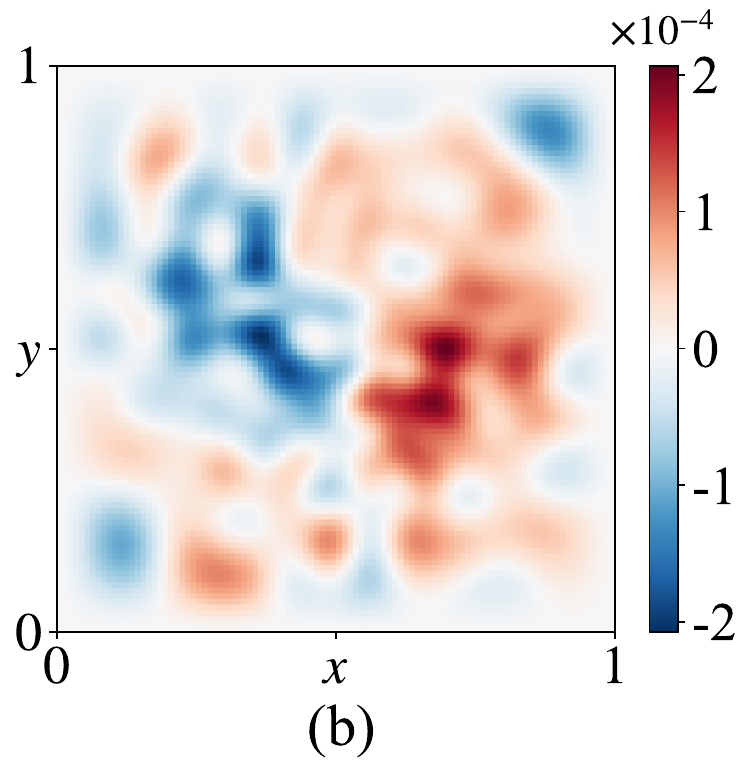,width=1.9in,height=1.5in,angle=0}
		\hspace{0in}
		\psfig{figure=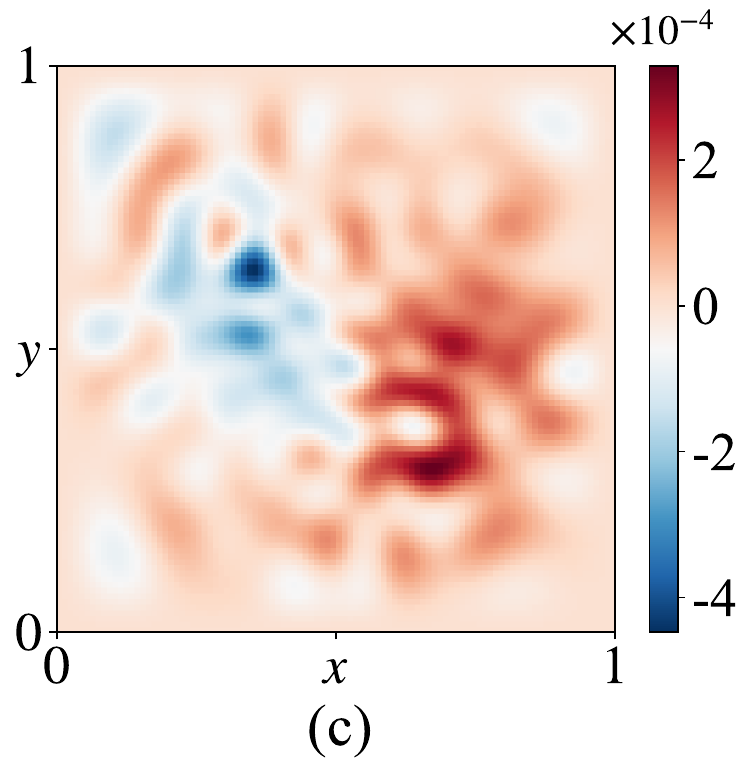,width=1.9in,height=1.5in,angle=0}
		
	}
	\caption{
    Numerical solution (1st row) and pointwise error (2nd row) of DGNO for Example~\ref{Ex:MIO_unsup} on three test samples:  \\ (a) $\alpha=0.5,\,\gamma=2$; (b) $\alpha=0.6,\,\gamma=3$; (c) $\alpha=0.7,\,\gamma=4$.   
    }
	\label{fig:mionet_unsupervised_solution}
\end{figure}

In \eqref{eq:elliptic}, different choices of $\kappa(\bx)$ correspond to different integral kernels in the convolution \eqref{eq:convolution}. 
Our goal is to learn the mapping $\mathcal{G}: (\rho, \kappa) \mapsto \phi$ with $\rho$ and $\kappa$ jointly serving as inputs, which constitutes a multiple-input operator learning problem. The differential formulation in which the problem is posed enables an unsupervised training strategy, and
we apply the DGNO algorithm with two branch networks to input $\rho$ and $\kappa$. To obtain reference solutions for tests, 
we set
$
\kappa(\mathbf{x}) = 1 + \alpha \sin(2\pi x)\sin(2\pi y)
$
with $|\alpha|<1$ 
and  
\[
\begin{aligned}
	\rho(\mathbf{x}) &=
	2\pi^{2}\alpha\,
	\kappa(\mathbf{x})\Bigl[\gamma\sin(\pi x)\sin(\pi y)+4\sin(2\pi x)\sin(2\pi y)\Bigr] 
	- 2\pi^{2}\alpha^{2}\Bigl[
	\gamma\bigl(\cos(2\pi x)\sin(2\pi y)\cos(\pi x)\sin(\pi y) \\
	&\quad + \sin(2\pi x)\cos(2\pi y)\sin(\pi x)\cos(\pi y)\bigr) 
	+ 2\bigl(\cos^{2}(2\pi x)\sin^{2}(2\pi y)+\sin^{2}(2\pi x)\cos^{2}(2\pi y)\bigr)
	\Bigr],
\end{aligned}
\]
where 
$
\phi(\mathbf{x}) = \alpha\gamma\sin(\pi x)\sin(\pi y)+\alpha\sin(2\pi x)\sin(2\pi y)
$ is the exact solution to \eqref{eq:elliptic}. 
For the training data, a set $\{ (\kappa^{(b)}, \rho^{(l)}) \}_{b=1,l=1}^{N_b,N_l}$ with $N_b = N_l = 20$ is generated, resulting in a total of 400 samples. The parameters $\alpha^{(b)}$ and $\tau^{(l)}$ are independently sampled from $\mathcal{U}(0.5,1)$ and $\mathcal{U}[2,4]$, respectively.
All other training settings are the same as those in Example~\ref{Ex:Poisson}, except that $N_x=N_y=110$ are used during training, while a finer grid with $N_x=N_y=150$ is adopted for error evaluation.
In view of \eqref{eq:loss_deeponet_un}, the extended practical loss for DGNO in the multiple-input case reads 
\begin{equation}\label{eq:loss_mionet_un}
	\text{Loss}_s(\theta) =\frac{-1}{N_b\times N_l } \sum_{b=1}^{N_b} \sum_{l=1}^{N_l}  \frac{R_{s-1}(\kappa^{(b)},\rho^{(l)},\mathcal{G}_{\theta_s}(\kappa^{(b)},\rho^{(l)}))}{\|\mathcal{G}_{\theta_s}(\kappa^{(b)},\rho^{(l)})\|_X}.
\end{equation}

For testing, 200 samples that are disjoint from the training set are randomly selected to form the test set. 
As shown in Fig.~\ref{fig:mio_unsupervised_error}(b), the mean relative $L^2$ error on the test set decreases progressively with the stage index $s$, and ultimately reaches an accuracy of $7.08 \times 10^{-5}$.
We also select some representative samples from the test set to visually assess the approximation capability of the trained basis operators in
Fig. \ref{fig:mionet_unsupervised_solution}.
These results demonstrate that DGNO as the unsupervised method also performs well for the multiple-input case.

\section{Conclusion}\label{sec:con}
In this work, we have introduced two novel and general multi-stage neural operator learning frameworks, the Deep Collocation Neural Operator (DCNO) and the Deep Galerkin Neural Operator (DGNO), designed to significantly enhance the accuracy of operator approximation, particularly applied for convolutions. DCNO offers a supervised approach that progressively builds a richer approximation space by training basis operators on the residuals of the target operator, leveraging available input-output data. Complementing this, DGNO provides an unsupervised framework that learns operator bases by minimizing the residual of a governing PDE in its weak form, thus eliminating the need for labeled data. 

We have demonstrated the efficacy of these methods through extensive numerical experiments. Both DCNO and DGNO achieved high accuracy in learning convolution operators, often approaching machine precision of single-precision floating-point arithmetic. Notably, for convolution problems lacking a straightforward PDE formulation, DCNO proved to be a robust supervised solution. When a PDE formulation is available, DGNO offers an efficient unsupervised alternative. Furthermore, we have extended these frameworks to handle multi-input operator learning scenarios, encompassing variations in both the density and the kernel for convolutions. 

		\appendix
		\section*{Acknowledgements}
	Z. Wen and X. Zhao are supported by 
	National Key Research and Development Program of
	China, National MCF Energy R$\&$D Program (No. 2024YFE03240400), National Natural Science
	Foundation of China (Nos. 42450275, 12271413).
    Z. Mao is supported by 
    the Zhejiang Provincial Natural Science Foundation of China under Grant No. LR26A010001 and the National Natural Science Foundation of China under Grant No. 12531015.
    Y. Zhang is supported by  the National Natural Science Foundation of China (No. 12271400), the National Key R$\&$D Program of China (No. 2024YFA1012803) and basic research fund of Tianjin University under Grant 2025XJ21-0010.

		\bibliographystyle{siamplain}
		\bibliography{refs}
		
	\end{document}